\documentclass[10pt,twocolumn,letterpaper]{article}

\usepackage[top=0.9 in,bottom=0.9in,left=0.85in,right=0.85in,columnsep=0.4in]{geometry}

\usepackage{amsmath,amsthm,amssymb,eucal}
\usepackage{mathtools}
\usepackage{url}
\usepackage{hyperref} 
\usepackage{stmaryrd}
\usepackage{tikz-cd}
\usepackage{textcomp}
\usepackage{dsfont}
\usepackage[utf8]{inputenc} 
\usepackage[T1]{fontenc}    
\usepackage{url}            
\usepackage{booktabs}       
\usepackage{amsfonts}       
\usepackage{nicefrac}       
\usepackage{microtype}

\usepackage{graphicx}

\DeclareMathOperator{\Tr}{Tr}

\usepackage{color}

\newcommand{\M}{\mathbf{M}}

\newcommand{\R}{\mathbb{R}}
\newcommand{\Q}{\mathbf{Q}}

\newcommand{\F}{\mathcal{F}}
\newcommand{\G}{\mathcal{G}}
\newcommand{\B}{\mathbf{B}}
\newcommand{\C}{\mathbf{C}}

\newcommand{\T}{\mathbf{T}}
\newcommand{\x}{\mathbf{x}}
\newcommand{\X}{\mathbf{X}}
\newcommand{\y}{\mathbf{y}}
\newcommand{\Y}{\mathbf{Y}}
\newcommand{\w}{\mathbf{w}}
\newcommand{\W}{\mathbf{W}}
\newcommand{\z}{\mathbf{z}}

\newcommand{\Z}{\mathbf{Z}}
\newcommand{\I}{\mathbf{I}}
\newcommand{\U}{\mathbf{U}}
\newcommand{\V}{\mathbf{V}}
\newcommand{\D}{\mathbf{D}}
\newcommand{\A}{\mathbf{A}}

\newtheorem{theorem}{Theorem}[section]
 
\newtheorem{lemma}[theorem]{Lemma}

\title{\textbf{Information Theoretic Lower Bounds for Feed-Forward Fully-Connected Deep Networks}}

\author{
	Xiaochen Yang \\
	Department of Statistics \\
	Purdue University\\
	West Lafayette, IN 47906 \\
	\texttt{yang1641@purdue.edu} \\
	\and
	Jean Honorio \\
	Department of Computer Science \\
	Purdue University \\
	West Lafayette, IN 47906 \\
	\texttt{jhonorio@purdue.edu} \\
}

\date{}

\begin{document}

\maketitle

\begin{abstract}
  In this paper, we study the sample complexity lower bounds for the exact recovery of parameters and for a positive excess risk of a feed-forward, fully-connected neural network for binary classification, using information-theoretic tools. We prove these lower bounds by the existence of a generative network characterized by a backwards data generating process, where the input is generated based on the binary output, and the network is parametrized by weight parameters for the hidden layers. The sample complexity lower bound for the exact recovery of parameters is $\Omega(d r \log(r) + p )$ and for a positive excess risk is $\Omega(r \log(r) + p )$, where $p$ is the dimension of the input, $r$ reflects the rank of the weight matrices and $d$ is the number of hidden layers. To the best of our knowledge, our results are the first information theoretic lower bounds.
\end{abstract}

\section{Introduction}

\paragraph{Motivation.} 

There has been an abundace of studies on the generalization upper bound on fully-connected, feed-forward neural networks. Various methods have been applied to obtain generalization and error probability upper bounds, including shattering coefficients (Bartlett, 1998), margin-based bounds (Bartlett et al, 2017), PAC-Bayes methods (Neyshabur et al, 2018) and Rademacher complexity (Neyshabur et al, 2015), among others. With further assumptions, there have also been less size-dependent upper bounds (Golowich et al, 2018). Regardless, one factor common in these upper bounds is the product of norms of parameter matrices, $\prod_{j=1}^d \| \W_j \|$ where $ \| \cdot \|$ is some matrix norm (e.g. Frobenius norm, spectral norm, (2,1)-norm) and $\W_j$'s are the parameter matrices of each layer in the network. 

One interesting behavior of deep learning is that over-parametrized networks seem to generalize well (Li et al, 2018), yet it is not well understood which sample size makes an over-parametrized network perform well and which does not. To the best of our knowledge, there is no sample complexity lower bound by information-theoretic tools yet, and in this paper we provide such a bound. Specifically, we prove our claim by finding one generative neural network achieving these lower bounds. 

\paragraph{Contribution.} 

In this paper, we study the necessary sample complexity lower bounds by considering a generative model, which to the best of our knowledge, we are the first to find such characterization. Specifically, in this generative model, the output, binary labels, are first generated, and then the input is generated conditioned on the output and intermediate hidden layers. Moreover, we consider a setup where the parameter matrices are permutation matrices of low rank, and we show that $\Omega(d r \log(r) + p)$ samples are necessary for the exact recovery of the model parameters, and $\Omega(r \log(r) + p)$ samples for a positive excess risk, where $p$ is the dimension of the input, $r$ is the rank of the parameter matrices, and $d$ is the number of hidden layers in the deep network.

\section{Preliminaries} \label{sec:2}

We begin with some background on fully-connected feed-forward neural networks and the minimax risk framework.

\subsection{Fully-connected feed-forward neural network}

Consider a fully-connected feed-forward neural network with $d$ hidden layers. A common representation is
\begin{align} \label{eq:1}
\begin{split}
\x \mapsto \w_{d+1} \sigma( \W_{d} \sigma (\W_{d-1} \cdots \sigma(\W_1 \x))) 
\end{split}
\end{align}
where $\x \in \R^p$, $\sigma$ is some activation function possibly applied elementwise, $\W_1, \W_2, \cdots \W_{d}$ are matrices parametrizing the hidden layers, $\w_{d+1}$ is a parameter vector, and $y$ is the output.

We focus on the binary classification setting, where each input $\x_i$ is associated with a label, $y_i \in \{ -1, +1 \}$, and the data set $S = \{ (\x_i, y_i) \}_{i=1}^n$ contains $n$ i.i.d. samples. For the common representation introduced above, there is a Markov chain from input $\x$ to output $y$, formally described as:

\begin{align}
\begin{split}
\x 
&\mapsto \W_1 \x 
\mapsto \cdots 
\mapsto \W_{d} \sigma (\W_{d-1} \cdots \sigma(\W_1 \x)) \\
&\mapsto \w_{d+1} \sigma( \W_{d} \sigma (\W_{d-1} \cdots \sigma(\W_1 \x))) 
\mapsto y
\end{split}
\end{align}

Our argument will build on this Markov chain, but we will consider the reverse of this Markov chain, that is, we will consider a generative model, where the label $y$ is first generated, then the input $x$ is generated conditioned on $y$. We will utilize the symmetry of the mutual information between $\x$ and $y$ when the parameter is given.

\subsection{A backwards generative neural network}

Next we propose a generative neural network characterized by a backward data generating process, where the input $\x$ is generated based on the label $y$ and the network is parametrized by some parameter $\tilde{\W}$ different from the parameter $\W$ above. It is described by the Markov chain below, 
\begin{align}
\begin{split} \label{eq:bwd-mc}
y 
&\mapsto y \tilde{\w}_0
\mapsto \tilde{\W}_{1} \tilde{\sigma}(y \tilde{\w}_0) 
\mapsto \cdots \\
&\mapsto \tilde{\W}_{d} \tilde{\sigma} (\tilde{\W}_{d-1} \cdots \tilde{\sigma}(y \tilde{\w}_0)) 
\mapsto \x
\end{split}
\end{align}
where $\tilde{\sigma}$ is (the inverse of) some activation function, and the parameter $\tilde{\W} := (\tilde{\w}_0, \tilde{\W}_1, \cdots, \tilde{\W}_d)$ is a collection of weight matrices and vector characterizing this backwards generative model.

We argue that such a generative model is worth studying, as the mutual information between the data $S = \{ (\x_i, y_i) \}_{i=1}^n$ and the parameter, $\mathbb{I}(S; \tilde{\W})$, is central to information-theorectic framework, and $\mathbb{I}(S; \tilde{\W})$ does not dictate a particular direction of dependency between $\x$ and $y$. Moreover, it is well known that a deep network can represent almost any function, thus studying the sample complexity of such a generative network is also relevant to the understanding of the expressiveness of deep networks.

We will derive the sample complexity lower bounds of the probability of exact recovery of parameter $\tilde{\W}$ and the probability of having a positive excess risk. Note that this backwards model correspond to a hypothesis class including any deep network, as it is possible to have parameter $\tilde{\W}$ generating input $\x$ corresponding to the (randomly generated) labels. Imagine the binary labels are first generated, then there must exist some vector $\tilde{\w}_0$, some matrices $\tilde{\W}_1, \cdots, \tilde{\W}_d$ and some activation function $\tilde{\sigma}$ that gives $(\x, y); \tilde{\W}$ with distribution resembling the desired data distribution $(\x, y)$. 

Another thing to note is that the dimension of $\Tilde{\W}_i$ is different from that of $\W_i$ for $i \in \{ 1, 2, \cdots, d\}$. In fact, there is also no guarantee that $\Tilde{\W}_i$ will be the (pseudo-)inverse of $\W_i$.

\subsection{Minimax Framework}

We briefly review the minimax framework in this section and introduce the two measures of risk we examine in this paper. The minimax framework consists of a well defined objective that aims to shed light about the optimality of algorithms and has been widely used in statistics and machine learning (Wainwright, 2019), (Wasserman, 2006).

We start with the definition of a minimax decoder. Let $\mathcal{P}$ denote a family of distributions on $\F \times \mathcal{X}^n$, where $\F$ is a hypothesis class and $f \in \F$ is a hypothesis that parametrizes the model, $\mathcal{X}^n$ is the $n$-Cartesian product of the sample space $\mathcal{X}$, and $S \in \mathcal{X}^n$ is an i.i.d. dataset of size $n$. Thus $\mathbb{P}(f, S) \in \mathcal{P}$ is a joint distribution of $(f, S)$.

Let $\hat{f}: \mathcal{X}^n \to \F$ denote a decoder, which is a function that takes a dataset in $\mathcal{X}^n$ and returns a hypothesis in $\F$. Let $\Psi(\F) = \{ \hat{f}: \mathcal{X}^n \to \F \}$ denote the class of decoders that output hypothesis in $\F$ by any conceivable procedure. Suppose the true $f^*$ is a random variable with distribution $\mathbb{P}(f^*)$, and note that by Bayes theorem we can write $\mathbb{P}(f, S) = \mathbb{P}(f^*) \mathbb{P}(S | f^*)$. Given $\mathbb{P} = \mathbb{P}(f, S)$, a joint distribution of $(f, S)$, we can assess the quality of a decoder $\hat{f}$ by its risk $r: \Psi \times \F \to \R_{+}$,  
\begin{align} \label{eq:4}
r(\hat{f}, f^*) := \mathbb{E}_{(f^*, S) \sim \mathbb{P}}[\rho(\hat{f}(S), f^*)]
\end{align}
where $\rho: \F \times \F \to \R_{+}$ is a (semi)metric on the hypothesis class $\F$. 

In this paper we analyze the necessary number of samples for the probability of exact recovery of parameter to be greater than $1/2$, and study $P_{(f^*, S) \sim \mathbb{P}}(\hat{f}(S) \neq f^*)$, which corresponds to $\rho(f, f') = \mathbf{1} \{f \neq f'\}$ for $f, f' \in \F$. 

Another measure of risk we study is related to the excess probability of making a wrong prediction. We show that there is a parameter identifiability issue with respect to risk for certain deep networks. We analyze the lower bound of the excess probability of wrong prediction and the sample complexity lower bound for achieving this lower bound in spite of the identifiability issue. We provide the exact definition of this risk measure when presenting our theorem statement.

Another way to assess the quality of a decoder $\hat{f}$ is to consider its maximum risk over the hypothesis class $\F$, that is, $\sup_{f^* \in \F} r(\hat{f}, f^*)$, and we say a decoder $\hat{f}$ is a minimax decoder if it minimizes the maximum risk, that is, 
$$
\inf_{\bar{f} \in \Psi(\F)} \sup_{\mathbb{P}(f^*, S) \in \mathcal{P}} r(\bar{f}, f^*) = \sup_{\mathbb{P}(f^*, S) \in \mathcal{P}} r(\hat{f}, f^*)
$$
and we call this risk the minimax risk.

Under this framework, we will show that there exist a subset $\mathcal{Q}$ of $\mathcal{P}$ and a subset $\G$ of $\F$ such that 
$$
\inf_{\bar{f} \in \Psi(\F)} \sup_{\mathbb{P}(f^*, S) \in \mathcal{P}} r(\bar{f}, f^*) \geq \inf_{\bar{f} \in \Psi(\G)} \sup_{\mathbb{P}(f^*, S) \in \mathcal{Q}} r(\bar{f}, f^*) 
$$
thus the necessary number of samples for $\G$ and $\mathcal{Q}$ becomes a sample complexity lower bound for $\F$ and $\mathcal{P}$.

\section{Information-Theoretic Sample Complexity Lower Bound} \label{sec:3}

In this section, we state our main results on the information-theoretic sample complexity lower bound and discuss some of their implication. We first introduce some definition to help explaining our minimax argument. 

\subsection{Notation for minimax argument and theorems}

Let $\F_{p, d}$ denote the class of parameter matrices and vector for $d$-layer deep networks with input dimension $p$. More formally, 
$$
\F_{p, d} = \{ \tilde{\w}_0 \in \R^p \} \times \left( \bigtimes_{i=1}^d \{ \tilde{\W}_i \in \R^{n_i \times n_{i-1}} \} \right),
$$
where $n_i$ is the size of the $i$-th layer. Then for $\tilde{\W} \in \F_{p, d}$, $\tilde{\W} = (\tilde{\w}_0, \tilde{\W}_1, \cdots, \tilde{\W}_d)$, where $\tilde{\w}_0 \in \R^p$, and each $\tilde{\W}_i \in \R^{n_i \times n_{i-1}}$.

Consider the hypothesis $\tilde{\W} \in \F_{p, d}$ as a random variable, i.e., the random variable $\tilde{\W}$ has support on $\F_{p, d}$, and denote its distribution by $\mathbb{P}(\tilde{\W})$. Let $\mathcal{P}_{\tilde{\W}}(\mathcal{F}_{p, d})$ denote the collection of possible distributions of  $\tilde{\W}$, that is, $\mathbb{P}(\tilde{\W})$ above is an element of $\mathcal{P}_{\tilde{\W}}(\mathcal{F}_{p, d})$.

Given a fixed hypothesis $\tilde{\W}^*$, a dataset $S = \{ (\x_i, y_i) \}_{i=1}^n$ consists of i.i.d. observations generated from a $d$-layer network with input dimension $p$, parameterized by $\tilde{\W}^*$. We use $\mathbb{P}(\tilde{\W}^*, S)$ to denote the joint distribution of hypothesis $\tilde{\W}^*$ and dataset $S$, and by Bayes theorem, $\mathbb{P}(\tilde{\W}^*, S) = \mathbb{P}(\tilde{\W}^*) \mathbb{P}(S | \tilde{\W}^*)$, where $\mathbb{P}(\tilde{\W}^*) \in \mathcal{P}_{\tilde{\W}}(\mathcal{F}_{p, d})$, and the probability distribution $\mathbb{P}(S | \tilde{\W}^*)$ has a variety of choices, as it depends on the choice of activation function and the distribution of the input $\x$. Because $S$ is an i.i.d. dataset, we can write $\mathbb{P}(S | \tilde{\W}^*) = \mathbb{P}^n((\x, y) | \tilde{\W}^*)$, and we further discuss $\mathbb{P}((\x, y) | \tilde{\W}^*)$ below.

Let $\mathcal{P}_{(\x, y) | \tilde{\W}}$ denote a collection of distributions given fixed $\tilde{\W}$, i.e. for $\mathbb{P}_{(\x, y) | \tilde{\W}} \in \mathcal{P}_{(\x, y) | \tilde{\W}}$, $\mathbb{P}_{(\x, y) | \tilde{\W}}$ is the joint distribution of the input $\x$ and output $y$, where $y \in \{-1, +1\}$ as we study binary classifier in this paper. The distribution $\mathbb{P}_{(\x, y) | \tilde{\W}}$ is affected by the activation function and distribution of $\x$, and $\mathcal{P}_{(\x, y) | \tilde{\W}}$ consists of all possible distributions of $(\x, y)$ generated by deep networks parametrized by $\tilde{\W}$ with \textit{all} possible activation functions and \textit{all} possible data distributions for $\x$. 

Now we describe the joint distribution of the hypothesis and the data. Let $\mathcal{P}_{\tilde{\W}, S}(\mathcal{F}_{p, d})$ denote a collection of joint distributions of the hypothesis $\tilde{\W}$ and the dataset $S$, then
\begin{align}
	\begin{split}
		&\mathcal{P} := \mathcal{P}_{\tilde{\W}, S}(\mathcal{P}_{\tilde{\W}}(\mathcal{F}_{p, d}), \mathcal{P}_{(\x, y) | \tilde{\W}}) \\
		= &\{ \mathbb{P}(\tilde{\W}, S) = \mathbb{P}(\tilde{\W}) \mathbb{P}^n((\x, y) | \tilde{\W}): \\
		&\mathbb{P}(\tilde{\W}) \in \mathcal{P}_{\tilde{\W}}(\mathcal{F}_{p, d}), \mathbb{P}((\x, y) | \tilde{\W}) \in \mathcal{P}_{(\x, y) | \tilde{\W}} \} \\
		= & \mathcal{P}_{\tilde{\W}}(\mathcal{F}_{p, d}) \times \left( \mathcal{P}_{(\x, y) | \tilde{\W}} \right)^n
	\end{split}
\end{align}
where we slightly abuse the notation to let $\mathcal{P}$ be shorthand for $\mathcal{P}_{\tilde{\W}, S}(\mathcal{P}_{\tilde{\W}}(\mathcal{F}_{p, d}), \mathcal{P}_{(\x, y) | \tilde{\W}})$. 

As it is customary in minimax analyses (Wang et al, 2010), (Santhanam et al, 2012), (Tandon et al, 2014), we consider a restricted class. Thus, we can further specify a subset of $\mathcal{P}_{(\x, y) | \tilde{\W}}$, restricting the choice of activation function and the type of distribution for $\x$. Let $\mathcal{P}_{(\x, y) | \tilde{\W}}^{\sigma, \mathcal{Q}_{\x}}$ denote the collections of distributions of $(\x, y)$ generated from a $d$-layer network with input dimension $p$, activation function $\sigma$, parameterized by $\tilde{\W}$, and where the input $\x$ has distribution $\mathbb{P}(\x) \in \mathcal{Q}_{\x}$, then $\mathcal{P}_{(\x, y) | \tilde{\W}}^{\sigma, \mathcal{Q}_{\x}} \subseteq \mathcal{P}_{(\x, y) | \tilde{\W}}$. For example, $\mathcal{Q}_{\x}$ could be the collection of all multivariate normal distributions of dimension $p$, and $\sigma$ could be the sigmoid, tanh, leaky ReLU, or identity activation function.

We can also consider any subset $\mathcal{G}_{p,d} \subseteq \F_{p, d}$ and as such $\mathcal{G}_{p,d}$ corresponds to weight matrices and vector with certain restrictions, e.g. low rankness. 

If we define a collection of joint distributions of the hypothesis $\tilde{\W}$ and data $S$ based on some $\mathcal{G}_{p,d} \subseteq \F_{p, d}$ and $\mathcal{P}_{(\x, y) | \tilde{\W}}^{\sigma, \mathcal{Q}_{\x}}$, then this collection $\mathcal{P}_{\tilde{\W}, S}(\mathcal{P}_{\tilde{\W}}(\mathcal{G}_{p, d}), \mathcal{P}_{(\x, y) | \tilde{\W}}^{\sigma, \mathcal{Q}_{\x}})$ would be a subset of $\mathcal{P}$.

Now we are ready to state our main results.

\subsection{Main Theorems}

We define a shorthand for the probability of failing to recover the exact parameter, 
\begin{align}
\xi_1(\hat{f}, \mathbb{P}) := P_{(\tilde{\W}^*, S) \sim \mathbb{P}}(\hat{f}(S) \neq \tilde{\W}^*)
\end{align}

\begin{theorem}[Sample complexity lower bound for exact recovery of hypothesis] \label{thm:3.1}
	There exists a subset
	\begin{align}
		\mathcal{P}' := \mathcal{P}_{\tilde{\W}, S}( \{\text{Uniform}(\mathcal{G}_{p, d})\}, \mathcal{P}_{(\x, y) | \tilde{\W}}^{\text{Id}, \mathcal{Q}_{\x}} )
	\end{align}
	of $\mathcal{P} := \mathcal{P}_{\tilde{\W}, S}(\mathcal{P}_{\tilde{\W}}(\mathcal{F}_{p, d}), \mathcal{P}_{(\x, y) | \tilde{\W}})$ where ${\text{Id}}: \R \to \R$ is ${\text{Id}}(z) = z$, $\mathcal{G}_{p, d} \subseteq \mathcal{F}_{p, d}$ and $\mathcal{P}_{(\x, y) | \tilde{\W}}^{\text{Id}, \mathcal{Q}_{\x}} \subseteq \mathcal{P}_{(\x, y) | \tilde{\W}}$, such that if the nature chooses $\tilde{\W}^*$ uniformly at random from $\mathcal{G}_{p, d}$, and data $S$ is generated from $\mathcal{P}_{(\x, y) | \tilde{\W}}^{\text{Id}, \mathcal{Q}_{\x}}$ for some $\mathcal{Q}_{\x}$, then for any decoder $\hat{f} \in \Psi(\mathcal{G}_{p, d})$ where 
	\begin{align}
		\Psi(\mathcal{G}_{p, d}) := \{ \hat{f}: (\R^p \times \{-1, +1\})^n \to \G_{p,d} \},
	\end{align}
	if
	$$ n \leq \sigma^2 \frac{  d \left( \sum_{i=1}^{r} \log(i) \right) + p \log(2) - \log(4) }{4}$$ 
	then 
	$$\inf_{\hat{f} \in \Psi(\G_{p,d})} \sup_{\mathbb{P} \in \mathcal{P}'} \xi_1(\hat{f}, \mathbb{P}) \geq \frac{1}{2},$$
	where $\sigma^2$ is a constant controlling the variance of the distribution of the input $\x$ and is associated with $\mathcal{Q}_{\x}$, and $r$ is the rank of all parameter matrix $\tilde{\W}_i$ in $\mathcal{G}_{p, d}$
\end{theorem}

Now we state our second theorem, which follows from Theorem \ref{thm:3.1}. 

\begin{theorem}[Sample complexity lower bound for exact recovery of hypothesis] \label{thm:3.2}
	If the sample size $n$ fulfills
	$$ n \in \mathcal{O} \left( \frac{  d \left( \sum_{i=1}^{r} \log(i) \right) + p \log(2) - \log(4) }{4} \right),$$
	then
	\begin{align}
		\begin{split}
			\inf_{\hat{f} \in \Psi(\F_{p,d})} \sup_{\mathbb{P} \in \mathcal{P}} \xi_1(\hat{f}, \mathbb{P}) 
			\geq \frac{1}{2}
		\end{split}
	\end{align}
\end{theorem}
\begin{proof}
	Since $\xi_1(\hat{f}, \mathbb{P}) := P_{(\tilde{\W}^*, S) \sim \mathbb{P}}(\hat{f}(S) \neq \tilde{\W}^*)$, 
	\begin{align} \label{eq:9}
		\inf_{\hat{f} \in \Psi(\F_{p,d})} \sup_{\mathbb{P} \in \mathcal{P}'} \xi_1(\hat{f}, \mathbb{P}) \geq \inf_{\hat{f} \in \Psi(\G_{p,d})} \sup_{\mathbb{P} \in \mathcal{P}'} \xi_1(\hat{f}, \mathbb{P})
	\end{align}
	as $\mathbb{P} \in \mathcal{P}'$ implies $\tilde{\W}^* \in \G_{p,d}$. Note that a decoder $\hat{f} \in \Psi(\F_{p,d})$ may return a hypothesis outside of $\mathcal{G}_{p, d}$, i.e. there may be $(\tilde{\W}^*, S) \sim \mathbb{P}$ such that $\hat{f}(S) \in \mathcal{F}_{p, d} \setminus \mathcal{G}_{p, d}$. It is possible this $\hat{f}(S)$ offers a good prediction, however, since we are studying the probability of exact recovery of the true parameter $\tilde{\W}^* \in \mathcal{G}_{p, d}$, $\xi_1(\hat{f}, \mathbb{P})$ will always be $1$, the worst risk, if $\hat{f}(S) \in \mathcal{F}_{p, d} \setminus \mathcal{G}_{p, d}$.
	Moreover, by basic property of the infimum, 
	$$\inf_{\hat{f} \in \Psi(\F_{p,d})} \sup_{\mathbb{P} \in \mathcal{P}'} \xi_1(\hat{f}, \mathbb{P}) \leq \inf_{\hat{f} \in \Psi(\G_{p,d})} \sup_{\mathbb{P} \in \mathcal{P}'} \xi_1(\hat{f}, \mathbb{P}).$$ 
	Thus
	$$
	\inf_{\hat{f} \in \Psi(\F_{p,d})} \sup_{\mathbb{P} \in \mathcal{P}'} \xi_1(\hat{f}, \mathbb{P}) = \inf_{\hat{f} \in \Psi(\G_{p,d})} \sup_{\mathbb{P} \in \mathcal{P}'} \xi_1(\hat{f}, \mathbb{P})
	$$
	Besides, $
	\sup_{\mathbb{P} \in \mathcal{P}} \xi_1(\hat{f}, \mathbb{P}) \geq \sup_{\mathbb{P} \in \mathcal{P}'} \xi_1(\hat{f}, \mathbb{P})
	$ by basic property of the supremum, therefore, by Theorem \ref{thm:3.1}, if $$ n \leq \sigma^2 \frac{  d \left( \sum_{i=1}^{r} \log(i) \right) + p \log(2) - \log(4) }{4},$$
	\begin{align}
		\begin{split}
			\inf_{\hat{f} \in \Psi(\F_{p,d})} \sup_{\mathbb{P} \in \mathcal{P}} \xi_1(\hat{f}, \mathbb{P}) 
			&\geq
			\inf_{\hat{f} \in \Psi(\F_{p,d})} \sup_{\mathbb{P} \in \mathcal{P}'} \xi_1(\hat{f}, \mathbb{P}) \\
			&= \inf_{\hat{f} \in \Psi(\G_{p,d})} \sup_{\mathbb{P} \in \mathcal{P}'} \xi_1(\hat{f}, \mathbb{P}) \\
			&\geq \inf_{\hat{f} \in \Psi(\G_{p,d})} \sup_{\mathbb{P} \in \mathcal{P}'} \frac{1}{2} = \frac{1}{2}
		\end{split}
	\end{align}
\end{proof}

\paragraph{Remark 1} Theorem \ref{thm:3.2} implies that the sample complexity lower bound of order $\Omega(dr \log(r) + p)$ in Theorem \ref{thm:3.1} is also a sample complexity lower bound for the minimax risk $\inf_{\hat{f} \in \Psi(\F_{p,d})} \sup_{\mathbb{P} \in \mathcal{P}} \xi_1(\hat{f}, \mathbb{P})$, which is with respect to any $d$-layer deep network with input dimension $p$, any activation function and any data distribution. $\mathcal{P}'$ will be explicitly defined later.

Now we define the probability of making a wrong prediction on a new data point $(\x, y)$ coming from a network parameterized by $\tilde{\W}^*$ with data distribution $\mathbb{P}_{(\x, y) | \tilde{\W}^*}$. From now on we will sometimes use $\mathbb{P}_{(\x, y) ; \tilde{\W}}$ instead of $\mathbb{P}_{(\x, y) | \tilde{\W}^*}$ to emphasize that $\tilde{\W}$ serves as a parameter.

For any hypothesis $\tilde{\W}$ and the data distribution $\mathbb{P}_{(\x, y) ; \tilde{\W}}$, we can find the conditional marginal distribution of $\x | y; \tilde{\W}$. Let $\boldsymbol{\mu}_{\x | y; \tilde{\W}} = \mathbb{E}_{\x | y; \tilde{\W}} \left[ \x | y \right]$ be the mean of $\x$ conditioned on $y$ and given parameter $\tilde{\W}$, then we can use this mean as a predictor for binary classification, and for any hypothesis $\tilde{\W}$, define the prediction risk of $\tilde{\W}$ with respect to true parameter $\tilde{\W}^*$ as
\begin{align}
	R(\tilde{\W} , \tilde{\W}^*) := P_{(\x, y) ; \tilde{\W}^*} \left[ \x^{\top} \boldsymbol{\mu}_{\x | y; \tilde{\W}} \leq 0 \right].
\end{align}
We will use $R(\tilde{\W})$ as shorthand for $R(\tilde{\W} , \tilde{\W}^*)$, and let $\tilde{R}(\tilde{\W}) := R(\tilde{\W} , \tilde{\W}^*) - R(\tilde{\W}^* , \tilde{\W}^*) = R(\tilde{\W}) - R(\tilde{\W}^*)$ denote the excess risk. 
Now we state our result regarding the probability of having a lower-bounded, positive excess risk. 

\begin{theorem} [Sample complexity lower bound for probability of having a positive excess risk] \label{thm:3.3}
	There exists a subset $\mathcal{P}' = \mathcal{P}_{\tilde{\W}, S}(\{\text{Uniform}(\mathcal{G}_{p, d})\}, \mathcal{P}_{(\x, y) | \tilde{\W}}^{\text{Id}, \mathcal{Q}_{\x}} )$ of $\mathcal{P} = \mathcal{P}_{\tilde{\W}, S}(\mathcal{P}_{\tilde{\W}}(\mathcal{F}_{p, d}), \mathcal{P}_{(\x, y) | \tilde{\W}})$, both same as in Theorem \ref{thm:3.1}, then for any decoder $\hat{f} \in \Psi(\G_{p,d})$, if
	$$n \leq \sigma^2 \frac{\sum_{i=1}^{r} \log(i)  + p \log(2) - \log(4) }{ 4 }$$ 
	then 
	$$\inf_{\hat{f} \in \Psi(\G_{p,d})} \sup_{\mathbb{P} \in \mathcal{P}'} \xi_2(\hat{f}, \mathbb{P}) \geq \frac{1}{2},$$
	where $\xi_2(\hat{f}, \mathbb{P})$ is defined as
	\begin{align}
		\begin{split}
			\xi_2(\hat{f}, \mathbb{P}) := P_{(\tilde{\W}^*, S) \sim \mathbb{P}} \left(\tilde{R}(\hat{f}(S)) \geq
			\frac{\text{erf} \left( c_1 \right) - \text{erf} \left( c_0 \right) }{2} \right)
		\end{split}
	\end{align}
	and constants $c_0, c_1$ are
	\begin{align} \label{eq:14}
		c_0 &:= \frac{ 1 - \frac{1}{2^r} + c^{2d} \left( \frac{1}{2^r} - \frac{1}{2^{p-2}} \right) }{ \sigma \sqrt{2 \left[(d+1) \left( 1 - \frac{1}{2^r} \right) + \left( \frac{1-c^{2(d+1)}}{1 - c^2} \right) \left( \frac{ c^{2d}}{2^r} \right) \right]} }, 
	\end{align}
	\begin{align} \label{eq:15}
		c_1 &:= \frac{ 1 - \frac{1}{2^r} + \left( \frac{c^{2d}}{2^r} \right) }{ \sigma \sqrt{2 \left[(d+1) \left( 1 - \frac{1}{2^r} \right) + \left( \frac{1-c^{2(d+1)}}{1 - c^2} \right) \left( \frac{ c^{2d}}{2^r} \right) \right]} },
	\end{align}
	\text{erf} stands for the Gauss error functionå, $\sigma^2$ is a constant controlling the variance of the distribution of the input $\x$ and is associated with $\mathcal{Q}_{\x}$, $r$ is the rank of all parameter matrix $\tilde{\W}_i$ in $\mathcal{G}_{p, d}$, $c = \frac{1}{p-r+1}$, and $\mathbb{P} = \mathbb{P}_{(\tilde{\W}^*, S)} \in \mathcal{P}'$. 
\end{theorem}

\begin{theorem} [Sample complexity lower bound for probability of having a positive excess risk] \label{thm:3.4}
	If the sample size $n$ fulfills
	$$ n \in \mathcal{O} \left( \frac{\sum_{i=1}^{r} \log(i) + p \log(2) - \log(4) }{4} \right),$$
	then
	\begin{align}
	\begin{split}
	\inf_{\hat{f} \in \Psi(\F_{p,d})} \sup_{\mathbb{P} \in \mathcal{P}} \xi_2(\hat{f}, \mathbb{P}) 
	\geq \frac{1}{2}
	\end{split}
	\end{align}
\end{theorem}
\begin{proof}
	This proof is similar to Theorem \ref{thm:3.2}. Given any $\hat{f} \in \Psi(\F_{p,d})$ and any dataset $S \in \mathcal{X}^n$, let $\tilde{\W}$ denote $\hat{f}(S)$, then we can write $\tilde{\W} = (\tilde{\w}_0, \tilde{\W}_1, \cdots, \tilde{\W}_d)$, as $\hat{f}$ maps to $\F_{p,d}$.
	Now let $\tilde{\w} := \tilde{\W}_d \tilde{\W}_{d-1} \cdots \tilde{\W}_1 \tilde{\w}_0$, and similarly $\tilde{\w}^* := \tilde{\W}_d^* \tilde{\W}_{d-1}^* \cdots \tilde{\W}_1^* \tilde{\w}_0^*$. Then we claim the event $\{ \tilde{R}(\tilde{\W}) \geq
	\frac{\text{erf} \left( c_1 \right) - \text{erf} \left( c_0 \right) }{2} \}$ is equivalent to the event $\{ \tilde{\W} \neq \tilde{\W}^* \text{ and } \tilde{\w} \neq \tilde{\w}^* \}$, which is shown in the proof of Theorem \ref{thm:3.3} in the appendix. 
	Then, similar to eq \eqref{eq:9} in the proof of Theorem \ref{thm:3.2}, 
	\begin{align}
		\inf_{\hat{f} \in \Psi(\F_{p,d})} \sup_{\mathbb{P} \in \mathcal{P}'} \xi_2(\hat{f}, \mathbb{P}) \geq \inf_{\hat{f} \in \Psi(\G_{p,d})} \sup_{\mathbb{P} \in \mathcal{P}'} \xi_2(\hat{f}, \mathbb{P})
	\end{align}
	and similarly, by basic property of the infimum, 
	$$\inf_{\hat{f} \in \Psi(\F_{p,d})} \sup_{\mathbb{P} \in \mathcal{P}'} \xi_2(\hat{f}, \mathbb{P}) \leq \inf_{\hat{f} \in \Psi(\G_{p,d})} \sup_{\mathbb{P} \in \mathcal{P}'} \xi_2(\hat{f}, \mathbb{P}).$$ 
	Thus
	$$
	\inf_{\hat{f} \in \Psi(\F_{p,d})} \sup_{\mathbb{P} \in \mathcal{P}'} \xi_2(\hat{f}, \mathbb{P}) = \inf_{\hat{f} \in \Psi(\G_{p,d})} \sup_{\mathbb{P} \in \mathcal{P}'} \xi_2(\hat{f}, \mathbb{P})
	$$
	Besides, $
	\sup_{\mathbb{P} \in \mathcal{P}} \xi_2(\hat{f}, \mathbb{P}) \geq \sup_{\mathbb{P} \in \mathcal{P}'} \xi_2(\hat{f}, \mathbb{P})
	$ by basic property of the supremum, therefore, by Theorem \ref{thm:3.3}, if $$ n \leq \sigma^2 \frac{ \sum_{i=1}^{r} \log(i) + p \log(2) - \log(4) }{4},$$
	\begin{align}
	\begin{split}
	\inf_{\hat{f} \in \Psi(\F_{p,d})} \sup_{\mathbb{P} \in \mathcal{P}} \xi_2(\hat{f}, \mathbb{P}) 
	&\geq
	\inf_{\hat{f} \in \Psi(\F_{p,d})} \sup_{\mathbb{P} \in \mathcal{P}'} \xi_2(\hat{f}, \mathbb{P}) \\
	&= \inf_{\hat{f} \in \Psi(\G_{p,d})} \sup_{\mathbb{P} \in \mathcal{P}'} \xi_2(\hat{f}, \mathbb{P}) \\
	&\geq \inf_{\hat{f} \in \Psi(\G_{p,d})} \sup_{\mathbb{P} \in \mathcal{P}'} \frac{1}{2} = \frac{1}{2}
	\end{split}
	\end{align}
\end{proof}

\paragraph{Remark 2} The sample complexity lower bound in Theorem \ref{thm:3.3} is $\Omega(r \log(r) + p)$ and the first term is a factor of $d$ less than that of $\Omega(d r \log(r) + p)$, the sample complexity lower bound in Theorem \ref{thm:3.1}. This is due to identifiability issue with respect to risk $R(\tilde{\W})$, that is, given true parameter $\tilde{\W}^* \in \mathcal{G}_{p, d}$, it is possible that the decoder $\hat{f}$ outputs $\tilde{\W} = \hat{f}(S)$ such that $\tilde{\W} \neq \tilde{\W}^*$ but $\tilde{\W}$ achieves the same risk as $\tilde{\W}^*$, which in turn is due to $\boldsymbol{\mu}_{\x | y; \tilde{\W}} = \boldsymbol{\mu}_{\x | y; \tilde{\W}^*}$. Though the set $\{ \tilde{\W} \in \mathcal{G}_{p, d}: \boldsymbol{\mu}_{\x | y; \tilde{\W}} = \boldsymbol{\mu}_{\x | y; \tilde{\W}^*} \}$ for a given $\tilde{\W}^*$ is small in cardinality, the sample complexity lower bound still shrinks by a factor of $d$. We will further explain this once we explicitly construct $\mathcal{P}'$.

\subsection{Construction of $\mathcal{P}' \subseteq \mathcal{P}$}

Note $\mathcal{P}' = \mathcal{P}_{\tilde{\W}, S}(\{\text{Uniform}(\mathcal{G}_{p, d})\}, \mathcal{P}_{(\x, y) | \tilde{\W}}^{\text{Id}, \mathcal{Q}_{\x}} )$ and $\mathcal{P} = \mathcal{P}_{\tilde{\W}, S}(\mathcal{P}_{\tilde{\W}}(\mathcal{F}_{p, d}), \mathcal{P}_{(\x, y) | \tilde{\W}})$. We will first construct $\mathcal{P}_{(\x, y) | \tilde{\W}}^{\text{Id}, \mathcal{Q}_{\x}}$, then define $\mathcal{G}_{p, d}$. 

\paragraph{Construction of $\mathcal{P}_{(\x, y) | \tilde{\W}}^{\text{Id}, \mathcal{Q}_{\x}} \subseteq \mathcal{P}_{(\x, y) | \tilde{\W}}$} We present a backward data generation process below. Recall that for $\tilde{\W}$ parametrizing a $d$-layer network with input dimension $p$, $\tilde{\W} = (\tilde{\w}_0, \tilde{\W}_1, \cdots, \tilde{\W}_d)$ is a collection of weight matrices and vector. 
\begin{align} \label{eq:19}
\begin{split} 
y &\sim \text{Uniform}\{-1, +1\} \\
\z_0 | y &\sim N(y \tilde{\w}_0, \text{covar} = \sigma^2 \I_{n_0}), \tilde{\w}_0 \in \R^{n_0} \\
\z_1 | \z_0 &\sim N(\tilde{\W}_1 \z_0, \text{covar} = \sigma^2 \I_{n_1}), \tilde{\W}_1 \in \R^{n_1 \times n_0} \\
\z_2 | \z_1 &\sim N(\tilde{\W}_2 \z_1, \text{covar} = \sigma^2 \I_{n_2}), \tilde{\W}_2 \in \R^{n_2 \times n_1} \\
&\cdots \\
\x := \z_d | \z_{d-1} &\sim N(\tilde{\W}_{d}\z_{d-1}, \text{covar} = \sigma^2 \I_{n_d} = \sigma^2 \I_{p}), \\
\tilde{\W}_{d} \in &\R^{n_{d} \times n_{d-1}} = \R^{p \times n_{d-1}}
\end{split}
\end{align}
The above defines a Markov chain from label $y$ to input $\x$, formally described as:
\begin{align}
\begin{split}
\y 
\mapsto \z_0
\mapsto \z_1
\mapsto \cdots 
\mapsto \z_{d-1}
\mapsto \z_{d} =: \x
\end{split}
\end{align}
Here the activation function is linear, that is, $\sigma(z) = z$, and $\mathcal{Q}_{\x}$ consists of multivariate normal distributions of dimension $p$ as described in eq \eqref{eq:19}.

\paragraph{Construction of $\mathcal{G}_{p, d} \subseteq \mathcal{F}_{p, d}$} We state a few assumptions, which define $\mathcal{G}_{p, d}$.

Restriction \textbf{A1}: 
$\mathcal{G}_{p, d}$ can be written as a Cartesian product, that is, $\mathcal{G}_{p, d}  = \mathcal{G}^{(d)} \times \cdots \times \mathcal{G}^{(1)} \times \mathcal{G}^{(0)} $. Let $c \in (0,1)$ be a constant to be determined, we define
\begin{align} \label{eq:11}
\begin{split}
	&\G^{(i)} := \left\{  \begin{bmatrix}
\mathbf{R}_i & \mathbf{0} \\
\mathbf{0} & c \I_{p-r}
\end{bmatrix} \right. : \mathbf{R}_i \in \{ 0,1 \}^{r \times r} \text{ is any rank-} r \\
	&\left.  \text{ permutation matrix on } \R^r \right\}, \forall i \in \{1, \cdots, d\}\\
&\G^{(0)} := \left\{ \frac{\pm1}{\sqrt{2}} \right\} 
\times \left\{ \frac{\pm1}{\sqrt{4}} \right\} 
\times \left\{ \frac{\pm1}{\sqrt{8}} \right\} 
\times \cdots 
\times \left\{ \frac{\pm1}{\sqrt{2^{p-3}}} \right\} \\
&\times \left\{ \frac{\pm1}{\sqrt{2^{p-2}}} \right\} 
\times \left\{ \frac{\pm1}{\sqrt{2^{p-1}}} \right\} 
\times \left\{ \frac{\pm1}{\sqrt{2^{p-1}}} \right\}
\end{split}
\end{align}
Note that $\G^{(0)}$ consists of vectors in $\R^p$ with $\ell_2$ norm $1$. We will write $\mathcal{G}_{p, d, r}$ from now on instead of $ \G_{p, d}$ to show the dependence on rank $r$. Also note that Restriction \textbf{A1} states $n_0 = n_1 = \cdots = n_d = p$, i.e., the parametrized deep networks have equal-sized layers.

It is apparent that $| \G^{(0)} | = 2^p$. It is also easy to see that the size of $\G^{(i)}$ for $i \in \{1, \cdots, d\}$ is the number of $r$-permutations, meaning that $| \G^{(d)} | = \cdots = | \G^{(1)} | = r!$. Thus
\begin{align} \label{eq:12}
	\begin{split}
		| \G_{p, d, r} | 
		&= | \G^{(d)} | \times \cdots \times | \G^{(1)} | \times | \G^{(0)} | 
		= \left( r! \right)^{d}  \cdot 2^p \\
		\implies& \log | \G_{p, d, r} | = d \left( \sum_{i=1}^{r} \log(i) \right) + p \log(2)
	\end{split}
\end{align}

Now we impose a second restriction on the value of $c \in (0,1)$, and the reason of such choice of $c$ is given in the proof in the appendix. 

Restriction \textbf{A2}: $c = \frac{1}{p-r+1}$.

These two restrictions define the subset $\G_{p, d, r}$ of $\F_{p, d}$. 

\subsection{Proof Sketch}

We first state a few results about the joint distribution of $(\z_0, \z_1, \z_2, \cdots, \z_d) | y ; \tilde{\W}$ and the marginal distribution of $(\z_d | y ; \tilde{\W}) = (\x | y ; \tilde{\W})$ with respect to general $\tilde{\W} \in \F_{p, d}$, then state some intermediate information theoretic results when we restrict $\tilde{\W}$ to $\G_{p, d, r}$, where $\G_{p, d, r} \subseteq \F_{p, d}$, and provide references to some information-theoretic tools that we used for our theorems. We leave all the detailed proofs to the appendix.

\subsubsection{Proof sketch of Theorem \ref{thm:3.1}}

\begin{lemma} \label{lemma:3.3}
	For any $\tilde{\W} \in \F_{p, d}$, the joint distribution of $(\z_0, \z_1, \cdots, \z_d) | y ; \tilde{\W}$ is multivariate normal with mean $(y \tilde{\w}_0, y \tilde{\W}_1 \tilde{\w}_0, \cdots, y \tilde{\W}_{d} \tilde{\W}_{d-1} \cdots \tilde{\W}_2 \tilde{\w}_1 \tilde{\w}_0)$ and precision matrix $\kappa^{(d)}$ of dimension $(\sum_{i=0}^d n_i) \times (\sum_{i=0}^d n_i)$ being a block matrix with $(d+1) \times (d+1)$ blocks. Let $\kappa^{(d)}_{i,j}$ represent the $(i,j)-$th block of $\kappa^{(d)}$ $\left( 0 \leq i, j \leq d \right)$, and let $\Tilde{\mathbf{\Sigma}}_i := \left( \sigma^2 \I_{n_i} \right)^{-1}$ be the precision matrix of $\z_i$ as in \eqref{eq:19}, then $\kappa^{(d)}$ has tri-diagonal blocks and 
	\begin{align} \label{eq:47}
	\begin{split}
	\text{For } 0 \leq i \leq d-1 :& \kappa^{(d)}_{ii} = \Tilde{\mathbf{\Sigma}}_i + \Tilde{\W}_{i+1}^{\top} \Tilde{\mathbf{\Sigma}}_{i+1} \Tilde{\W}_{i+1}, \\ &\kappa^{(d)}_{i, i+1} = -(\Tilde{\mathbf{\Sigma}}_{i+1} \Tilde{\W}_{i+1})^{\top} \\
	&\kappa^{(d)}_{i+1, i} = -(\Tilde{\mathbf{\Sigma}}_{i+1} \Tilde{\W}_{i+1}) \\
	\text{For } i = d :& \kappa^{(d)}_{dd} = \Tilde{\mathbf{\Sigma}}_d \\
	\text{For all other }i, j:& \kappa^{(d)}_{i,j} = [0]_{n_i \times n_j}
	\end{split}
	\end{align}
\end{lemma}

Then we find the marginal distribution of  $(\z_d | y ; \tilde{\W}) = (\x | y ; \tilde{\W})$ with a recursively defined covariance matrix.

\begin{lemma} [Marginal distribution of $\x | y ; \tilde{\W}$] \label{lemma:3.4} 
	For any $\tilde{\W} \in \F_{p, d}$, $\x | y; \tilde{\W}$ is multivariate normal with mean $y \tilde{\W}_{d} \tilde{\W}_{d-1} \cdots \tilde{\W}_1 \tilde{\w}_0$ and covariance matrix $\sigma^2 \left( \I_p - \M_d(\tilde{\W}) \right)^{-1}$, which has the following recursive definition:
	\begin{align}
	\begin{split}
	\M_1(\tilde{\W}) := & \tilde{\W}_1 \left( \I_{n_0} + \tilde{\W}_1^{\top} \tilde{\W}_1 \right)^{-1} \tilde{\W}_1^{\top} \\
	\M_2(\tilde{\W}) := & \tilde{\W}_2 \left( \I_{n_1} + \tilde{\W}_2^{\top} \tilde{\W}_2 - \M_1(\tilde{\W}) \right)^{-1} \tilde{\W}_2^{\top} \\
	& \cdots \\
	\M_d(\tilde{\W}) := & \tilde{\W}_d \left( \I_{n_{d-1}} + \tilde{\W}_{d}^{\top} \tilde{\W}_{d} - \M_{d-1}(\tilde{\W}) \right)^{-1} \tilde{\W}_d^{\top}
	\end{split}
	\end{align}
	where the $\tilde{\W}$ inside $\M_d(\tilde{\W})$ emphasizes the dependency of the covariance matrix on the choice of parameter  $\tilde{\W}$.
\end{lemma}

\paragraph{Remark 3} Now we have the formula for the predictor,  $\boldsymbol{\mu}_{\x | y; \tilde{\W}} = \tilde{\W}_{d} \tilde{\W}_{d-1} \cdots \tilde{\W}_1 \tilde{\w}_0$, the risk of $\tilde{\W}$ given true parameter $\tilde{\W}^*$, $R(\tilde{\W}) := P_{(\x, y) ; \tilde{\W}^*} \left[ \x^{\top} \boldsymbol{\mu}_{\x | y; \tilde{\W}} \leq 0 \right]$ is easy to calculate. Previously in Remark 2 we mentioned an identifiability issue. This issue is present for general $\tilde{\W} \in \F_{p, d}$, thus also for $\tilde{\W} \in \G_{p, d, r}$, because it is possible that for some $\tilde{\W} \in \F_{p, d}$ that there exists $\tilde{\W}' \neq \tilde{\W}$ in $\F_{p,d}$ such that $\tilde{\W}_{d} \tilde{\W}_{d-1} \cdots \tilde{\W}_1 \tilde{\w}_0 = \tilde{\W}'_{d} \tilde{\W}'_{d-1} \cdots \tilde{\W}'_1 \tilde{\w}'_0$ \textit{and} $\sigma^2 \left( \I_p - \M_d(\tilde{\W}) \right)^{-1} = \sigma^2 \left( \I_p - \M_d(\tilde{\W}') \right)^{-1}$, meaning that $\x | y; \tilde{\W} \stackrel{d}{=} \x | y; \tilde{\W}'$, where $\stackrel{d}{=}$ means 'identical in distribution'. In addition, recall $y \sim \text{Uniform}\{-1, +1\}$ as in eq \eqref{eq:19}, thus $(\x,  y); \tilde{\W} \stackrel{d}{=} (\x,  y); \tilde{\W}'$, thus $R(\tilde{\W}) = R(\tilde{\W}')$. This implies there would be $\tilde{\W}'$ different from the true $\tilde{\W}^*$ achieving the same risk as $\tilde{\W}^*$. 

Our choice of $\G_{p, d, r}$ inevitably has this issue as well, as it contains parameter matrices consisting of a permutation block and a scaled diagonal block. We overcome this issue with the distance-based Fano's inequality in Theorem~\ref{thm:3.3}. 

Now we provide a KL-divergence upper bound between distributions from $\mathcal{P}_{(\x, y) | \tilde{\W}}^{\text{Id}, \mathcal{Q}_{\x}}$ with our choice of $\G_{p, d, r}$, which in turn gives an upper bound on $\mathbb{I}(\tilde{\W}; S)$, the mutual information between hypothesis $\tilde{\W}$ and the dataset $S$, a key quantity in Fano's inequality. 

\begin{lemma} [Upper bound on KL divergence between $(\x,y) ; \tilde{\W}$ and $(\x,y) ; \tilde{\W}^{\prime}$ for $\tilde{\W}, \tilde{\W}^{\prime} \in \G_{p, d, r}$] \label{lemma:3.5}
	Under Restriction \textbf{A1}, for any $\tilde{\W}, \tilde{\W}^{\prime} \in \G_{p, d, r}$ and $\tilde{\W} \neq \tilde{\W}^{\prime}$, we have that
	$$\mathbb{KL}(\mathbb{P}_{(\x, y);\tilde{\W}} || \mathbb{P}_{(\x, y);\tilde{\W}^{\prime}} ) \leq \frac{2}{\sigma^2}$$
	where $\mathbb{P}_{(\x, y);\tilde{\W}}$ is the joint distribution of $(\x, y)$ parametrized by $\tilde{\W} = (\tilde{\w}_0, \tilde{\W}_1, \tilde{\W}_2, \cdots, \tilde{\W}_{d-1}, \tilde{\W}_{d})$, $\mathbb{P}_{(\x, y);\tilde{\W}^{\prime}}$ is the joint distribution of $(\x, y)$ parametrized by $\tilde{\W}^{\prime} = (\tilde{\w}_0^{\prime}, \tilde{\W}_1^{\prime}, \tilde{\W}_2^{\prime}, \cdots, \tilde{\W}_{d-1}^{\prime}, \tilde{\W}_{d}^{\prime})$ under $\mathcal{P}_{(\x, y) | \tilde{\W}}^{\text{Id}, \mathcal{Q}_{\x}}$ outlined in eq \eqref{eq:19}, and $\tilde{\W} \neq \tilde{\W}^{\prime}$ means $(\tilde{\w}_0, \tilde{\W}_1, \tilde{\W}_2, \cdots, \tilde{\W}_{d-1}, \tilde{\W}_{d}) \neq (\tilde{\w}_0^{\prime}, \tilde{\W}_1^{\prime}, \tilde{\W}_2^{\prime}, \cdots, \tilde{\W}_{d-1}^{\prime}, \tilde{\W}_{d}^{\prime})$, and $\sigma^2$ is the constant for the diagonal covariance as in eq \eqref{eq:19}.
\end{lemma}
This lemma gives us a constant upper bound on the KL divergence, which is relatively small in order in the sense that it does not grow with $p$, $d$ or $r$. This will in turn gives us an upper bound on $\mathbb{I}(\tilde{\W}; S)$ that is linear in $n$ and not dependent on any of $p$, $r$, or $d$. Thus, when applying Fano's inequality, the dependency of the sample complexity lower bound on $p, d, r$ comes solely from the size of $\G_{p, d, r}$, and in fact the sample complexity lower bound in Theorem \ref{thm:3.1} is of the same order as $\log | \G_{p, d, r} | = d \left( \sum_{i=1}^{r} \log(i) \right) + p \log(2)$.

\subsubsection{Proof sketch of Theorem \ref{thm:3.3}}

We first reproduce the distance-based Fano inequality, proposed in (Duchi et al, 2013), then explain how this distance-based Fano inequality applies to the sample complexity lower bound of having a positive excess risk. 

\begin{lemma} [Distance-Based Fano's Inequality for Discrete Problem] \label{lemma:3.6}
	Consider any Markov chain $V \to X \to \hat{V}$, where the random variable $V \sim \text{Uniform}(\mathcal{V})$ with $2 \leq |\mathcal{V}| < \infty $, and a symmetric function $\rho: \mathcal{V} \times \mathcal{V} \to \mathbb{R}$ (e.g. a (semi)metric on the space $\mathcal{V}$), then for a given scalar $t \geq 0$, define the maximum and minimum neighborhood sizes at radius t, 
	\begin{align}
	\begin{split}
	&N_t^{\max} := \max_{v \in \mathcal{V}} \{ \text{card} \{ v' \in \mathcal{V}: \rho(v, v') \leq t \} \}, \\
	&N_t^{\min} := \min_{v \in \mathcal{V}} \{ \text{card} \{ v' \in \mathcal{V}: \rho(v, v') \leq t \} \}
	\end{split}
	\end{align}
	Then, if $|\mathcal{V}| - N_t^{\min} > N_t^{\max}$, we have
	\begin{align}
	P(\rho(\hat{V}, V) > t) \geq 1 - \frac{\mathbb{I}(V;X) + \log 2 }{\log \frac{|\mathcal{V}|}{N_t^{\max}}}.
	\end{align}
\end{lemma}

To apply this inequality, we will let $\mathcal{V} = \G_{p, d, r}$, $X = S$, and let $\rho(v, v')$ be $\rho(\tilde{\W}, \tilde{\W}') := \mathds{1} \{\tilde{\W} \neq \tilde{\W}'\} + \mathds{1} \{\tilde{\w} \neq \tilde{\w}'\}$ for any $\tilde{\W}, \tilde{\W}' \in \G_{p, d, r}$, where $\tilde{\w} := \tilde{\W}_{d} \tilde{\W}_{d-1} \cdots \tilde{\W}_1 \tilde{\w}_0$. In fact, $\rho$ is a metric on $\G_{p, d, r}$, and this is shown in the proof of Theorem \ref{thm:3.3} in the appendix. We choose $t = 1$ and show that $N_1^{\max} = N_1^{\min} = (r!)^{d-1}$, which is why we lose a factor of $d$ in the sample complexity lower bound in Theorem \ref{thm:3.3}. The event $\{ \rho(\tilde{\W}, \tilde{\W}^*) > 1 \}$ where $\tilde{\W}^*$ is the true parameter is equivalent to $\{ \tilde{\W} \neq \tilde{\W}^* \text{ and } \tilde{\w} \neq \tilde{\w}^*\}$, and this event is sufficient for $R(\tilde{\W}) > R(\tilde{\W}^*)$, which is shown in the proof of Theorem \ref{thm:3.3} in the appendix. 

\subsection{Linear approximation of the excess risk lower bound in Theorem \ref{thm:3.3}}

In Theorem \ref{thm:3.3} we have the probability of 
\begin{align} \label{eq:17}
	R(\tilde{\W}) - R(\tilde{\W}^*) 
	\geq 
	\frac{\text{erf} \left( c_1 \right) - \text{erf} \left( c_0 \right) }{2} 
\end{align}
greater than $1/2$, where $c_0$ and $c_1$ are defined in \eqref{eq:14} and \eqref{eq:15}. 

The excess risk lower bound in eq \eqref{eq:17} is a difference of erf functions and has no simple formula. However, we still manage to analyze this lower bound via a linear approximation of the erf functions, as the numerator of both $c_0$ and $c_1$ is less than $1$ and denominator is $\sigma \sqrt{2 \left[(d+1) \left( 1 - \frac{1}{2^r} \right) + \left( \frac{1-c^{2(d+1)}}{1 - c^2} \right) \left( \frac{ c^{2d}}{2^r} \right) \right]}$, which is very close to $\sigma \sqrt{2d}$, with $\sigma$ being a constant reflecting the variance of input data $\x$ and small $c$, where $c$ is defined in Restriction \textbf{A2}. The erf function is approximately linear around zero, thus we use its derivative at zero, $\frac{d \text{erf}(x)}{dx} \rvert _{x=0} = 2 / \sqrt{\pi}$, for a linear approximation of the excess risk lower bound. The detailed analysis is in the appendix.

\section{Comparison with Existing Upper Bounds} \label{sec:4}

We mainly compare with the four existing upper bounds on Rademacher complexity and generalization error. All those existing bounds are discussed based on the assumption that the input $\x$ is bounded, while our model has unbounded $\x$ as it follows a Gaussian distribution. Therefore we ignore all terms involving norm of $\x$ when comparing the bounds.

Theorem 1 in (Neyshabur et al, 2015) gives an upper bound on the Rademacher complexity of a class of $d$-layer network with ReLU activation. Further imposing our Restriction \textbf{A1}, which assumes the parameter matrices consist of a block of permutation matrix of rank $r$ and a scaled diagonal block, the upper bound in Theorem 1 in (Neyshabur et al, 2015), ignoring the term $\max_i \| \x_i \|_{p^*}^2$, becomes
\begin{align}
O \left( \sqrt{ \left(2 p^{1/6} \right)^{2(d-1)} / n
} \right)
\end{align}
which implies that the sample size $n$ of order $O(p^{d/6})$ is sufficient for a decaying Rademacher complexity.

Theorem 1.1 in (Bartlett et al, 2017) gives an upper bound on the prediction error probability for a given $d$-layer network with ReLU activation. The upper bound includes a spectral density $R_{\mathcal{A}}$, where $\mathcal{A}$ corresponds to $\tilde{\W}$ in our analysis. Under our restrictions \textbf{A1} and \textbf{A2}, this upper bound becomes 
\begin{align} \label{eq:20}
O \left(
\frac{R_{\mathcal{A}}}{\gamma n} \ln(p) + \sqrt{\frac{1/\delta}{n}}
\right)
\end{align}
where $R_{\mathcal{A}} = \left( \prod_{i=1}^d \lambda_{\max}(\tilde{\W}_i) \right) \cdot \left( \sum_{i=1}^d \left( \frac{ \| \tilde{\W}_i^{\top} \|_{2,1} }{\lambda_{\max}(\tilde{\W}_i)} \right)^{2/3} \right)^{3/2} = r d^{3/2}$. Thus this upper bound states that $n \in O(d^{\frac{3}{2}} \cdot r \cdot \ln(p))$ is sufficient for a decreasing error probability. 

Theorem 1 in (Neyshabur et al, 2018) gives an upper bound on the generalization error for a given $d$-layer network with ReLU activation. The upper bound holds with probability $1 - \delta$:
\begin{align} \label{eq:21}
O \left(
\sqrt{
	\frac{d^2 p \ln(d p) d r + \ln \frac{d n}{\delta}}{\gamma^2 n}
}
\right)
\end{align}
which implies that $n \in O\left(r d^3 p \ln(dp) \right)$ is sufficient for a decreasing error probability.

There has also been a size-indepepdent upper bound on the Rademacher complexity, i.e., the width of the network, which is $p$ in our analysis, does not appear in the bound. Theorem 5 in (Golowich et al, 18) gives such a size-independent upper bound for a given $d$-layer network with ReLU activation.
\begin{align}
O \left(
r^{d/2} \cdot \min 
\left\{
\sqrt{ \log \left( r^{d/2} \right) / \sqrt{n}, \sqrt{d / n} }
\right\}
\right)
\end{align}
If the minimum evaluates to the first term, $n \in O\left(r^{2d} \cdot d^2 (\log r)^2 \right)$ is sufficient a decreasing Rademacher complexity. If the minimum evaluates to the second term, $n \in O\left(r^{d} \cdot d \right)$ is sufficient.

Meanwhile, our sample complexity lower bound in Theorem \ref{thm:3.1} suggests that at least $n \in \Omega(d r \log(r) + p)$ samples are needed to exactly recover the parameter matrices and vector, and in Theorem \ref{thm:3.3} we show $n \in \Omega(r \log(r) + p)$ samples are needed to recover any parameter having the same prediction risk as the truth. 

Moreover, if we consider full-rank matrices, i.e. let $r = p$, or simply let $r \propto p$, then our sample complexity lower bounds become $\Omega(d p \log(p))$ and $\Omega(p \log(p))$. Among the existing bounds, the upper bounds provided by (Bartlett et al, 2017) and (Neyshabur et al, 2018) are the closest to ours. The former (eq \eqref{eq:20}) has order $O(d^{\frac{3}{2}} \cdot p \cdot \ln(p))$ when $r \propto p$, and the latter (eq \eqref{eq:21}) has order $O\left(d^3 p^2 \ln(dp) \right)$. The former differs from our lower bound in Theorem \ref{thm:3.1} by only a factor of $d^{1/2}$, while the latter is greater than ours by a factor of $d^2 p$ and has an additional term $d^3 p^2 \ln(d)$. 

Lastly, our sample complexity lower bounds, $\Omega(d r \log(r) + p)$ for exact recovery of truth, and $\Omega(r \log(r) + p)$ for recoverying parameters with same risk as the truth, are intuitive in the sense that a network should be harder to learn when the input dimension is large, a deeper network should require more samples to train, and a network parametrized by matrices with higher rank should also require more samples.

\clearpage
\appendix
\onecolumn

\noindent The supplementary material contains proofs of the lemmas and theorems in the main text.

\section{Detailed proofs} 

We first present some intermediate results for lemmas stated in Section 3.4.1, specifically, results about the joint distribution of $(\z_0, \z_1, \cdots, \z_d) | y; \tilde{\W}$ (Lemma \ref{lemma:4.1} and \ref{lemma:4.3}), the marginal distribution of $(\z_d | y; \tilde{\W}) =(\x | y; \tilde{\W})$ (Lemma \ref{lemma:4.4}) for $d = 1$ and $d = 2$. We also present proofs for lemmas relevant to the corresponding KL-divergence upper bound (Lemma \ref{lemma:4.2} and \ref{lemma:4.5}). In section 1.3 below, we prove Lemma 3.5, 3.6 and 3.7 in Section 3.4.1 of the main text. In Section 1.4, we prove our main theorems. In Section 1.5, we discuss the linear approximation of the excess risk in Theorem 3.3 in the main text.

\subsection{One hidden layer}

We first consider a 1-layer model, which is described below, in a manner similar to eq (19) in the main text:

\begin{align} \label{eq:32}
\begin{split}
y &\sim \text{Uniform}\{-1, +1\} \\
\z_0 | y &\sim N(y \tilde{\w}_0, \text{covar} = \sigma^2 \I_{n_0}), \tilde{\w}_0 \in \R^{n_0} \\
\x := \z_1 | \z_0 &\sim N(\tilde{\W}_1 \z_0, \text{covar} = \sigma^2 \I_{n_1} = \sigma^2 \I_{p}), \tilde{\W}_1 \in \R^{n_1 \times n_0} = \R^{p \times n_0} \\
\end{split}
\end{align}

Here we first find the joint distribution of $(\z_0, \z_1) | y$ and then show an upper bound on the KL divergence of interest. These results for the 1-layer network will be a building block for our analysis of the 2-hidden layer network. We use $\x$ and $\z_1$ interchangeably in this subsection, as $(\z_1 | y; \tilde{\W}) = (\x | y; \tilde{\W})$. 

We first find the joint distribution of $(\z_0, \z_1) | y; \tilde{\W}$, where $\tilde{\W} = (\tilde{\w}_0, \tilde{\W}_1)$. 
\begin{lemma} [Joint distribution of $(\z_0, \z_1) | y; \tilde{\W}$] ~\label{lemma:4.1}
	The random variable $(\z_0, \z_1) | y; \tilde{\W}$ is normally distributed with mean $(y \tilde{\w}_0, y \tilde{\W}_1 \tilde{\w}_0)$ and covariance matrix 
	\begin{align}
	\begin{bmatrix}
	\Tilde{\mathbf{\Sigma}}_0 + \tilde{\W}_1^{\top} \Tilde{\mathbf{\Sigma}}_1 \tilde{\W}_1 & -(\Tilde{\mathbf{\Sigma}}_1 \tilde{\W}_1)^{\top} \\
	-(\Tilde{\mathbf{\Sigma}}_1 \tilde{\W}_1) & \Tilde{\mathbf{\Sigma}}_1 
	\end{bmatrix} ^{-1} 
	\end{align}
	where $\Tilde{\mathbf{\Sigma}}_i = \left( \sigma^2 \I_{n_i} \right)^{-1}$ is the precision matrix of $\z_i$ as defined in \eqref{eq:32}. 
\end{lemma}

\begin{proof}
	We just need to show that the exponent of the proposed distribution in the claim above matches the exponent in $p(\z_0 | y ; \tilde{\W}) p(\z_1 | \z_0; \tilde{\W})$ and then check the positive-definiteness of the precision matrix.
	\begin{align}
	\begin{split}
	p(\z_0, \z_1 | y) =& p(\z_0 | y) p(\z_1 | \z_0) \\
	\propto& \exp \left(-\frac{1}{2} (\z_0 - y \tilde{\w}_0)^{\top} \Tilde{\mathbf{\Sigma}}_0 (\z_0 - y \tilde{\w}_0) \right)  
	\exp \left(-\frac{1}{2} (\z_1 - \tilde{\W}_1 \z_0)^{\top} \Tilde{\mathbf{\Sigma}}_1 (\z_1 - \tilde{\W}_1 \z_0) \right) \\
	\end{split}
	\end{align}
	
	and the exponent, ignoring the $-\frac{1}{2}$ factor, becomes
	\begin{align} \label{eq:26}
	\begin{split}
	&\z_0^{\top} \Tilde{\mathbf{\Sigma}}_0 \z_0 - 2 \z_0^{\top} \Tilde{\mathbf{\Sigma}}_0 (y \tilde{\w}_0) + (y \tilde{\w}_0)^{\top} \Tilde{\mathbf{\Sigma}}_0 (y \tilde{\w}_0) 
	+ \z_1^{\top} \Tilde{\mathbf{\Sigma}}_1 \z_1 - 2 \z_1^{\top} \Tilde{\mathbf{\Sigma}}_1 (\tilde{\W}_1 \z_0) + (\tilde{\W}_1 \z_0)^{\top} \Tilde{\mathbf{\Sigma}}_1 (\tilde{\W}_1 \z_0) \\
	=& \z_0^{\top} (\Tilde{\mathbf{\Sigma}}_0 + \tilde{\W}_1^{\top} \Tilde{\mathbf{\Sigma}}_1 \tilde{\W}_1) \z_0 
	- 2 y \z_0^{\top} \Tilde{\mathbf{\Sigma}}_0 \tilde{\w}_0) 
	- 2 \z_1^{\top} \Tilde{\mathbf{\Sigma}}_1 \tilde{\W}_1 \z_0 + \tilde{\w}_0^{\top} \Tilde{\mathbf{\Sigma}}_0 \tilde{\w}_0 \\
	\end{split}
	\end{align}
	
	while the density of proposed distribution is proportional to the exponential of 
	\begin{align} \label{eq:27}
	\begin{split}
	-\frac{1}{2} 
	\begin{bmatrix}
	\z_0 - y \tilde{\w}_0 \\
	\z_1 - y \tilde{\W}_1 \tilde{\w}_0
	\end{bmatrix}^{\top} 
	\begin{bmatrix}
	\Tilde{\mathbf{\Sigma}}_0 + \tilde{\W}_1^{\top} \Tilde{\mathbf{\Sigma}}_1 \tilde{\W}_1 & -(\Tilde{\mathbf{\Sigma}}_1 \tilde{\W}_1)^{\top} \\
	-(\Tilde{\mathbf{\Sigma}}_1 \tilde{\W}_1) & \Tilde{\mathbf{\Sigma}}_1 
	\end{bmatrix}        
	\begin{bmatrix}
	\z_0 - y \tilde{\w}_0 \\
	\z_1 - y \tilde{\W}_1 \tilde{\w}_0
	\end{bmatrix}
	\end{split}
	\end{align}
	
	ignoring the factor of $-1/2$, \eqref{eq:27} evaluates to
	\begin{align}
	\begin{split}
	&\begin{bmatrix}
	\z_0 - y \tilde{\w}_0 \\
	\z_1 - y \tilde{\W}_1 \tilde{\w}_0
	\end{bmatrix}^{\top} 
	\begin{bmatrix}
	(\Tilde{\mathbf{\Sigma}}_0 + \tilde{\W}_1^{\top} \Tilde{\mathbf{\Sigma}}_1 \tilde{\W}_1) (\z_0 - y \tilde{\w}_0) -(\Tilde{\mathbf{\Sigma}}_1 \tilde{\W}_1)^{\top} (\z_1 - y \tilde{\W}_1 \tilde{\w}_0) \\
	-(\Tilde{\mathbf{\Sigma}}_1 \tilde{\W}_1) (\z_0 - y \tilde{\w}_0) + \Tilde{\mathbf{\Sigma}}_1 (\z_0 - y \tilde{\W}_1 \tilde{\w}_0) 
	\end{bmatrix} \\
	= &(\z_0 - y \tilde{\w}_0)^{\top} (\Tilde{\mathbf{\Sigma}}_0 + \tilde{\W}_1^{\top} \Tilde{\mathbf{\Sigma}}_1 \tilde{\W}_1) (\z_0 - y \tilde{\w}_0) 
	- (\z_0 - y \tilde{\w}_0)^{\top} (\Tilde{\mathbf{\Sigma}}_1 \tilde{\W}_1)^{\top} (\z_1 - y \tilde{\W}_1 \tilde{\w}_0) \\
	&- (\z_1 - y \tilde{\W}_1 \tilde{\w}_0)^{\top} (\Tilde{\mathbf{\Sigma}}_1 \tilde{\W}_1) (\z_0 - y \tilde{\w}_0) 
	+ (\z_1 - y \tilde{\W}_1 \tilde{\w}_0)^{\top} \Tilde{\mathbf{\Sigma}}_1 (\z_1 - y \tilde{\W}_1 \tilde{\w}_0) \\
	=& [ \z_0^{\top} \Tilde{\mathbf{\Sigma}}_0 \z_0 - 2 \z_0^{\top} \Tilde{\mathbf{\Sigma}}_0 (y \tilde{\w}_0) + (y \tilde{\w}_0)^{\top} \Tilde{\mathbf{\Sigma}}_0 (y \tilde{\w}_0) + 
	\z_0^{\top} \tilde{\W}_1^{\top} \Tilde{\mathbf{\Sigma}}_1 \tilde{\W}_1 \z_0 - 2 \z_0^{\top} \tilde{\W}_1^{\top} \Tilde{\mathbf{\Sigma}}_1 \tilde{\W}_1 (y \tilde{\w}_0) \\
	& + (y \tilde{\w}_0)^{\top} \tilde{\W}_1^{\top} \Tilde{\mathbf{\Sigma}}_1 \tilde{\W}_1 (y \tilde{\w}_0) ] 
	- 2 [\z_0^{\top} \tilde{\W}_1^{\top} \Tilde{\mathbf{\Sigma}}_1 \z_1 - y \z_0^{\top} \tilde{\W}_1^{\top} \Tilde{\mathbf{\Sigma}}_1 \tilde{\W}_1 \tilde{\w}_0 - y \tilde{\w}_0^{\top} \tilde{\W}_1^{\top} \Tilde{\mathbf{\Sigma}}_1 \z_1 \\
	& + \tilde{\w}_0^{\top} \tilde{\W}_1^{\top} \Tilde{\mathbf{\Sigma}}_1 \tilde{\W}_1 \tilde{\w}_0] 
	+ [\z_1^{\top} \Tilde{\mathbf{\Sigma}}_1 \z_1 - 2 y \z_1^{\top} \Tilde{\mathbf{\Sigma}}_1 \tilde{\W}_1 \tilde{\w}_0 + \tilde{\w}_0^{\top} \tilde{\W}_1^{\top} \Tilde{\mathbf{\Sigma}}_1 \tilde{\W}_1 \tilde{\w}_0] \\
	=& \z_0^{\top} (\Tilde{\mathbf{\Sigma}}_0 + \tilde{\W}_1^{\top} \Tilde{\mathbf{\Sigma}}_1 \tilde{\W}_1) \z_0 
	+ \z_0^{\top} [-2 y \Tilde{\mathbf{\Sigma}}_0 \tilde{\w}_0 
	- 2 y \tilde{\W}_1^{\top} \Tilde{\mathbf{\Sigma}}_1 \tilde{\W}_1 \tilde{\w}_0 
	+ 2 y \tilde{\W}_1^{\top} \Tilde{\mathbf{\Sigma}}_1 \tilde{\W}_1 \tilde{\w}_0 ] \\
	&+ \z_0^{\top} [-2 \tilde{\W}_1^{\top} \Tilde{\mathbf{\Sigma}}_1] \z_1 + \z_1^{\top} [2 y \Tilde{\mathbf{\Sigma}}_1 \tilde{\W}_1 \tilde{\w}_0 
	- 2 y \Tilde{\mathbf{\Sigma}}_1 \tilde{\W}_1 \tilde{\w}_0] \\
	& + [\tilde{\w}_0^{\top} \Tilde{\mathbf{\Sigma}}_0  \tilde{\w}_0 + \tilde{\w}_0^{\top} \tilde{\W}_1^{\top} \Tilde{\mathbf{\Sigma}}_1 \tilde{\W}_1 \tilde{\w}_0 
	-2 \tilde{\w}_0^{\top} \tilde{\W}_1^{\top} \Tilde{\mathbf{\Sigma}}_1 \tilde{\W}_1 \tilde{\w}_0 + \tilde{\w}_0^{\top} \tilde{\W}_1^{\top} \Tilde{\mathbf{\Sigma}}_1 \tilde{\W}_1 \tilde{\w}_0] \\
	=& \z_0^{\top} (\Tilde{\mathbf{\Sigma}}_0 + \tilde{\W}_1^{\top} \Tilde{\mathbf{\Sigma}}_1 \tilde{\W}_1) \z_0 
	+ \z_0^{\top} [-2 y \Tilde{\mathbf{\Sigma}}_0 \tilde{\w}_0] + \z_0^{\top} [-2 \tilde{\W}_1^{\top} \Tilde{\mathbf{\Sigma}}_1] \z_1 + \tilde{\w}_0^{\top} \Tilde{\mathbf{\Sigma}}_0 \tilde{\w}_0 
	\end{split}
	\end{align}
	which is exactly same as in \eqref{eq:26}. Now we check the positive-definiteness of the precision matrix $\begin{bmatrix}
	\Tilde{\mathbf{\Sigma}}_0 + \tilde{\W}_1^{\top} \Tilde{\mathbf{\Sigma}}_1 \tilde{\W}_1 & -(\Tilde{\mathbf{\Sigma}}_1 \tilde{\W}_1)^{\top} \\
	-(\Tilde{\mathbf{\Sigma}}_1 \tilde{\W}_1) & \Tilde{\mathbf{\Sigma}}_1 
	\end{bmatrix} $ in the claim. Note $\Tilde{\mathbf{\Sigma}}_0$ and $\Tilde{\mathbf{\Sigma}}_1$ are both positive-definite, as they are precision matrices of normal distribution. Consider any vector $\x \neq \mathbf{0}$ and write $\x = (\x_0, \x_1)$ where $\x_0 \in \R^{n_0}$ and $\x_1 \in \R^{n_1}$, then
	\begin{align} \label{eq:29}
	\begin{split}
	&\begin{bmatrix}
	\x_0 \\
	\x_1
	\end{bmatrix} ^{\top}
	\begin{bmatrix}
	\Tilde{\mathbf{\Sigma}}_0 + \tilde{\W}_1^{\top} \Tilde{\mathbf{\Sigma}}_1 \tilde{\W}_1 & -(\Tilde{\mathbf{\Sigma}}_1 \tilde{\W}_1)^{\top} \\
	-(\Tilde{\mathbf{\Sigma}}_1 \tilde{\W}_1) & \Tilde{\mathbf{\Sigma}}_1 
	\end{bmatrix}
	\begin{bmatrix}
	\x_0 \\
	\x_1
	\end{bmatrix} \\
	=& \x_0^{\top} \Tilde{\mathbf{\Sigma}}_0 \x_0^{\top} 
	+ \x_0^{\top} \tilde{\W}_1^{\top} \Tilde{\mathbf{\Sigma}}_1 \tilde{\W}_1 \x_0 
	- 2 \x_0^{\top} \Tilde{\mathbf{\Sigma}}_1 \tilde{\W}_1 \x_1 
	+ \x_1^{\top} \Tilde{\mathbf{\Sigma}}_1 \x_1 \\
	=& \| \Tilde{\mathbf{\Sigma}}_0^{1/2} \x_0 \|_2^2 
	+ \| \Tilde{\mathbf{\Sigma}}_1^{1/2} \tilde{\W}_1 \x_0 \|_2^2 
	+ \| \Tilde{\mathbf{\Sigma}}_1^{1/2} \x_1^{\top} \|_2^2 
	- 2 \x_0^{\top} \Tilde{\mathbf{\Sigma}}_1 \tilde{\W}_1 \x_1 \\
	\geq & \| \Tilde{\mathbf{\Sigma}}_0^{1/2} \x_0 \|_2^2 
	+ \| \Tilde{\mathbf{\Sigma}}_1^{1/2} \tilde{\W}_1 \x_0 \|_2^2 + \| \Tilde{\mathbf{\Sigma}}_1^{1/2} \x_1^{\top}  \|_2^2
	- 2 \| \Tilde{\mathbf{\Sigma}}_1^{1/2} \tilde{\W}_1 \x_0 \|_2 \| \Tilde{\mathbf{\Sigma}}_2^{1/2} \x_1 \|_2 \\
	=& \| \Tilde{\mathbf{\Sigma}}_0^{1/2} \x_0 \|_2^2
	\end{split}
	\end{align}
	by the Cauchy-Schward inequality. It is easy to see that there are two cases: either $\x_0 = \mathbf{0}$ or not. If $\x_0 \neq \mathbf{0}$, then the RHS of \eqref{eq:29} is positive. Otherwise, if $\x_0 = \mathbf{0}$, then from \eqref{eq:29} before the Cauchy-Schwarz inequality is applied, we can see 
	\begin{align}
	&\begin{bmatrix}
	\mathbf{0} \\
	\x_1
	\end{bmatrix} ^{\top}
	\begin{bmatrix}
	\Tilde{\mathbf{\Sigma}}_0 + \tilde{\W}_1^{\top} \Tilde{\mathbf{\Sigma}}_1 \tilde{\W}_1 & -(\Tilde{\mathbf{\Sigma}}_1 \tilde{\W}_1)^{\top} \\
	-(\Tilde{\mathbf{\Sigma}}_1 \tilde{\W}_1) & \Tilde{\mathbf{\Sigma}}_1 
	\end{bmatrix}
	\begin{bmatrix}
	\mathbf{0} \\
	\x_1
	\end{bmatrix} = \x_1^{\top} \Tilde{\mathbf{\Sigma}}_1 \x_1
	\end{align}
	which is zero iff $\x_1 = \mathbf{0}$ as well. Therefore we conclude the precision matrix is indeed positive definite. 
	
\end{proof}

Now based on Lemma ~\ref{lemma:4.1}, we give an upper bound on the KL divergence between the distribution of $(\z_1, y) ; \tilde{\W}$ and a prior distribution, where $\tilde{\W} = (\tilde{\w}_0, \tilde{\W}_1, \tilde{\W}_2)$. This upper bound of the KL divergence sheds light on the analysis for networks with more layers.

\begin{lemma} \label{lemma:4.2}
	Let $\mathbb{P}_{(\z_1, y);\tilde{\W}}$ be the joint distribution of $(\z_1, y)$ parametrized by $\tilde{\W}$, then we have
	$$\mathbb{KL}(\mathbb{P}_{(\z_1, y);\tilde{\W}} || \mathbb{Q}) \leq \frac{1}{2} \left[ \frac{\sigma^2}{\tau^2} \sum_i (1+d_{1,i}^2) + \frac{1}{\tau^2} \| \tilde{\W}_1 \tilde{\w}_0 \|_2^2 + n_1 \ln \left(\frac{\tau^2}{\sigma^2} \right) - n_1 \right]$$
	where $\mathbb{Q} = N(\mathbf{0}, \tau^2 \I_{n_1}) \times \text{Uniform}\{-1, +1\}$ is a prior distribution, $\tau$ is a fixed constant, and $\text{diag}_i(d_{1,i})$ is the diagonal matrix of singular values in the decomposition of $\tilde{\W}_1$.
\end{lemma}

\begin{proof}
	By Lemma ~\ref{lemma:4.1}, taking the inverse of the precision matrix, we know 
	\begin{align}
	Cov(\z_0, \z_1 | y) = \begin{bmatrix}
	\Tilde{\mathbf{\Sigma}}_0 + \tilde{\W}_1^{\top} \Tilde{\mathbf{\Sigma}}_1 \tilde{\W}_1 & -(\Tilde{\mathbf{\Sigma}}_1 \tilde{\W}_1)^{\top} \\
	-(\Tilde{\mathbf{\Sigma}}_1 \tilde{\W}_1) & \Tilde{\mathbf{\Sigma}}_1
	\end{bmatrix}^{-1}
	\end{align}
	
	where $\Tilde{\mathbf{\Sigma}}_0 = (\sigma^2 \I_{n_0})^{-1}$ and $\Tilde{\mathbf{\Sigma}}_1 = (\sigma^2 \I_{n_1})^{-1}$. By the block matrix inversion formula, 
	\begin{align}
	Cov(\z_1 | y) = (\Tilde{\mathbf{\Sigma}}_1 - (\Tilde{\mathbf{\Sigma}}_1 \tilde{\W}_1) (\Tilde{\mathbf{\Sigma}}_0 + \tilde{\W}_1^{\top} \Tilde{\mathbf{\Sigma}}_1 \tilde{\W}_1)^{-1} (\Tilde{\mathbf{\Sigma}}_1 \tilde{\W}_1)^{\top} )^{-1}
	\end{align}
	
	Denote the density of $\mathbb{Q}$ by $q(\z, y)$, and denote the density of $\mathbb{P}_{(\z_1, y);\tilde{\W}}$ by $p(\z, y)$. Also denote the marginal distribution of $\z$ under $\mathbb{Q}$ and $\mathbb{P}_{(\z_1, y);\tilde{\W}}$ by $q(\cdot)$ and $p(\cdot)$, respectively. Note $q(\z)$ is density of $N(\mathbf{0}, \tau^2 \I_{n_1})$, and $q(y)$ and $p(y)$ are both $\frac{1}{2}$ for $y \ \in \{-1, +1\}$. The KL divergence between $\mathbb{P}_{(\z_1, y);\tilde{\W}}$ and $\mathbb{Q}$ is
	\begin{align}
	\begin{split}
	\mathbb{KL}(\mathbb{P}_{(\z_1, y);\tilde{\W}} || \Q) 
	&= \sum_{y \ \in \{-1, +1\}} \int p(\z, y) \log \frac{p(\z, y)}{q(\z, y)} d\z \\
	&= \sum_{y \ \in \{-1, +1\}} \int p(\z | y) p(y) \log \frac{p(\z | y) p(y)}{q(\z | y) q(y)} d\z \\
	&= \frac{1}{2} \int p(\z | y = -1) \log \frac{p(\z | y = -1)}{q(\z)} d\z 
	+ \frac{1}{2} \int p(\z | y = +1) \log \frac{p(\z | y = +1)}{q(\z)} d\z \\
	&= \frac{1}{2} \mathbb{KL}(\mathbb{P}_{(\z_1 | y = -1);\tilde{\W}} || N(\mathbf{0}, \tau^2 \I_{n_1})) 
	+ \frac{1}{2} \mathbb{KL}(\mathbb{P}_{(\z_1 | y = +1);\tilde{\W}} || N(\mathbf{0}, \tau^2 \I_{n_1}))
	\end{split}
	\end{align}
	
	\noindent Note $\mathbb{P}_{(\z_1 | y = -1);\tilde{\W}}$ is a normal distribution with mean $-\tilde{\W}_1 \tilde{\w}_0$ and covariance \\ $\left\{ \Tilde{\mathbf{\Sigma}}_1 - (\Tilde{\mathbf{\Sigma}}_1 \tilde{\W}_1) (\Tilde{\mathbf{\Sigma}}_0 + \tilde{\W}_1^{\top} \Tilde{\mathbf{\Sigma}}_1 \tilde{\W}_1)^{-1} (\Tilde{\mathbf{\Sigma}}_1 \tilde{\W}_1)^{\top} \right\} ^{-1}$. Thus 
	\begin{align}
	\begin{split}
	&\mathbb{KL}(\mathbb{P}_{(\z_1 | y = -1);\tilde{\W}} || N(\mathbf{0}, \tau^2 \I_{n_1})) \\
	=& 
	\frac{1}{2} \underbrace{ \Tr [ (\tau^2 \I_{n_1})^{-1}  (\Tilde{\mathbf{\Sigma}}_1 - (\Tilde{\mathbf{\Sigma}}_1 \tilde{\W}_1) (\Tilde{\mathbf{\Sigma}}_0 + \tilde{\W}_1^{\top} \Tilde{\mathbf{\Sigma}}_1 \tilde{\W}_1)^{-1} (\Tilde{\mathbf{\Sigma}}_1 \tilde{\W}_1)^{\top} )^{-1} ]  }_\textbf{I} \\
	& + \frac{1}{2} \underbrace{\left( (-\tilde{\W}_1 \tilde{\w}_0 - \mathbf{0})^{\top} (\tau^2 \I_{n_1})^{-1} (-\tilde{\W}_1 \tilde{\w}_0 - \mathbf{0}) \right)}_\textbf{II} \\
	& + \frac{1}{2} \underbrace{\ln \left( \frac{\text{det}(\tau^2 \I_{n_1}))}{\text{det}((\Tilde{\mathbf{\Sigma}}_1 - (\Tilde{\mathbf{\Sigma}}_1 \tilde{\W}_1) (\Tilde{\mathbf{\Sigma}}_0 \tilde{\W}_1^{\top} \Tilde{\mathbf{\Sigma}}_1 \tilde{\W}_1)^{-1} (\Tilde{\mathbf{\Sigma}}_1 \tilde{\W}_1)^{\top} )^{-1})} \right) }_\textbf{III} - \frac{n_1}{2}
	\end{split}
	\end{align}
	
	Recall $\Tilde{\mathbf{\Sigma}}_0 = (\sigma^2 \I_{n_0})^{-1}$ and $\Tilde{\mathbf{\Sigma}}_1 = (\sigma^2 \I_{n_1})^{-1}$. Thus
	\begin{align}
	\begin{split}
	\textbf{I} =& \frac{1}{\tau^2} \Tr \left[  \left\{ \Tilde{\mathbf{\Sigma}}_1 - (\Tilde{\mathbf{\Sigma}}_1 \tilde{\W}_1) (\Tilde{\mathbf{\Sigma}}_0 + \tilde{\W}_1^{\top} \Tilde{\mathbf{\Sigma}}_1 \tilde{\W}_1)^{-1} (\Tilde{\mathbf{\Sigma}}_1 \tilde{\W}_1)^{\top} \right\}^{-1} \right] \\
	=& \frac{\sigma^2}{\tau^2} \Tr \left[ \left\{ \I_{n_1} - \tilde{\W}_1 (\I_{n_0} + \tilde{\W}_1^{\top} \tilde{\W}_1 )^{-1} \tilde{\W}_1^{\top} \right\}^{-1} \right] \\
	=& \frac{\sigma^2}{\tau^2} \sum_{i=1}^{n_1} \lambda_i \left( \left\{ \I_{n_1} - \tilde{\W}_1 (\I_{n_0} + \tilde{\W}_1^{\top} \tilde{\W}_1 )^{-1} \tilde{\W}_1^{\top} \right\}^{-1} \right)
	\end{split}
	\end{align}
	where $\lambda_i(\cdot)$ denotes the $i$-th largest eigenvalue of a matrix. 
	
	It is known that the eigenvalues of $(\I-\mathbf{A})^{-1}$ are $\{ \frac{1}{1-\lambda_i(\mathbf{A})} \}$ for positive semi-definite matrix $\mathbf{A}$ with eigenvalues less than 1. Thus we need to show that $\tilde{\W}_1 (\I_{n_0} + \tilde{\W}_1^{\top} \tilde{\W}_1 )^{-1} \tilde{\W}_1^{\top}$ is positive semi-definite and has all eigenvalues less than 1. It is easy to see that $\tilde{\W}_1 (\I_{n_0} + \tilde{\W}_1^{\top} \tilde{\W}_1 )^{-1} \tilde{\W}_1^{\top}$ is symmetric, thus it remains to show the eigenvalues are in $[0,1)$. 
	
	It is also known that $\tilde{\W}_1 (\I_{n_0} + \tilde{\W}_1^{\top} \tilde{\W}_1 )^{-1} \tilde{\W}_1^{\top}$ and $(\I_{n_0} + \tilde{\W}_1^{\top} \tilde{\W}_1 )^{-1} \tilde{\W}_1^{\top} \tilde{\W}_1$ have same eigenvalues.  
	Suppose $\mu$ is an eigenvalue of $(\I_{n_0} + \tilde{\W}_1^{\top} \tilde{\W}_1 )^{-1} \tilde{\W}_1^{\top} \tilde{\W}_1$ with corresponding eigenvector $\x \in \R^{n_0} \setminus \{ \mathbf{0} \}$. Thus
	\begin{align}
	\begin{split}
	(\I_{n_0} + \tilde{\W}_1^{\top} \tilde{\W}_1 )^{-1} \tilde{\W}_1^{\top} \tilde{\W}_1 \x &= \mu \x \\
	\tilde{\W}_1^{\top} \tilde{\W}_1 \x &= \mu (\I_{n_0} + \tilde{\W}_1^{\top} \tilde{\W}_1 ) \x \\
	\x^{\top} \tilde{\W}_1^{\top} \tilde{\W}_1 \x &= \mu \x^{\top} (\I_{n_0} + \tilde{\W}_1^{\top} \tilde{\W}_1 ) \x \\
	\mu &= \frac{ \| \tilde{\W}_1 \x \|_2^2 }{ \|\x\|_2^2 + \| \tilde{\W}_1 \x \|_2^2 } \in [0,1)
	\end{split}
	\end{align}
	
	Consider the singular value decomposition of $\tilde{\W}_1 = \U_1 \D_1 \V_1^{\top}$, where $\U_1 \in \R^{n_1 \times n_1}$ is orthonormal, $\D_1 = \text{diag}(d_{1,i}) \in \R^{n_1 \times n_0}$ and $\V_1 \in \R^{n_0 \times n_0}$ is orthonormal. Then 
	\begin{align}
	\begin{split}
	\lambda \left(\tilde{\W}_1 (\I_{n_0} + \tilde{\W}_1^{\top} \tilde{\W}_1 )^{-1} \tilde{\W}_1^{\top} \right) 
	&= \lambda \left( \U_1 \D_1 \V_1^{\top} (\V_1 \V_1^{\top} + \V_1 \D_1^{\top} \U_1^{\top} \U_1 \D_1 \V_1^{\top})^{-1}  \V_1 \D_1^{\top} \U_1^{\top} \right) \\
	&= \lambda \left(\U_1 \D_1 (\I_{n_0} + \D_1^{\top}  \D_1 \right)^{-1} \D_1^{\top} \U_1^{\top}) \\
	&= \lambda \left(\D_1 (\I_{n_0} + \D_1^{\top}  \D_1 )^{-1} \D_1^{\top} \right) \\
	&= \left\{ \frac{d_{1,i}^2}{1+d_{1,i}^2} \right\}
	\end{split}
	\end{align}
	
	Therefore
	\begin{align}
	\begin{split}
	\textbf{I} =& \frac{\sigma^2}{\tau^2} \sum_{i=1}^{n_1} \lambda_i \left( \{ \I_{n_1} - \tilde{\W}_1 (\I_{n_0} + \tilde{\W}_1^{\top} \tilde{\W}_1 )^{-1} \tilde{\W}_1^{\top} \}^{-1} \right) 
	= \frac{\sigma^2}{\tau^2} \sum_{i=1}^{n_1} \frac{1}{1 - \frac{d_{1,i}^2}{1+d_{1,i}^2}} = \frac{\sigma^2}{\tau^2} \sum_{i=1}^{n_1} (1+d_{1,i}^2)
	\end{split}
	\end{align}
	
	\noindent It is easy to see that $\textbf{II} = \frac{1}{\tau^2} \| \tilde{\W}_1 \tilde{\w}_0 \|_2^2$. Furthermore, we have 
	\begin{align}
	\begin{split}
	\textbf{III} &= 
	\ln \left( \frac{\text{det}(\tau^2 \I_{n_1})}{\text{det} \left( \{ \Tilde{\mathbf{\Sigma}}_1 - (\Tilde{\mathbf{\Sigma}}_1 \tilde{\W}_1) (\Tilde{\mathbf{\Sigma}}_0 \tilde{\W}_1^{\top} \Tilde{\mathbf{\Sigma}}_1 \tilde{\W}_1)^{-1} (\Tilde{\mathbf{\Sigma}}_1 \tilde{\W}_1)^{\top} \} ^{-1} \right) } \right) \\
	&= \ln \left( \frac{\tau^{2 n_1}}{\text{det} \left( \sigma^2 \{\I_{n_1} - \tilde{\W}_1 (\I_{n_0} + \tilde{\W}_1^{\top} \tilde{\W}_1 )^{-1} \tilde{\W}_1^{\top} \}^{-1} \right) } \right)  \\
	&= \ln \left( \frac{\tau^{2 n_1}}{\text{det} \left( \sigma^2 \prod_{i=1}^{n_1} \lambda_i \left( \{ \I_{n_1} - \tilde{\W}_1 (\I_{n_0} + \tilde{\W}_1^{\top} \tilde{\W}_1 )^{-1} \tilde{\W}_1^{\top} \}^{-1} \right) \right) } \right) \\
	&= \ln \left( \frac{\tau^{2 n_1}}{ \sigma^{2 n_1} \prod_{i=1}^{n_1} \lambda_i \left( \{ \I_{n_1} - \tilde{\W}_1 (\I_{n_0} + \tilde{\W}_1^{\top} \tilde{\W}_1 )^{-1} \tilde{\W}_1^{\top} \}^{-1} \right) } \right) \\
	&= \ln \left( \frac{\tau^{2 n_1}}{ \sigma^{2 n_1} \prod_{i=1}^{n_1} (1+d_{1,i}^2) } \right) \leq \ln \left( \frac{\tau^{2 n_1}}{ \sigma^{2 n_1}} \right) = n_1 \ln \left(\frac{\tau^2}{\sigma^2} \right)
	\end{split}
	\end{align}
	
	Thus $\mathbb{KL}(\mathbb{P}_{(\z_1 | y = -1);\tilde{\W}} || N(\mathbf{0}, \tau^2 \I_{n_1})) \leq \frac{1}{2} \left[ \frac{\sigma^2}{\tau^2} \sum_i 1+d_{1,i}^2 + \frac{1}{\tau^2} \| \tilde{\W}_1 \tilde{\w}_0 \|_2^2 + n_1 \ln \left(\frac{\tau^2}{\sigma^2} \right) - n_1 \right]$. Similar reasoning gives the same upper bound on $\mathbb{KL}(\mathbb{P}_{(\z_1 | y = +1);\tilde{\W}} || N(\mathbf{0}, \tau^2 \I_{n_1}))$. 
\end{proof}

\subsection{Two hidden layers}

Now we consider a neural network with the same backward data generatig process with 2 hidden layers, similar to eq (19) in the main text: 
\begin{align} \label{eq:18}
\begin{split}
y &\sim \text{Uniform}\{-1, +1\} \\
\z_0 | y &\sim N(y \tilde{\w}_0, \text{covar} = \sigma^2 \I_{n_0}), \tilde{\w}_0 \in \R^{n_0} \\
\z_1 | \z_0 &\sim N(\tilde{\W}_1 \z_0, \text{covar} = \sigma^2 \I_{n_1}), \tilde{\W}_1 \in \R^{n_1 \times n_0} \\
\x := \z_2 | \z_1 &\sim N(\tilde{\W}_2 \z_1, \text{covar} = \sigma^2 \I_{n_2}= \sigma^2 \I_{p}), \tilde{\W}_2 \in \R^{n_2 \times n_1} = \R^{p \times n_{1}} \\
\end{split}
\end{align}
where $\sigma$ is a constant. We will use $\z_2$ and $\x$ interchangeably, and such definition of $\x$ helps us present the arguments. We will also use $n_2$ and $p$ interchangeably.

In this section we have similar analysis for the entire network. We first generalize the technique in Lemma~\ref{lemma:4.1} to $(\z_0, \z_1, \x) | y; \tilde{\W}$, where $\tilde{\W} = (\tilde{\w}_0, \tilde{\W}_1, \tilde{\W}_2)$.

\begin{lemma} [Joint distribution of $(\z_0, \z_1, \x) | y ; \tilde{\W}$] \label{lemma:4.3}  
	The random variable $(\z_0, \z_1, \x) | y ; \tilde{\W}$ is normally distributed with mean $(y \tilde{\w}_0, y \tilde{\W_1} \tilde{\w}_0, y \tilde{\W_2} \tilde{\W_1} \tilde{\w}_0)$ and covariance matrix
	\begin{align}
	\begin{split}
	\begin{bmatrix}
	\Tilde{\mathbf{\Sigma}}_0 + \tilde{\W_1}^{\top} \Tilde{\mathbf{\Sigma}}_1 \tilde{\W_1} & -(\Tilde{\mathbf{\Sigma}}_1 \tilde{\W_1})^{\top} & \mathbf{0} \\
	-(\Tilde{\mathbf{\Sigma}}_1 \tilde{\W_1}) & \Tilde{\mathbf{\Sigma}}_1 + \tilde{\W_2}^{\top} \Tilde{\mathbf{\Sigma}}_2 \tilde{\W_2} & -(\Tilde{\mathbf{\Sigma}}_2 \tilde{\W_2})^{\top} \\
	\mathbf{0} & -(\Tilde{\mathbf{\Sigma}}_2 \tilde{\W_2}) & \Tilde{\mathbf{\Sigma}}_2
	\end{bmatrix}^{-1} 
	\end{split}
	\end{align}
	where $\Tilde{\mathbf{\Sigma}}_i := \left( \sigma^2 \I_{n_i} \right)^{-1}$ is the precision matrix of $\z_i$ as defined in \eqref{eq:18}.
\end{lemma}

\begin{proof}
	Similarly, we want to check the exponent in $p(\z_0, \z_1, \x | y) = p(\z_0 | y)p(\z_1 | \z_0) p(\x | \z_1)$ matches that of the distribution stated in the claim. 
	
	The exponent of the density of the multivariate normal proposed in the lemma, ignoring the $-\frac{1}{2}$ factor, is
	\begin{align} \label{eq:40}
	\begin{split}
	&\begin{bmatrix}
	\z_0 - y \tilde{\w}_0 \\
	\z_1 - y \tilde{\W}_1 \tilde{\w}_0 \\
	\x - y \tilde{\W}_2 \tilde{\W}_1 \tilde{\w}_0
	\end{bmatrix}^{\top}     
	\begin{bmatrix}
	\Tilde{\mathbf{\Sigma}}_0 + \tilde{\W}_1^{\top} \Tilde{\mathbf{\Sigma}}_1 \tilde{\W}_1 & -(\Tilde{\mathbf{\Sigma}}_1 \tilde{\W}_1)^{\top} & \mathbf{0} \\
	-(\Tilde{\mathbf{\Sigma}}_1 \tilde{\W}_1) & \Tilde{\mathbf{\Sigma}}_1 + \tilde{\W}_2^{\top} \Tilde{\mathbf{\Sigma}}_2 \tilde{\W}_2 & -(\Tilde{\mathbf{\Sigma}}_2 \tilde{\W}_2)^{\top} \\
	\mathbf{0} & -(\Tilde{\mathbf{\Sigma}}_2 \tilde{\W}_2) & \Tilde{\mathbf{\Sigma}}_2
	\end{bmatrix}
	\begin{bmatrix}
	\z_0 - y \tilde{\w}_0 \\
	\z_1 - y \tilde{\W}_1 \tilde{\w}_0 \\
	\x - y \tilde{\W}_2 \tilde{\W}_1 \tilde{\w}_0å
	\end{bmatrix} \\
	\\
	=&  \{ (\z_0 - y \tilde{\w}_0)^{\top} \Tilde{\mathbf{\Sigma}}_0 (\z_0 - y \tilde{\w}_0) 
	+ (\z_0 - y \tilde{\w}_0)^{\top} \tilde{\W}_1^{\top} \Tilde{\mathbf{\Sigma}}_1 \tilde{\W}_1 (\z_0 - y \tilde{\w}_0) \\
	&- (\z_0 - y \tilde{\w}_0)^{\top} (\Tilde{\mathbf{\Sigma}}_1 \tilde{\W}_1)^{\top} (\z_1 - y \tilde{\W}_1 \tilde{\w}_0) 
	- (\mathbf{z}_1 - y \tilde{\W}_1 \tilde{\w}_0)^{\top} (\Tilde{\mathbf{\Sigma}}_1 \tilde{\W}_1) (\z_0 - y \tilde{\w}_0) \\
	&+ (\z_1 - y \tilde{\W}_1 \tilde{\w}_0)^{\top} \Tilde{\mathbf{\Sigma}}_1 (\z_1 - y \tilde{\W}_1 \tilde{\w}_0) \}\\
	& + \{ (\z_1 - y \tilde{\W}_1 \tilde{\w}_0)^{\top} \tilde{\W}_2^{\top} \Tilde{\mathbf{\Sigma}}_2 \tilde{\W}_2 (\z_1 - y \tilde{\W}_1 \tilde{\w}_0) - (\z_1 - y \tilde{\W}_1 \tilde{\w}_0)^{\top} (\Tilde{\mathbf{\Sigma}}_2 \tilde{\W}_2)^{\top} (\z_2 - y \tilde{\W}_2 \tilde{\W}_1 \tilde{\w}_0) \\
	&  -  (\z_2 - y \tilde{\W}_2 \tilde{\W}_1 \tilde{\w}_0)^{\top}  (\Tilde{\mathbf{\Sigma}}_2 \tilde{\W}_2) (\z_1 - y \tilde{\W}_1 \tilde{\w}_0) + (\x - y \tilde{\W}_2 \tilde{\W}_1 \tilde{\w}_0)^{\top} \Tilde{\mathbf{\Sigma}}_2 (\z_2 - y \tilde{\W}_2 \tilde{\W}_1 \tilde{\w}_0) \}  \\
	\end{split}
	\end{align}
	Note that the terms in the first pair of curly braces in \eqref{eq:40} are identical to the part in the exponent of $p(\z_1 , \z_0 | y)$ from Lemma ~\ref{lemma:4.1}.
	
	This means when we compare the exponent of $p(\z_2, \z_1, \z_0 | y) = p(\z_1 , \z_0 | y) p(\z_2 | \z_1)$ with that of the proposed density, the terms in the first pair of curly braces will cancel out. Thus we only need to show the terms in the second pair of curly braces evaluate to $p(\z_2 | \z_1)$ in \eqref{eq:18}:
	
	\begin{align}
	\begin{split}
	&(\z_1 - y \tilde{\W}_1 \tilde{\w}_0)^{\top} \tilde{\W}_2^{\top} \Tilde{\mathbf{\Sigma}}_2 \tilde{\W}_2 (\z_1 - y \tilde{\W}_1 \tilde{\w}_0) - (\z_1 - y \tilde{\W}_1 \tilde{\w}_0)^{\top} (\Tilde{\mathbf{\Sigma}}_2 \tilde{\W}_2)^{\top} (\z_2 - y \tilde{\W}_2 \tilde{\W}_1 \tilde{\w}_0) \\
	&  -  (\z_2 - y \tilde{\W}_2 \tilde{\W}_1 \tilde{\w}_0)^{\top}  (\Tilde{\mathbf{\Sigma}}_2 \tilde{\W}_2) (\z_1 - y \tilde{\W}_1 \tilde{\w}_0) + (\x - y \tilde{\W}_2 \tilde{\W}_1 \tilde{\w}_0)^{\top} \Tilde{\mathbf{\Sigma}}_2 (\z_2 - y \tilde{\W}_2 \tilde{\W}_1 \tilde{\w}_0) \\
	= & \left( \z_1^{\top} \tilde{\W}_2^{\top} \Tilde{\mathbf{\Sigma}}_2 \tilde{\W}_2 \z_1 
	- 2 \z_1^{\top} [ y \tilde{\W}_2^{\top} \Tilde{\mathbf{\Sigma}}_2 \tilde{\W}_2 \tilde{\W}_1 \tilde{\w}_0] 
	+ \tilde{\w}_0^{\top} \tilde{\W}_1^{\top} \tilde{\W}_2^{\top} \Tilde{\mathbf{\Sigma}}_2 \tilde{\W}_2 \tilde{\W}_1 \tilde{\w}_0 \right) \\
	& + \left( - \z_1^{\top} [\tilde{\W}_2^{\top} \Tilde{\mathbf{\Sigma}}_2] \z_2 
	+ \z_1^{\top} [y \tilde{\W}_2^{\top} \Tilde{\mathbf{\Sigma}}_2 \tilde{\W}_2 \tilde{\W}_1 \tilde{\w}_0] 
	+ y \tilde{\w}_0^{\top} \tilde{\W}_1^{\top} \tilde{\W}_2^{\top} \Tilde{\mathbf{\Sigma}}_2 \z_2 
	- \tilde{\w}_0^{\top} \tilde{\W}_1^{\top} \tilde{\W}_2^{\top} \Tilde{\mathbf{\Sigma}}_2 \tilde{\W}_2 \tilde{\W}_1 \tilde{\w}_0 \right) \\
	& + \left( - \z_2^{\top} [\Tilde{\mathbf{\Sigma}}_2 \tilde{\W}_2] \z_1 
	+ \z_2^{\top} [y \Tilde{\mathbf{\Sigma}}_2 \tilde{\W}_2 \tilde{\W}_1 \tilde{\w}_0] 
	å+ y \tilde{\w}_0^{\top} \tilde{\W}_1^{\top} \tilde{\W}_2^{\top} \Tilde{\mathbf{\Sigma}}_2 \tilde{\W_2} \z_1 
	- \tilde{\w}_0^{\top} \tilde{\W}_1^{\top} \tilde{\W}_2^{\top} \Tilde{\mathbf{\Sigma}}_2 \tilde{\W}_2 \tilde{\W}_1 \tilde{\w}_0 \right)\\
	& + \left( \z_2^{\top} \Tilde{\mathbf{\Sigma}}_2 \z_2 - 2 y  \z_2^{\top} \Tilde{\mathbf{\Sigma}}_2 \tilde{\W}_2 \tilde{\W}_1 \tilde{\w}_0 + \tilde{\w}_0^{\top} \tilde{\W}_1^{\top} \tilde{\W}_2^{\top} \Tilde{\mathbf{\Sigma}}_2 \tilde{\W}_2 \tilde{\W}_1 \tilde{\w}_0 \right)
	\end{split}
	\end{align}
	which cancels out to $(\z_2 - \tilde{\W}_2 \z_1)^{\top} \Tilde{\mathbf{\Sigma}}_2 (\z_2 - \tilde{\W}_2 \z_1)$ and is indeed the exponent of the kernel density of $\z_2 | \z_1$.
	
	Now it remains to show the proposed covariance matrix is positive definite, which is equivalent to show that the inverse is positive definite. Consider any vector $\x \neq \mathbf{0}$ and write $\x = (\x_0, \x_1, \x_2)$ where $\x_0 \in \R^{n_0}$, $\x_1 \in \R^{n_1}$ and $\x_2 \in \R^{n_2}$, then, 
	\begin{align} \label{eq:42}
	\begin{split}
	&\begin{bmatrix}
	\x_0 \\
	\x_1 \\
	\x_2
	\end{bmatrix} ^{\top} 
	\begin{bmatrix}
	\Tilde{\mathbf{\Sigma}}_0 + \tilde{\W_1}^{\top} \Tilde{\mathbf{\Sigma}}_1 \tilde{\W_1} & -(\Tilde{\mathbf{\Sigma}}_1 \tilde{\W_1})^{\top} & \mathbf{0} \\
	-(\Tilde{\mathbf{\Sigma}}_1 \tilde{\W_1}) & \Tilde{\mathbf{\Sigma}}_1 + \tilde{\W_2}^{\top} \Tilde{\mathbf{\Sigma}}_2 \tilde{\W_2} & -(\Tilde{\mathbf{\Sigma}}_2 \tilde{\W_2})^{\top} \\
	\mathbf{0} & -(\Tilde{\mathbf{\Sigma}}_2 \tilde{\W_2}) & \Tilde{\mathbf{\Sigma}}_2
	\end{bmatrix}
	\begin{bmatrix}
	\x_0 \\
	\x_1 \\
	\x_2
	\end{bmatrix} \\
	=& \x_0^{\top} (\Tilde{\mathbf{\Sigma}}_0 
	+ \tilde{\W_1}^{\top} \Tilde{\mathbf{\Sigma}}_1 \tilde{\W_1}) \x_0 - \x_0^{\top} (\Tilde{\mathbf{\Sigma}}_1 \tilde{\W}_1)^{\top} \x_1 - \x_1^{\top} (\Tilde{\mathbf{\Sigma}}_1 \tilde{\W}_1) \x_0 
	+ \x_1^{\top} (\Tilde{\mathbf{\Sigma}}_1 
	+ \tilde{\W_2}^{\top} \Tilde{\mathbf{\Sigma}}_2 \tilde{\W_2}) \x_1 \\
	& - \x_1^{\top} (\Tilde{\mathbf{\Sigma}}_2 \tilde{\W}_2)^{\top} \x_2 
	- \x_2^{\top} (\Tilde{\mathbf{\Sigma}}_2 \tilde{\W}_2) \x_1 
	+ \x_2^{\top} \Tilde{\mathbf{\Sigma}}_2 \x_2 \\
	=& \| \Tilde{\mathbf{\Sigma}}_0^{1/2} \x_0 \|_2^2 
	+ \| \Tilde{\mathbf{\Sigma}}_1^{1/2} \tilde{\W}_1 \x_0 \|_2^2 
	- 2 \x_0^{\top} \tilde{\W}_1^{\top} \Tilde{\mathbf{\Sigma}}_1 \x_1 
	+ \| \Tilde{\mathbf{\Sigma}}_1^{1/2} \x_1 \|_2^2 \\
	& + \| \Tilde{\mathbf{\Sigma}}_2^{1/2} \tilde{\W}_2 \x_1 \|_2^2  
	- 2 \x_1^{\top} \tilde{\W}_2^{\top} \Tilde{\mathbf{\Sigma}}_2 \x_2 + \| \Tilde{\mathbf{\Sigma}}_2^{1/2} \x_2 \|_2^2 \\
	\geq & \| \Tilde{\mathbf{\Sigma}}_0^{1/2} \x_0 \|_2^2 
	+ \left( 
	\| \Tilde{\mathbf{\Sigma}}_1^{1/2} \tilde{\W}_1 \x_0 \|_2^2 
	- 2 \| \Tilde{\mathbf{\Sigma}}_1^{1/2} \tilde{\W}_1 \x_0 \|_2 \| \Tilde{\mathbf{\Sigma}}_1^{1/2} \x_1 \|_2
	+ \| \Tilde{\mathbf{\Sigma}}_1^{1/2} \x_1 \|_2^2 
	\right) \\
	&+ \left( 
	\| \Tilde{\mathbf{\Sigma}}_2^{1/2} \tilde{\W}_2 \x_1 \|_2^2 
	- 2 \| \Tilde{\mathbf{\Sigma}}_2^{1/2} \tilde{\W}_2 \x_1 \|_2  \| \Tilde{\mathbf{\Sigma}}_2^{1/2} \x_2 \|_2
	+ \| \Tilde{\mathbf{\Sigma}}_2^{1/2} \x_2 \|_2^2 
	\right) \\
	= & \| \Tilde{\mathbf{\Sigma}}_0^{1/2} \x_0 \|_2^2
	\end{split}
	\end{align}
	which is positive if $\x_0 \neq \mathbf{0}$. Otherwise, if $\x_0 = \mathbf{0}$, then \eqref{eq:42} evaluates to 
	\begin{align}
	\begin{split}
	&\begin{bmatrix}
	\mathbf{0} \\
	\x_1 \\
	\x_2
	\end{bmatrix} ^{\top} 
	\begin{bmatrix}
	\Tilde{\mathbf{\Sigma}}_0 + \tilde{\W_1}^{\top} \Tilde{\mathbf{\Sigma}}_1 \tilde{\W_1} & -(\Tilde{\mathbf{\Sigma}}_1 \tilde{\W_1})^{\top} & \mathbf{0} \\
	-(\Tilde{\mathbf{\Sigma}}_1 \tilde{\W_1}) & \Tilde{\mathbf{\Sigma}}_1 + \tilde{\W_2}^{\top} \Tilde{\mathbf{\Sigma}}_2 \tilde{\W_2} & -(\Tilde{\mathbf{\Sigma}}_2 \tilde{\W_2})^{\top} \\
	\mathbf{0} & -(\Tilde{\mathbf{\Sigma}}_2 \tilde{\W_2}) & \Tilde{\mathbf{\Sigma}}_2
	\end{bmatrix}
	\begin{bmatrix}
	\mathbf{0} \\
	\x_1 \\
	\x_2
	\end{bmatrix} \\
	=& \| \Tilde{\mathbf{\Sigma}}_1^{1/2} \x_1 \|_2^2 + \| \Tilde{\mathbf{\Sigma}}_2^{1/2} \tilde{\W}_2 \x_1 \|_2^2  - 
	2 \x_1^{\top} \tilde{\W}_2^{\top} \Tilde{\mathbf{\Sigma}}_2 \x_2 
	+ \| \Tilde{\mathbf{\Sigma}}_2^{1/2} \x_2 \|_2^2 \\
	\geq & \| \Tilde{\mathbf{\Sigma}}_1^{1/2} \x_1 \|_2^2 + \| \Tilde{\mathbf{\Sigma}}_2^{1/2} \tilde{\W}_2 \x_1 \|_2^2 
	- 2 \| \Tilde{\mathbf{\Sigma}}_2^{1/2} \tilde{\W}_2 \x_1 \|_2 \| \Tilde{\mathbf{\Sigma}}_2^{1/2} \x_2 \|_2 
	+ \| \Tilde{\mathbf{\Sigma}}_2^{1/2} \x_2 \|_2^2 = \| \Tilde{\mathbf{\Sigma}}_1^{1/2} \x_1 \|_2^2
	\end{split}
	\end{align}
	which is positive if $\x_1 \neq \mathbf{0}$. Otherwise, meaning that both $\x_0$ and $\x_1$ are $\mathbf{0}$, \eqref{eq:42} would evaluate to 
	\begin{align}
	\begin{split}
	&\begin{bmatrix}
	\mathbf{0} \\
	\mathbf{0} \\
	\x_2
	\end{bmatrix} ^{\top} 
	\begin{bmatrix}
	\Tilde{\mathbf{\Sigma}}_0 + \tilde{\W_1}^{\top} \Tilde{\mathbf{\Sigma}}_1 \tilde{\W_1} & -(\Tilde{\mathbf{\Sigma}}_1 \tilde{\W_1})^{\top} & \mathbf{0} \\
	-(\Tilde{\mathbf{\Sigma}}_1 \tilde{\W_1}) & \Tilde{\mathbf{\Sigma}}_1 + \tilde{\W_2}^{\top} \Tilde{\mathbf{\Sigma}}_2 \tilde{\W_2} & -(\Tilde{\mathbf{\Sigma}}_2 \tilde{\W_2})^{\top} \\
	\mathbf{0} & -(\Tilde{\mathbf{\Sigma}}_2 \tilde{\W_2}) & \Tilde{\mathbf{\Sigma}}_2
	\end{bmatrix} 
	\begin{bmatrix}
	\mathbf{0} \\
	\mathbf{0} \\
	\x_2
	\end{bmatrix} 
	= \| \Tilde{\mathbf{\Sigma}}_2^{1/2} \x_2 \|_2^2
	\end{split}
	\end{align}
	which is positive if $\x_2 \neq \mathbf{0}$. Thus we showed that the quadratic form of the precision matrix is zero iff $\x = (\x_0, \x_1, \x_2) = \mathbf{0}$. Thus the proposed distribution of $(\z_0, \z_1, \z_2) | y$ is a valid multivariate normal distribution.
\end{proof}

Next we find the marginal distribution of $\z_2 | y ; \tilde{\W}$.

\begin{lemma} [Marginal Distribution of $\z_2 | y ; \tilde{\W}$] \label{lemma:4.4} 
	The random variable $\x | y; \tilde{\W}$ is normally distributed with mean $y \tilde{\W}_2 \tilde{\W}_1 \tilde{\w}_0$, and covariance matrix
	$$\sigma^2 \left[ \I_p - \tilde{\W}_2 \left[ (\I_{n_1} + \tilde{\W}_2^{\top} \tilde{\W}_2) - \tilde{\W}_1 
	(\I_{n_0} + \tilde{\W}_1^{\top} \tilde{\W}_1)^{-1} \tilde{\W}_1^{\top} \right]^{-1} \tilde{\W}_2^{\top} \right]^{-1}. $$
\end{lemma}

\begin{proof}
	This follows from repeated application of block matrix inversion formula, which gives the covariance matrix of $\x | y; \tilde{\W}$. First recall the precision matrix of $(\z_0, \z_1, \x) | y; \tilde{\W}$ is a tri-diagonal matrix, and take the inverse of the precision matrix of $(\z_0, \z_1, \x) | y; \tilde{\W}$, 
	\begin{align}
	\begin{split}
	\text{Cov} \left( 
	\begin{bmatrix}
	\z_0 \\ \z_1 \\ \x
	\end{bmatrix} | y \right)
	&=
	\begin{bmatrix}
	\Tilde{\mathbf{\Sigma}}_0 + \tilde{\W_1}^{\top} \Tilde{\mathbf{\Sigma}}_1 \tilde{\W_1} & -(\Tilde{\mathbf{\Sigma}}_1 \tilde{\W_1})^{\top} & \mathbf{0} \\
	-(\Tilde{\mathbf{\Sigma}}_1 \tilde{\W_1}) & \Tilde{\mathbf{\Sigma}}_1 + \tilde{\W_2}^{\top} \Tilde{\mathbf{\Sigma}}_2 \tilde{\W_2} & -(\Tilde{\mathbf{\Sigma}}_2 \tilde{\W_2})^{\top} \\
	\mathbf{0} & -(\Tilde{\mathbf{\Sigma}}_2 \tilde{\W_2}) & \Tilde{\mathbf{\Sigma}}_2
	\end{bmatrix} ^{-1} \\
	&= 
	\begin{bmatrix}
	\T & \mathbf{Z}^{\top} \\ \mathbf{Z} & \D
	\end{bmatrix}^{-1} 
	=: \begin{bmatrix}
	\mathbf{S}_1 & \mathbf{S}_2^{\top} \\
	\mathbf{S}_2 & \mathbf{S}_3
	\end{bmatrix}
	\end{split}
	\end{align}
	where $\T = \begin{bmatrix}
	\Tilde{\mathbf{\Sigma}}_0 + \tilde{\W_1}^{\top} \Tilde{\mathbf{\Sigma}}_1 \tilde{\W_1} & -(\Tilde{\mathbf{\Sigma}}_1 \tilde{\W_1})^{\top} \\
	-(\Tilde{\mathbf{\Sigma}}_1 \tilde{\W_1}) & \Tilde{\mathbf{\Sigma}}_1 + \tilde{\W_2}^{\top} \Tilde{\mathbf{\Sigma}}_2 \tilde{\W_2}
	\end{bmatrix}$, $\mathbf{Z} = \begin{bmatrix}
	\mathbf{0} & -(\Tilde{\mathbf{\Sigma}}_2 \tilde{\W}_2)
	\end{bmatrix}$, $\mathbf{Z}^{\top} = \begin{bmatrix}
	\mathbf{0} \\ -(\Tilde{\mathbf{\Sigma}}_2 \tilde{\W}_2)^{\top}
	\end{bmatrix}$, and $\D = \begin{bmatrix} \Tilde{\mathbf{\Sigma}}_2
	\end{bmatrix}$. We are interested in $\mathbf{S}_3$, which is  precisely $\text{Cov}(\x | y; \tilde{\W})$.
	Then by the block matrix inversion formula, we know 
	\begin{align}
	\begin{split}
	\mathbf{S}_3^{-1} 
	&= (\D - \Z \T^{-1} \Z^{\top}) \\
	&= \Tilde{\mathbf{\Sigma}}_2 - 
	\begin{bmatrix}
	\mathbf{0} \\ -(\Tilde{\mathbf{\Sigma}}_2 \tilde{\W}_2)^{\top}
	\end{bmatrix}^{\top} 
	\begin{bmatrix}
	\Tilde{\mathbf{\Sigma}}_0 + \tilde{\W_1}^{\top} \Tilde{\mathbf{\Sigma}}_1 \tilde{\W_1} & -(\Tilde{\mathbf{\Sigma}}_1 \tilde{\W_1})^{\top} \\
	-(\Tilde{\mathbf{\Sigma}}_1 \tilde{\W_1}) & \Tilde{\mathbf{\Sigma}}_1 + \tilde{\W_2}^{\top} \Tilde{\mathbf{\Sigma}}_2 \tilde{\W_2}
	\end{bmatrix}^{-1}
	\begin{bmatrix}
	\mathbf{0} \\ -(\Tilde{\mathbf{\Sigma}}_2 \tilde{\W}_2)^{\top}
	\end{bmatrix} \\
	&= \Tilde{\mathbf{\Sigma}}_2 
	- (\Tilde{\mathbf{\Sigma}}_2 \tilde{\W}_2) \left[ (\Tilde{\mathbf{\Sigma}}_1 
	+ \tilde{\W}_2^{\top} \Tilde{\mathbf{\Sigma}}_2 \tilde{\W}_2) 
	- (\Tilde{\mathbf{\Sigma}}_1 \tilde{\W}_1) (\Tilde{\mathbf{\Sigma}}_0 + \tilde{\W}_1^{\top} \Tilde{\mathbf{\Sigma}}_1 \tilde{\W}_1)^{-1} (\Tilde{\mathbf{\Sigma}}_1 \tilde{\W}_1)^{\top} \right]^{-1} (\Tilde{\mathbf{\Sigma}}_0 \tilde{\W}_1)^{\top} \\
	&= \frac{1}{\sigma^2} 
	\left[\I_p 
	- \tilde{\W}_2 \left[ (\I_{n_1} 
	+ \tilde{\W}_2^{\top} \tilde{\W}_2) 
	- \tilde{\W}_1 (\I_{n_0} 
	+ \tilde{\W}_1^{\top} \tilde{\W}_1)^{-1} \tilde{\W}_1^{\top} \right]^{-1} \tilde{\W}_2^{\top} \right]
	\end{split}
	\end{align}
\end{proof}

Now that we have the marginal distribution of $(\x | y ; \tilde{\W}) = (\z_2 | y ; \tilde{\W})$ , we can again find an upper bound on the KL divergence.

\begin{lemma} [Upper bound on KL divergence between $(\x,y) ; \W$ and a prior distribution $\mathbb{Q}$] \label{lemma:4.5}  
	We have that 
	$$\mathbb{KL}(\mathbb{P}_{(\x, y);\tilde{\W}} || \mathbb{Q}) \leq \frac{1}{2} \left[ \frac{\sigma^2}{\tau^2} \left( p + \sum_{i = 1}^r d_{2,i}^2 (1 + d_{1, 1}^2) \right) + \frac{1}{\tau^2} \| \tilde{\W_2} \tilde{\W_1} \tilde{\w}_0 \|_2^2 + p \ln \left(\frac{\tau^2}{\sigma^2} \right) - p \right]$$
	where $\mathbb{Q} \sim N(\mathbf{0}, \tau^2 \I_{p}) \times \text{Uniform} \{-1, +1\}$ with $\tau$ being a fixed constant, $\mathbb{P}_{(\x, y);\tilde{\W}}$ is the joint distribution of $(\x, y)$ when $\tilde{\W}$ is viewed as a parameter, $\text{diag}_i(d_{1,i})$ is the diagonal matrix of the singular values in the decomposition of $\tilde{\W}_1$ with diagonal entries in decreasing order, $d_{2,1}$ is the largest singular value of $\tilde{\W}_2$, and $r = \text{rank}(\tilde{\W}_1) = \text{rank}(\tilde{\W}_2)$.
\end{lemma}

\begin{proof}
	First we let $\M_2(\tilde{\W})$ denote $\tilde{\W}_2 \left[ (\I_{n_1} + \tilde{\W}_2^{\top} \tilde{\W}_2) - \tilde{\W}_1 (\I_{n_0} + \tilde{\W}_1^{\top} \tilde{\W}_1)^{-1} \tilde{\W}_1^{\top} \right]^{-1} \tilde{\W}_2^{\top}$. By a similar reasoning as in Lemma ~\ref{lemma:4.2} we know the KL divergence has
	\begin{align}
	\begin{split}
	\mathbb{KL}(\mathbb{P}_{(\x, y);\tilde{\W}} || \mathbb{Q}) 
	&= \sum_{y \ \in \{-1, +1\}} \int p(\x, y) \log \frac{p(\x, y)}{q(\x, y)} d\x \\
	&= \frac{1}{2} \mathbb{KL}(\mathbb{P}_{(\x | y = -1);\tilde{\W}} || N(\mathbf{0}, \tau^2 \I_{p})) + \frac{1}{2} \mathbb{KL}(\mathbb{P}_{(\x | y = +1);\tilde{\W}} || N(\mathbf{0}, \tau^2 \I_{p}))
	\end{split}
	\end{align}
	where $q(\x, y)$ is the density of $\mathbb{Q}$ and $p(\x,y)$ is the density of $\mathbb{P}_{(\x, y);\tilde{\W}}$.
	
	We focus on one of the two KL-divergences. Recall the formula of KL divergence between two multivariate normal distributions, $\mathbb{P} \sim \text{N}(\mu_1, \Sigma_1)$ and $\mathbb{Q} \sim \text{N}(\mu_2, \Sigma_2)$, $\mu_1, \mu_2 \in \R^k$, $\Sigma_1, \Sigma_2 \in \R^{k \times k}$,
	\begin{align}
	\begin{split}
	\mathbb{KL}(\mathbb{P} || \mathbb{Q}) 
	= \frac{1}{2} \left( \text{Tr}(\Sigma_2^{-1} \Sigma_1) + \left(\mu_2 - \mu_1 \right)^{\top} \Sigma_2^{-1} \left(\mu_2 - \mu_1 \right) + \ln \frac{\text{det} \Sigma_2}{\text{det} \Sigma_1} - k \right)
	\end{split}
	\end{align}
	
	Plug in the formula, 
	\begin{align}
	\begin{split}
	&\mathbb{KL}(\mathbb{P}_{(\x | y = -1);\tilde{\W}} || N(\mathbf{0}, \tau^2 \I_{p})) \\
	=& 
	\frac{1}{2} \frac{\sigma^2}{\tau^2} \underbrace{ \Tr \left( ( \I_p - \M_2(\tilde{\W}) )^{-1} \right)}_\textbf{I}
	+ \frac{1}{2} \underbrace{\left( (- \tilde{\W}_2 \tilde{\W}_1 \tilde{\w}_0 )^{\top} (\tau^2 \I_{p})^{-1} (- \tilde{\W}_2 \tilde{\W_1} \tilde{\w}_0 ) \right)}_\textbf{II} \\
	+& \frac{1}{2} \underbrace{\ln \left( \frac{\text{det}(\tau^2 \I_{p})}{\text{det} (\text{Cov}(\x| y)) 
		} \right) 
	}_\textbf{III} - \frac{p}{2}
	\end{split}
	\end{align}
	
	In next lemma we show that $\M_2(\tilde{\W})$ is positive semi-definite and has largest eigenvalue strictly less than 1.
	
	Observe $\textbf{I} = \frac{\sigma^2}{\tau^2} \Tr \left( (\I_p - \M_2(\tilde{\W}))^{-1} \right)  = \frac{\sigma^2}{\tau^2} \sum_i \frac{1}{1 - \lambda_i(\M_2)}$, where $\lambda(\M_2(\tilde{\W})) = \{\lambda_i(\M_2(\tilde{\W})) \}$ is the set of eigenvalues of $\M_2(\tilde{\W})$ and may contain zero. We claim $\textbf{I} \leq \frac{\sigma^2}{\tau^2} \sum_{i=1}^p (d_{2,i}^2 + d_{2,i}^2 d_{1,1}^2 + 1)$, where $\D_1 = \text{diag}_i (d_{1,i})$ is the diagonal matrix of singular values in the decomposition of $\tilde{\W}_1$, and $\D_2 = \text{diag}_i (d_{2,i})$ is the diagonal matrix of singular values in the decomposition of $\tilde{\W}_2$, and the diagonal entries in both matrices are in decreasing order. We leave the proof to Lemma \ref{lemma:4.7}.
	
	It is easy to see that $\textbf{II} = \frac{\| \tilde{\W}_2 \tilde{\W}_1 \tilde{\w}_0 \|_2^2}{\tau^2}$.
	
	Now we provide an upper bound on term $\textbf{III}$, where the last inequality follows from Lemma \ref{lemma:4.6} below. 
	
	\begin{align}
	\begin{split}
	\textbf{III} 
	&= \ln \left( \frac{\text{det}(\tau^2 \I_{p})}{\text{det} \left( \text{Cov}(\x| y) \right)} \right) 
	= p \ln \left( \frac{\tau^2}{\sigma^2} \right) + \ln \left( \text{det} \left( \I_p - \M_2(\tilde{\W}) \right) \right) \\
	&= p \ln \left( \frac{\tau^2}{\sigma^2} \right) + \ln \left( \prod_{i=1}^p (1 - \lambda_i(\M_2(\tilde{\W}))) \right) 
	\leq p \ln \left( \frac{\tau^2}{\sigma^2} \right)
	\end{split}
	\end{align}
	
\end{proof}

Below we list two lemmas regarding properties of a matrix key to the analysis of the KL divergence in Lemma \ref{lemma:4.5}. Moreover, Lemma \ref{lemma:4.7} will be generalized later in the discussion of a general $d$-layer network. 

\begin{lemma}\label{lemma:4.6}
	Let
	
	\begin{align}
	\begin{split}
	\M_1(\tilde{\W}) := & \tilde{\W}_1 \left( \I_{n_0} + \tilde{\W}_1^{\top} \tilde{\W}_1 \right)^{-1} \tilde{\W}_1^{\top} \\
	\M_2(\tilde{\W}) := & \tilde{\W}_2 \left( \I_{n_1} + \tilde{\W}_2^{\top} \tilde{\W}_2 - \M_1 \right)^{-1} \tilde{\W}_2^{\top} \\
	\end{split}
	\end{align}
	
	then $\M_2(\tilde{\W})$ is positive semi-definite and has all eigenvalues strictly less than 1.
\end{lemma}
\begin{proof}
	For simplicity we use $\M_2$ for $\M_2(\tilde{\W})$ in this proof. First we show positive semi-definiteness. Consider any $\x \neq \mathbf{0} \in \R^p$. Suppose $\x \not\in \text{Ker}(\tilde{\W}_2)$, then $\exists \mathbf{y} \neq \mathbf{0} \in \R^{n_1}$ such that $\x = \tilde{\W}_2^{\top} \mathbf{y}$, and 
	\begin{align}
	\begin{split}
	\x^{\top} \M_2 \x &= \x^{\top} \tilde{\W}_2 \left[ (\I_{n_1} + \tilde{\W}_2^{\top} \tilde{\W}_2) - \tilde{\W}_1 (\I_{n_0} + \tilde{\W}_1^{\top} \tilde{\W}_1)^{-1} \tilde{\W}_1^{\top} \right]^{-1} \tilde{\W}_2^{\top} \x \\
	&= \y^{\top} \left[ (\I_{n_1} + \tilde{\W}_2^{\top} \tilde{\W}_2) - \tilde{\W}_1 (\I_{n_0} + \tilde{\W}_1^{\top} \tilde{\W}_1)^{-1} \tilde{\W}_1^{\top} \right]^{-1} \y
	\end{split}
	\end{align}
	
	Thus it suffices to show positive semi-definiteness of $\left[ (\I_{n_1} + \tilde{\W}_2^{\top} \tilde{\W}_2) - \tilde{\W}_1 (\I_{n_0} + \tilde{\W}_1^{\top} \tilde{\W}_1)^{-1} \tilde{\W}_1^{\top} \right]^{-1}$, which is equivalent to show positive semi-definiteness of $\left[ (\I_{n_1} + \tilde{\W}_2^{\top} \tilde{\W}_2) - \tilde{\W}_1 (\I_{n_0} + \tilde{\W}_1^{\top} \tilde{\W}_1)^{-1} \tilde{\W}_1^{\top} \right]$.
	
	Now consider any $\mathbf{v} \neq \mathbf{0} \in \R^{n_1}$, and consider the singular value decomposition of $\tilde{\W}_1 = \U_1 \D_1 \V_1^{\top}$, where $\U_1 \in \R^{n_1 \times n_1}$ is orthonormal, $\D_1 = \text{diag}_i(d_{1,i}) \in \R^{n_1 \times n_0}$ has singular values of $\tilde{\W}_1$ along the diagonal, and $\V_1 \in \R^{n_0 \times n_0}$ is orthonormal.
	
	\begin{align}
	\begin{split}
	& \mathbf{v}^{\top} \left[ (\I_{n_1} + \tilde{\W}_2^{\top} \tilde{\W}_2) - \tilde{\W}_1 (\I_{n_0} + \tilde{\W}_1^{\top} \tilde{\W}_1)^{-1} \tilde{\W}_1^{\top} \right] \mathbf{v} \\
	=& \| \mathbf{v} \|_2^2 + \| \tilde{\W}_2 \mathbf{v} \|_2^2 - \mathbf{v}^{\top} \tilde{\W}_1 (\I_{n_0} + \tilde{\W}_1^{\top} \tilde{\W}_1)^{-1} \tilde{\W}_1^{\top} \mathbf{v} \\
	=& \| \mathbf{v} \|_2^2 + \| \tilde{\W}_2 \mathbf{v} \|_2^2 - \mathbf{v}^{\top} \U_1 \D_1 (\I_{n_0} + \D_1^{\top}  \D_1 )^{-1} \D_1^{\top} \U_1^{\top} \mathbf{v} \\
	=& \| \mathbf{v} \|_2^2 + \| \tilde{\W}_2 \mathbf{v} \|_2^2 - \| \text{diag}_i \left( \sqrt{\frac{d_{1,i}^2}{1+d_{1,i}^2}} \right) \U_1^{\top} \mathbf{v} \|_2^2 \\
	\geq & \| \tilde{\W}_2 \mathbf{v} \|_2^2
	\end{split}
	\end{align}
	
	This shows that $\M_2$ is positive semi-definite. Now we show that the eigenvalues of $\M_2$ are less than 1.
	
	First, continue using the singular value decomposition $\tilde{\W}_1 = \U_1 \D_1 \V_1^{\top}$ and consider the singular value decomposition $\tilde{\W}_2 = \U_2 \D_2 \V_2^{\top}$ where $\U_2 \in \R^{p \times p}$ is orthonormal, $\D_2 = \text{diag}(d_{2,i}) \in \R^{p \times n_1}$ is the diagonal matrix of singular values of $\tilde{\W_2}$ where the diagonal entries are in decreasing order, and $\V_2 \in \R^{n_1 \times n_1}$ is orthonormal, then
	\begin{align}
	\begin{split}
	\M_2 &= \tilde{\W}_2 \left[ (\I_{n_1} + \tilde{\W}_2^{\top} \tilde{\W}_2) - \tilde{\W}_1 (\I_{n_0} + \tilde{\W}_1^{\top} \tilde{\W}_1)^{-1} \tilde{\W}_1^{\top} \right]^{-1} \tilde{\W}_2^{\top} \\
	&= \tilde{\W}_2 \left[ (\I_{n_1} + \tilde{\W}_2^{\top} \tilde{\W}_2) - \U_1 \D_1 \V_1^{\top} (\V_1 \V_1^{\top} + \V_1 \D_1^{\top} \U_1^{\top} \U_1 \D_1 \V_1 )^{-1} \V_2 \D_1^{\top} \U_1^{\top} \right]^{-1} \tilde{\W}_2^{\top} \\
	&= \tilde{\W}_2 \left[ (\I_{n_1} + \tilde{\W}_2^{\top} \tilde{\W}_2) - \U_1 \D_1 \V_1^{\top} \left( \V_1 \text{diag}_i \left(\frac{1}{1 + d_{1,i}^2} \right) \V_1^{\top} \right) \V_1 \D_1^{\top} \U_1^{\top} \right]^{-1} \tilde{\W}_2^{\top} \\
	&= \tilde{\W}_2 \left[ (\U_1 \U_1^{\top} + \tilde{\W}_2^{\top} \tilde{\W}_2) - \U_1 \left( \text{diag}_i \left(\frac{d_{1,i}^2}{1 + d_{1,i}^2} \right) \right)  \U_1^{\top} \right]^{-1} \tilde{\W}_2^{\top} \\
	&= \tilde{\W}_2 \left[ \U_1 \left( \text{diag}_i \left(\frac{1}{1 + d_{1,i}^2} \right) \right) \U_1^{\top} + \tilde{\W}_2^{\top} \tilde{\W}_2 \right]^{-1} \tilde{\W}_2^{\top} \\
	\end{split}
	\end{align}
	
	Note the eigenvalues of $\M_2$, $\lambda(\M_2)$, are the same as the eigenvalues of $\M_2'$, $\lambda(\M_2')$, where 
	\begin{align} \label{eq:35}
	\begin{split}
	\M_2' := \left[ \U_1 \left( \text{diag}_i \left(\frac{1}{1 + d_{1,i}^2} \right) \right) \U_1^{\top} + \tilde{\W}_2^{\top} \tilde{\W}_2 \right]^{-1} \tilde{\W}_2^{\top} \tilde{\W}_2
	\end{split}
	\end{align}
	
	Now consider any eigenvalue $\mu$ of $\M_2$ with corresponding eigenvector $\x$, and we want to show that $\mu < 1$.
	\begin{align} \label{eq:36}
	\begin{split}
	\left[ \U_1 \left( \text{diag}_i \left(\frac{1}{1 + d_{1,i}^2} \right) \right) \U_1^{\top} + \tilde{\W}_2^{\top} \tilde{\W}_2 \right]^{-1} \tilde{\W}_2^{\top} \tilde{\W}_2 \x &= \mu \x  \\ 
	\tilde{\W}_2^{\top} \tilde{\W}_2 \x &= \mu \left[ \U_1 \left( \text{diag}_i \left(\frac{1}{1 + d_{1,i}^2} \right) \right) \U_1^{\top} + \tilde{\W}_2^{\top} \tilde{\W}_2 \right] \x \\
	(1 - \mu) \tilde{\W}_2^{\top} \tilde{\W}_2 \x &= \mu \U_1 \left( \text{diag}_i \left(\frac{1}{1 + d_{1,i}^2} \right) \right) \U_1^{\top} \x  \\
	(1 - \mu) \x^{\top} \tilde{\W}_2^{\top} \tilde{\W}_2 \x &= \mu \x^{\top} \U_1 \left( \text{diag}_i \left(\frac{1}{1 + d_{1,i}^2} \right) \right) \U_1^{\top} \x 
	\end{split}
	\end{align}
	
	\noindent Case 1: $\mu \neq 0$, which is equivalent to $\tilde{\W}_2 \x \neq \mathbf{0}$. Then note $\x^{\top} \U_1 \left( \text{diag}_i \left(\frac{1}{1 + d_{1,i}^2} \right) \right) \U_1^{\top} \x > 0$ as $\U_1 \left( \text{diag}_i \left(\frac{1}{1 + d_{1,i}^2} \right) \right) \U_1^{\top}$ is positive definite. Thus $(1 - \mu) \x^{\top} \tilde{\W}_2^{\top} \tilde{\W}_2 \x > 0$, and therefore $\mu < 1$.
	
	\noindent Case 2: $\mu = 0$. Then $\tilde{\W}_2 \x = \mathbf{0}$, $\x \in \text{Ker}(\tilde{\W}_2)$. 
\end{proof}

\begin{lemma} \label{lemma:4.7}
	For $\M_2(\tilde{\W})$ defined in Lemma \ref{lemma:4.6}, let $\lambda(\M_(\tilde{\W})) = \{\lambda_i^{\downarrow}(\M_2(\tilde{\W})): i = 1, 2, \cdots, p \}$ be the set of eigenvalues of $\M_2(\tilde{\W})$ in decreasing order, and $\lambda(\M_2(\tilde{\W}))$ may contain zero. Then 
	$$\Tr \left( (\I_p - \M_2(\tilde{\W}))^{-1} \right) = \sum_{i=1}^p \frac{1}{1 - \lambda_i^{\downarrow}(\M_2(\tilde{\W}))} \leq p + \sum_{i = 1}^r d_{2,i}^2 (1 + d_{1, 1}^2) $$ 
	where $r = \text{rank}(\M_2(\tilde{\W}))$, $\D_1 = \text{diag}_i (d_{1,i})$ is the diagonal matrix of singular values of the decomposition of $\tilde{\W}_1$, $\D_2 = \text{diag}_i (d_{2,i})$ is the diagonal matrix of singular values of the decomposition of $\tilde{\W}_2$, and the diagonal entries in both matrices are in decreasing order.
\end{lemma}

\begin{proof}
	For simplicity we use $\M_2$ for $\M_2(\tilde{\W})$ in this proof. First note that matrix $\M_2'$ defined below has the same spectrum as $\M_2$, i.e. $\lambda(\M_2) = \lambda(\M_2')$
	\begin{align} \label{eq:37}
	\begin{split}
	\M_2'
	&:=  \left[ (\I_{n_1} + \tilde{\W}_2^{\top} \tilde{\W}_2) - \tilde{\W}_1 
	(\I_{n_0} + \tilde{\W}_1^{\top} \tilde{\W}_1)^{-1} \tilde{\W}_1^{\top} \right]^{-1} \tilde{\W_2}^{\top} \tilde{\W}_2 \\ 
	&= \left[ \U_1 \left( \text{diag}_i \frac{1}{1 + d_{1,i}^2}\right) \U_1^{\top} + \tilde{\W}_2^{\top} \tilde{\W}_2 \right]^{-1} \tilde{\W}_2^{\top} \tilde{\W}_2 \\
	& =: \left[ \C + \D \right]^{-1} \D
	\end{split}
	\end{align}
	with $\C := \U_1 \left( \text{diag}_i \frac{1}{1 + d_{1,i}^2}\right) \U_1^{\top}$ is positive definite and $\D := \tilde{\W}_2^{\top} \tilde{\W}_2$ is positive semi-definite.
	
	Consider diagonalization of $\D = \Q \D' \Q^{\top} $, where $\Q$ is an orthonormal matrix and $\D'$ is diagonal. Then there exists a matrix $\C'$ such that $\C = \Q \C' \Q^{\top}$, and thus
	
	\begin{align} \label{eq:38}
	\begin{split}
	(\C + \D)^{-1} \D &= (\Q \C' \Q^{\top} + \Q \D' \Q^{\top})^{-1} \Q \D' \Q^{\top} \\
	&= \Q (\C' + \D')^{-1} \Q^{\top} \Q \D' \Q^{\top} \\
	\end{split}
	\end{align}
	$$
	\implies \lambda((\C + \D)^{-1} \D) = \lambda((\C' + \D')^{-1} \D')
	$$
	
	Without loss of generality we assume $\D'$ has all its diagonal entries in descending order, and $\D'$ may have zero diagonal entries. Consider below, where $\D_1$ is a diagonal matrix as well and contains all the positive diagonal entries of $\D'$.
	
	$$
	\D' = 
	\begin{bmatrix}
	\D_1 & \mathbf{0} \\
	\mathbf{0} & \mathbf{0}
	\end{bmatrix}, 
	\C' = 
	\begin{bmatrix}
	\X & \Y^{\top} \\
	\Y & \Z
	\end{bmatrix} \\
	$$	
	
	\begin{align} \label{eq:39}
	\begin{split}
	\C' + \D' &= 
	\begin{bmatrix}
	\X  + \D_1 & \Y^{\top} \\
	\Y & \Z
	\end{bmatrix}, \\
	(\C' + \D')^{-1} &= 
	\begin{bmatrix}
	(\D_1 + \X - \Y^{\top} \Z^{-1} \Y)^{-1} & * \\
	* & *
	\end{bmatrix}
	=: 
	\begin{bmatrix}
	(\D_1 + \mathbf{S})^{-1} & * \\
	* & *
	\end{bmatrix}, \\
	\end{split}
	\end{align}
	
	\begin{align}
	\begin{split}
	(\C' + \D')^{-1} \D' 
	&= 
	\begin{bmatrix}
	(\D_1 + \mathbf{S})^{-1} & * \\
	* & *
	\end{bmatrix} 
	\begin{bmatrix}
	\D_1 & \mathbf{0} \\
	\mathbf{0} & \mathbf{0} \\
	\end{bmatrix} 
	= 
	\begin{bmatrix}
	(\I + \D_1^{-1} \mathbf{S} )^{-1} & * \\
	* & *
	\end{bmatrix} 
	\end{split}
	\end{align}
	
	We will use this fact below: for product of two matrices, for any two operators $\A, \B$ on Hilbert space $\mathcal{H}$ with dimension $n$, for all $i, j$ such that $i+j \leq n+1,  \lambda_{i+j-1}(\A \B) \leq \lambda_i(\A) \lambda_j(\B)$, where $\lambda_i(\A)$ is the $i$-th largest eigenvalue of $\A$ (Bhatia, 1997).
	
	Therefore, for any $i \in \{1, 2, \cdots, r \}$ where $r = \text{rank}(\M_2')$, and let $\lambda_i^{\downarrow}(\cdot)$ denote the $i$-th largest eigenvalue of a matrix, $\lambda_i^{\uparrow}(\cdot)$ denote the $i$-th smallest eigenvalue of a matrix
	\begin{align}
	\begin{split}
	\lambda_i^{\downarrow}(\M_2') &= \lambda_i^{\downarrow}((\I + \D_1^{-1} \mathbf{S} )^{-1}) \\
	&= \frac{1}{1 + \lambda_{i}^{\uparrow}(\D_1^{-1} \mathbf{S})} = \frac{1}{1 + \frac{1}{\lambda_i^{\downarrow}(\mathbf{S}^{-1} \D_1)}} \\
	&\leq \frac{1}{1 + \frac{1}{\lambda_k^{\downarrow}(\mathbf{S}^{-1}) \lambda_j^{\downarrow} (\D_1) }} ,\ \forall j, k \in \{1, 2, \cdots, r \} \text{ s.t.} j+k = i+1 \\
	&\leq  \frac{1}{1 + \frac{1}{\lambda_k^{\downarrow}(\mathbf{\C'}^{-1}) \lambda_j^{\downarrow} (\D_1)}} \\
	&= \frac{1}{1 + \frac{1}{\lambda_k^{\downarrow}(\mathbf{\C}^{-1}) \lambda_j^{\downarrow} (\D_1)}} = \frac{1}{1 + \frac{\lambda_k^{\uparrow}(\C)}{\lambda_j^{\downarrow}(\D_1)}}
	\end{split}
	\end{align}
	where the first inequality follows from the fact stated above, the second inequality holds because $\mathbf{S}^{-1} = (\X - \Y^{\top} \Z^{-1} \Y)^{-1}$ is a principal submatrix of $\C'^{-1}$, and the second last equality holds because $\lambda(\C') = \lambda(\C)$.
	
	For each fixed $i$, pick $j = i$, $k = 1$, then we have
	\begin{align}
	\begin{split}
	\lambda_i^{\downarrow}(\M_2') &\leq \frac{1}{1 + \frac{\lambda_1^{\uparrow}(\C)}{\lambda_i^{\downarrow}(\D_1)}} 
	= \frac{1}{1 +  \frac{\lambda_1^{\uparrow}(\C)}{\lambda_i^{\downarrow}(\D)}}
	= \frac{1}{1 + \frac{\frac{1}{1 + d_{1, 1}^2}}{d_{2, i}^2}} = \frac{d_{2,i}^2 (1 + d_{1, 1}^2)}{d_{2,i}^2 (1 + d_{1, 1}^2) + 1}
	\end{split}
	\end{align}
	
	Thus
	\begin{align}
	\begin{split}
	\lambda_i^{\downarrow}((\I_p - \M_2')^{-1}) &= \frac{1}{1 - \lambda_i^{\downarrow}(\M_2')} \leq \frac{1}{1 - \frac{d_{2,i}^2 (1 + d_{1, 1}^2)}{d_{2,i}^2 (1 + d_{1, 1}^2) + 1}} = d_{2,i}^2 ( d_{1, 1}^2 + 1) + 1 ,\, i = 1, 2, \cdots, r \\ 
	\Tr((\I_p - \M_2)^{-1})
	&= \sum_{i = 1}^p \lambda_i^{\downarrow}((\I_p - \M_2)^{-1})
	= \sum_{i = 1}^r \frac{1}{1 - \lambda_i^{\downarrow}(\M_2)} + \sum_{i = r+1}^p \frac{1}{1 - 0} \\
	&\leq \sum_{i = 1}^r [d_{2,i}^2 (1 + d_{1, 1}^2) + 1] + (p-r) 
	= p + \sum_{i = 1}^r d_{2,i}^2 (1 + d_{1, 1}^2)
	\end{split}
	\end{align}
\end{proof}

\subsection{General $d$ hidden layers}

Now we consider the general $d$-layer setup, formally described below,
\begin{align} \label{eq:44}
\begin{split}
y &\sim \text{Uniform}\{-1, +1\} \\
\z_0 | y &\sim N(y \tilde{\w}_0, \text{covar} = \sigma^2 \I_{n_0}), \tilde{\w}_0 \in \R^{n_0} \\
\z_1 | \z_0 &\sim N(\tilde{\W}_1 \z_0, \text{covar} = \sigma^2 \I_{n_1}), \tilde{\W}_1 \in \R^{n_1 \times n_0} \\
\z_2 | \z_1 &\sim N(\tilde{\W}_2 \z_1, \text{covar} = \sigma^2 \I_{n_2}), \tilde{\W}_2 \in \R^{n_2 \times n_1} \\
&\cdots \\
\x := \z_d | \z_{d-1} &\sim N(\tilde{\W}_{d}\z_{d-1}, \text{covar} = \sigma^2 \I_{n_d} = \sigma^2 \I_{p}), \tilde{\W}_{d} \in \R^{n_{d} \times n_{d-1}} = \R^{p \times n_{d-1}}
\end{split}
\end{align}

As shown above, the dimension of the "input" $\x$ is $n_d$, and later we will use $p$ and $n_d$ interchangeably, as most readers are familiar with $p$ being the dimension of the feature. We will also sometimes use $\Tilde{\W}$ as a shorthand for $(\tilde{\w}_0, \tilde{\W}_1, \tilde{\W}_2, \cdots, \tilde{\W}_{d-1}, \tilde{\W}_{d})$. We will also use $\x$ and $\z_d$ interchangeably.

For the sake of simplicity, we consider the case where all $\tilde{\W}_i (i = 1, \cdots, d)$ have the same rank $r$, i.e., $\text{rank}(\tilde{\W}_1) = \text{rank}(\tilde{\W}_2) = \cdots = \text{rank}(\tilde{\W}_d) = r$.

We now prove Lemma 3.5 in the main text.

\subsubsection{Proof of Lemma 3.5 in the main text}

\begin{proof}
	We prove this lemma by induction. Let $\boldsymbol{\mu}_{\ell} = y \tilde{\W}_{\ell} \cdots \tilde{\W}_1 \tilde{\w}_0$, i.e. $\boldsymbol{\mu}_{\ell}$ is the mean of $\z_i$ under the marginal distribution, $\ell \in \{0, 1, \cdots, d\}$, and $\boldsymbol{\mu}_0 = y \tilde{\w}_0$. Note $\boldsymbol{\mu}_{\ell+1} = \tilde{\W}_{\ell+1} \boldsymbol{\mu}_{\ell}$. The base cases for $d=1$ and $d=2$ are proved, as for the 1-layer network we have
	\begin{align}
	\begin{split}
	p(\z_0 |y ; \tilde{\W}) p(\z_1 | \z_0 ; \tilde{\W}) \propto \exp\left( -\frac{1}{2} (\z_0 - \boldsymbol{\mu}_0, \z_1 - \boldsymbol{\mu}_1)^{\top} \kappa^{(1)} (\z_0 - \boldsymbol{\mu}_0, \z_1 - \boldsymbol{\mu}_1) \right) 			
	\end{split}
	\end{align}
	
	and for the 2-layer network we have
	\begin{align}
	\begin{split}
	&p(\z_0, \z_1, \z_2 | y ; \tilde{\w}_0, \tilde{\W}_1, \tilde{\W}_2) 
	= p(\z_0 |y ; \tilde{\W}) p(\z_1 | \z_0 ; \tilde{\W}) p(\z_2 | \z_1 ; \tilde{\W})\\
	\propto& \exp\left( -\frac{1}{2} (\z_0 - \boldsymbol{\mu}_0, \z_1 - \boldsymbol{\mu}_1, \z_2 - \boldsymbol{\mu}_2)^{\top} \kappa^{(2)} (\z_0 - \boldsymbol{\mu}_0, \z_1 - \boldsymbol{\mu}_1, \z_2 - \boldsymbol{\mu}_2) \right) \\	
	\end{split}
	\end{align}

	Let the statement in the lemma be the inductive hypothesis (\textbf{IH}) and now we want to show the (\textbf{IH}) holds for $d+1$ layers.
	
	\begin{align} \label{eq:65}
	\begin{split}
	&p(\z_0, \z_1, \z_2, \cdots, \z_d, \z_{d+1} | y ; \tilde{\w}_0, \tilde{\W}_1, \tilde{\W}_2, \cdots, \tilde{\W}_{d-1}, \tilde{\W}_d, \tilde{\W}_{d+1}) \\
	= & p(\z_0, \z_1, \z_2, \cdots, \z_d| y ; \tilde{\w}_0, \tilde{\W}_1, \tilde{\W}_2, \cdots, \tilde{\W}_{d-1}, \tilde{\W}_d) p(\z_{d+1} | \z_d ; \tilde{\W}_{d+1}) \\
	\stackrel{(\textbf{IH})}{\propto} 
	&  \exp \left( -\frac{1}{2} \left[ (\z_0 - \boldsymbol{\mu}_0, \cdots, \z_{d} - \boldsymbol{\mu}_d)^{\top} \kappa^{(d)} (\z_0 - \boldsymbol{\mu}_0, \cdots, \z_{d} - \boldsymbol{\mu}_d) 
	+ (\z_{d+1} - \tilde{\W}_{d+1} \z_d)^{\top} \tilde{\mathbf{\Sigma}}_{d+1} (\z_{d+1} - \tilde{\W}_{d+1} \z_d) \right] \right) 
	\end{split}
	\end{align}
	
	Now want to show the RHS of \eqref{eq:65} is proportional to the proposed density in Lemma 3.5 in the main text. That is, we want to show that they have the same exponent. Note that
	
	\begin{align}
	\begin{split}
	\kappa^{(d+1)} 
	&=
	\left(
	\renewcommand{\arraystretch}{1.5}
	\begin{array}{ccc|c}
	&  &  &  \mathbf{0} \\
	\multicolumn{3}{c|}{\kappa^{(d)}} & \vdots \\
	&  &  &  \mathbf{0} \\
	\cline{1-3}
	\mathbf{0} & \dots & \multicolumn{1}{c}{\mathbf{0}} & \mathbf{0}
	\end{array}
	\right) 
	+ \left(
	\renewcommand{\arraystretch}{1.5}
	\begin{array}{cccc|c}
	\mathbf{0} & \cdots &  & \mathbf{0} & \mathbf{0} \\
	\vdots & \ddots &  & \vdots & \vdots \\
	&  & \mathbf{0} & \mathbf{0} & \mathbf{0} \\
	\mathbf{0} & \cdots & \mathbf{0} & \tilde{\W}_{d+1}^{\top} \mathbf{\Sigma}_{d+1} \tilde{\W}_{d+1} &  - (\Tilde{\mathbf{\Sigma}}_{d+1} \Tilde{\W}_{d+1}) \\
	\cline{1-4}
	\mathbf{0} & \dots & \mathbf{0} & \multicolumn{1}{c}{-(\Tilde{\mathbf{\Sigma}}_{d+1} \Tilde{\W}_{d+1})^{\top}} & \Tilde{\mathbf{\Sigma}}_{d+1}
	\end{array}
	\right) 
	=: \tilde{\kappa}^{(d+1)}_1 + \tilde{\kappa}^{(d+1)}_2
	\end{split}
	\end{align}
	
	Thus the exponent of the proposed density in Lemma 3.5 in the main text, ignoring the factor of $-\frac{1}{2}$, is 
	\begin{align}
	\begin{split}
	& (\z_0 - \boldsymbol{\mu}_0, \cdots, \z_{d+1} - \boldsymbol{\mu}_{d+1}) ^{\top} \kappa^{(d+1)} (\z_0 - y \tilde{\w}_0, \cdots, \z_{d+1} - y \tilde{\W}_{d+1} \cdots \tilde{\W}_1 \tilde{\w}_0) \\
	=& \left(\z_0 - \boldsymbol{\mu}_0, \cdots, \z_{d+1} - \boldsymbol{\mu}_{d+1} \right)^{\top} \left( \tilde{\kappa}^{(d+1)}_1 + \tilde{\kappa}^{(d+1)}_2 \right) \left(\z_0 - \boldsymbol{\mu}_0, \cdots, \z_{d+1} - \boldsymbol{\mu}_{d+1} \right) \\
	=& \left(\z_0 - \boldsymbol{\mu}_0, \cdots, \z_{d} - \boldsymbol{\mu}_{d} \right)^{\top} \kappa^{(d)} \left(\z_0 - \boldsymbol{\mu}_0, \cdots, \z_{d} - \boldsymbol{\mu}_{d} \right) 
	+ \left(\z_{d+1} - \boldsymbol{\mu}_{d+1} \right)^{\top} \Tilde{\mathbf{\Sigma}}_{d+1} \left(\z_{d+1} - \boldsymbol{\mu}_{d+1} \right) \\
	& + \left( \z_{d} - \boldsymbol{\mu}_{d} \right)^{\top} \left( -(\Tilde{\mathbf{\Sigma}}_{d+1} \Tilde{\W}_{d+1})^{\top} \right) \left( \z_{d+1} - \boldsymbol{\mu}_{d+1} \right) 
	+ \left(\z_{d+1} - \boldsymbol{\mu}_{d+1} \right)^{\top} \left( -(\Tilde{\mathbf{\Sigma}}_{d+1} \Tilde{\W}_{d+1})^{\top} \right) \left(\z_{d} - \boldsymbol{\mu}_{d} \right) \\
	& + \left(\z_{d} - \boldsymbol{\mu}_{d} \right)^{\top} \left( \tilde{\W}_{d+1}^{\top} \mathbf{\Sigma}_{d+1} \tilde{\W}_{d+1} \right)
	\left(\z_{d} - \boldsymbol{\mu}_{d} \right) \\
	=& \left(\z_0 - \boldsymbol{\mu}_0, \cdots, \z_{d} - \boldsymbol{\mu}_{d} \right)^{\top} \kappa^{(d)} \left(\z_0 - \boldsymbol{\mu}_0, \cdots, \z_{d} - \boldsymbol{\mu}_{d} \right) 
	+ \left(\z_{d+1} - \boldsymbol{\mu}_{d+1} \right)^{\top} \Tilde{\mathbf{\Sigma}}_{d+1} \left(\z_{d+1} - \boldsymbol{\mu}_{d+1} \right) \\
	& - 2 \left( \tilde{\W}_{d+1} \z_d - \boldsymbol{\mu}_{d+1} \right)^{\top} \Tilde{\mathbf{\Sigma}}_{d+1} \left( \z_{d+1} - \boldsymbol{\mu}_{d+1} \right) + \left( \tilde{\W}_{d+1} \z_d - \boldsymbol{\mu}_{d+1} \right)^{\top} \Tilde{\mathbf{\Sigma}}_{d+1} \left( \tilde{\W}_{d+1} \z_d - \boldsymbol{\mu}_{d+1} \right) \\
	=& \left(\z_0 - \boldsymbol{\mu}_0, \cdots, \z_{d} - \boldsymbol{\mu}_{d} \right)^{\top} \kappa^{(d)} \left(\z_0 - \boldsymbol{\mu}_0, \cdots, \z_{d} - \boldsymbol{\mu}_{d} \right) \\
	&+ \left( \z_{d+1}^{\top} \Tilde{\mathbf{\Sigma}}_{d+1} \z_{d+1} 
	- 2 \boldsymbol{\mu}_{d+1}^{\top} \Tilde{\mathbf{\Sigma}}_{d+1} \z_{d+1} + \boldsymbol{\mu}_{d+1}^{\top} \Tilde{\mathbf{\Sigma}}_{d+1} \boldsymbol{\mu}_{d+1} \right) \\
	& - 2 \left( \tilde{\W}_{d+1} \z_d \right)^{\top} \Tilde{\mathbf{\Sigma}}_{d+1} \z_{d+1} 
	+ 2 \boldsymbol{\mu}_{d+1}^{\top} \Tilde{\mathbf{\Sigma}}_{d+1} \z_{d+1} 
	+ 2 \left( \tilde{\W}_{d+1} \z_d \right)^{\top} \Tilde{\mathbf{\Sigma}}_{d+1} \boldsymbol{\mu}_{d+1} 
	- 2 \boldsymbol{\mu}_{d+1}^{\top} \Tilde{\mathbf{\Sigma}}_{d+1} \boldsymbol{\mu}_{d+1} \\
	&+ \left( \tilde{\W}_{d+1} \z_d \right)^{\top} \Tilde{\mathbf{\Sigma}}_{d+1} \left( \tilde{\W}_{d+1} \z_d \right) 
	- 2 \left( \tilde{\W}_{d+1} \z_d \right)^{\top} \Tilde{\mathbf{\Sigma}}_{d+1} \boldsymbol{\mu}_{d+1} 
	+ \boldsymbol{\mu}_{d+1}^{\top} \Tilde{\mathbf{\Sigma}}_{d+1} \boldsymbol{\mu}_{d+1} \\
	=& \left( \z_0 - \boldsymbol{\mu}_0, \cdots, \z_{d} - \boldsymbol{\mu}_{d} \right)^{\top} \kappa^{(d)} \left( \z_0 - \boldsymbol{\mu}_0, \cdots, \z_{d} - \boldsymbol{\mu}_{d} \right) 
	+ \z_{d+1}^{\top} \Tilde{\mathbf{\Sigma}}_{d+1} \z_{d+1} \\
	& - 2 \left( \tilde{\W}_{d+1} \z_d \right)^{\top} \Tilde{\mathbf{\Sigma}}_{d+1} \z_{d+1} 
	+ \left( \tilde{\W}_{d+1} \z_d \right)^{\top} \Tilde{\mathbf{\Sigma}}_{d+1} \left( \tilde{\W}_{d+1} \z_d \right)
	\end{split}
	\end{align}
	which is exactly the exponent of $p(\z_0, \cdots , \z_{d} | y; \tilde{\W}) p(\z_{d+1} | \z_d; \tilde{\W})$, ignoring the factor of $-\frac{1}{2}$, thus \textbf{(IH)} holds. 
\end{proof}

Now we proved the joint distribution of $(\z_0, \z_1, \z_2, \cdots, \z_d) | y ; \tilde{\W}$, we then prove Lemma 3.6 in the main text.

\subsubsection{Proof of Lemma 3.6 in the main text}

\begin{proof}
	We only need to show that the covariance of $\x | y ; \tilde{\W}$ is $\sigma^2 \left( \I_p - \M_d(\tilde{\W}) \right)^{-1}$. Note that we are interested in $\left( \kappa^{(d)} \right)^{-1}_{d,d}$, the $(d,d)$-th block of the inverse of $\kappa^{(d)}$, where $\kappa^{(d)}$ is the precision matrix of $(\z_0, \z_1, \z_2, \cdots, \z_d) | y ; \tilde{\W}$. We will make use of the tri-diagonal block structure of $\kappa^{(d)}$.
	
	We will apply the block matrix inverse formula repeatedly. We start with the submatrix $\kappa^{(d)}_{[0:1], [0:1]}$, which denotes 
	$\begin{bmatrix}
	\kappa^{(d)}_{0,0} & \kappa^{(d)}_{0,1} \\
	\kappa^{(d)}_{1,0} & \kappa^{(d)}_{1,1} 
	\end{bmatrix}$
	as defined in eq (23) in the main text. We first look for the bottom right block in the inverse of $\kappa^{(d)}_{[0:1], [0:1]}$, which we denote as $\left( \kappa^{(d)}_{[0:1], [0:1]} \right)^{-1}_{1,1}$. By the block matrix inverse formula, we know 
	\begin{align}
	\begin{split}
	\left( \kappa^{(d)}_{[0:1], [0:1]} \right)^{-1}_{1,1} 
	&= 
	\left( \kappa^{(d)}_{1,1} - \kappa^{(d)}_{1,0} \left( \kappa^{(d)}_{0,0} \right)^{-1} \kappa^{(d)}_{0,1} \right)^{-1} \\
	&= \left( \left( \Tilde{\mathbf{\Sigma}}_1 + \Tilde{\W}_{2}^{\top} \Tilde{\mathbf{\Sigma}}_{2} \Tilde{\W}_{2} \right) 
	- \left( -(\Tilde{\mathbf{\Sigma}}_{i+1} \Tilde{\W}_{i+1}) \right) 
	\left( \Tilde{\mathbf{\Sigma}}_0 + \Tilde{\W}_{1}^{\top} \Tilde{\mathbf{\Sigma}}_{1} \Tilde{\W}_{1} \right) 
	\left( -(\Tilde{\mathbf{\Sigma}}_{i+1} \Tilde{\W}_{i+1}) \right)^{\top} \right)^{-1} \\
	&= \sigma^2 \left( \I_{n_1} + \Tilde{\W}_{2}^{\top} \Tilde{\W}_{2} - \M_1(\tilde{\W}) \right)^{-1}
	\end{split}
	\end{align}
	
	Now we start to make use of the tri-diagonal block structure of $\kappa^{(d)}$. We calculate $\left( \kappa^{(d)}_{[0:2], [0:2]} \right)^{-1}_{2,2}$ using $\left( \kappa^{(d)}_{[0:1], [0:1]} \right)^{-1}_{1,1}$:
	
	\begin{align}
	\begin{split}
	\left( \kappa^{(d)}_{[0:2], [0:2]} \right)^{-1}_{2,2}
	&= \left( \kappa^{(d)}_{2, 2} - \kappa^{(d)}_{2, [0:1]} \left( \kappa^{(d)}_{[0:1], [0:1]} \right)^{-1} \kappa^{(d)}_{[0:1], 2} \right)^{-1} \\
	&= \left( \kappa^{(d)}_{2, 2} - \kappa^{(d)}_{2, 1} \left( \kappa^{(d)}_{[0:1], [0:1]} \right)^{-1}_{1,1} \kappa^{(d)}_{1, 2} \right)^{-1} \\
	&= \left( \left( \Tilde{\mathbf{\Sigma}}_2 + \Tilde{\W}_{3}^{\top} \Tilde{\mathbf{\Sigma}}_{3} \Tilde{\W}_{3} \right) 
	- \left( -\Tilde{\mathbf{\Sigma}}_{2} \Tilde{\W}_{2} \right) 
	\sigma^2 \left( \I_{n_1} + \Tilde{\W}_{2}^{\top} \Tilde{\W}_{2} - \M_1(\tilde{\W}) \right)^{-1} 
	\left( -\Tilde{\mathbf{\Sigma}}_{2} \Tilde{\W}_{2} \right)^{\top} \right)^{-1} \\
	&= \sigma^2 \left( \I_{n_2} + \Tilde{\W}_{3}^{\top} \Tilde{\W}_{3} - \M_2(\tilde{\W}) \right)^{-1}
	\end{split}
	\end{align}
	where $\kappa^{(d)}_{2, [0:1]} = \begin{bmatrix}
	\kappa^{(d)}_{2, 0} & \kappa^{(d)}_{2, 1}
	\end{bmatrix}$ and $\kappa^{(d)}_{[0:1], 2} = \begin{bmatrix}
	\kappa^{(d)}_{0, 2} \\
	\kappa^{(d)}_{1, 2}
	\end{bmatrix}$, and the second equality follows from the fact that $\kappa^{(d)}_{2, 0} = [\mathbf{0}]$ and $\kappa^{(d)}_{0, 2} = [\mathbf{0}]$.
	
	It is easy to see this pattern would hold for any general $\ell \in \{ 1, 2, \cdots, d-1 \}$, i.e., 
	\begin{align}
	\begin{split}
	\left( \kappa^{(d)}_{[0: \ell], [0: \ell]} \right)^{-1}_{\ell, \ell} = \sigma^2 \left( \I_{n_{\ell}} + \Tilde{\W}_{\ell+1}^{\top} \Tilde{\W}_{\ell + 1} - \M_{\ell}(\tilde{\W}) \right)^{-1}
	\end{split}
	\end{align}
	
	Therefore, we have
	
	\begin{align}
	\begin{split}
	\left( \kappa^{(d)} \right)^{-1}_{d,d} 
	&= \left( \kappa^{(d)}_{d, d} - \kappa^{(d)}_{d, [0:d-1]} \left( \kappa^{(d)}_{[0:d-1], [0:d-1]} \right)^{-1} \kappa^{(d)}_{[0:d-1], d} \right)^{-1} \\
	&= 
	\left( \kappa^{(d)}_{d, d} - \kappa^{(d)}_{d, d-1} \left( \kappa^{(d)}_{[0:d-1], [0:d-1]} \right)^{-1}_{d-1,d-1} \kappa^{(d)}_{d-1, d} \right)^{-1} \\ 
	&= \left( 
	\Tilde{\mathbf{\Sigma}}_{d} - 
	\left( -(\Tilde{\mathbf{\Sigma}}_{d} \Tilde{\W}_{d}) \right) 
	\sigma^2 \left( \I_{n_{d-1}} + \Tilde{\W}_{d}^{\top} \Tilde{\W}_{d} - \M_{d-1}(\tilde{\W}) \right)^{-1} 
	\left( -(\Tilde{\mathbf{\Sigma}}_{d} \Tilde{\W}_{d}) \right)^{\top}
	\right)^{-1} \\
	&= \sigma^2 \left( \I_{n_d} - \M_d(\tilde{\W}) \right)^{-1} = \sigma^2 \left( \I_{p} - \M_d(\tilde{\W}) \right)^{-1} 
	\end{split}
	\end{align}
	which is the desired result.
\end{proof}

\subsubsection{A upper bound for Kl divergence between $(\x,y) ; \tilde{\W}$ and a prior distribution $\mathbb{Q}$}

With the marginal distribution of  $\x | y ; \tilde{\W}$, and using the fact that the label $y \sim \text{Unit}\{ -1, +1 \}$ is first generated, we can find the distribution of $(\x, y); \tilde{\W}$ and then obtain an upper bound of the KL divergence between itself and a prior distribution $\mathbb{Q}$. This upper bound will be used in our sample complexity lower bound.

\begin{lemma} [Upper bound on KL divergence between $(\x,y) ; \tilde{\W}$ and a prior distribution $\mathbb{Q}$] \label{lemma:4.10}
	We have that
	$$\mathbb{KL}(\mathbb{P}_{(\x, y);\tilde{\W}} || \mathbb{Q}) \leq \frac{1}{2} \left[ \frac{\sigma^2}{\tau^2} \left( p + \sum_{i = 1}^r m_{d,i} \right)  + \frac{1}{\tau^2} \|  \Tilde{\W}_d  \Tilde{\W}_{d-1} \cdots  \Tilde{\W}_1 \Tilde{\w}_0 \|_2^2 + p \ln \left(\frac{\tau^2}{\sigma^2} \right) -p \right]$$ 
	where $\mathbb{Q} \sim N(\mathbf{0}, \tau^2 \I_{p}) \times \text{Uniform}\{-1, +1\}$ with $\tau$ being a fixed constant is a prior distribution, $\mathbb{P}_{(\x, y);\tilde{\W}}$ is the joint distribution of $(\x, y)$ when $\tilde{\W} = (\tilde{\w}_0, \tilde{\W}_1, \tilde{\W}_2, \cdots, \tilde{\W}_{d-1}, \tilde{\W}_{d})$ is viewed as a parameter, $r = \text{rank}(\tilde{\W}_1) = \cdots = \text{rank}(\tilde{\W}_d)$, and $m_{d,i}$ is defined recursively as:
	\begin{align} \label{eq:59}
	\begin{split}
	m_{1, i} &:= d_{1, i}^2 \\
	m_{2,i} &:= d_{2, i}^2 (m_{1, 1} + 1) = d_{2, i}^2 (d_{1,1}^2 + 1) \\
	m_{3,i} &:= d_{3, i}^2 (m_{2, 1} + 1)= d_{3, i}^2 (d_{2,1}^2 (d_{1,1}^2 + 1) +1)  \\
	& \cdots \\
	m_{d,i} &:= d_{d, i}^2 (m_{d-1, 1} + 1)
	\end{split}
	\end{align}
	where $d_{\ell, i}$ is the $i$-th largest singular value of $\tilde{\W}_{\ell}$, $\ell \in \{1, 2, \cdots, d\}$.
\end{lemma}

\begin{proof}
	
	Similar to the proof of Lemma \ref{lemma:4.5}, we use the fact that 
	\begin{align}
	\begin{split}
	\mathbb{KL}(\mathbb{P}_{(\x, y);\tilde{\W}} || \mathbb{Q}) 
	=& \sum_{y \ \in \{-1, +1\}} \int p(\x, y) \log \frac{p(\x, y)}{q(\x, y)} d\x \\
	=& \frac{1}{2} \mathbb{KL}(\mathbb{P}_{(\x | y = -1);\tilde{\W}} || N(\mathbf{0}, \tau^2 \I_{p})) + \frac{1}{2} \mathbb{KL}(\mathbb{P}_{(\x | y = +1);\tilde{\W}} || N(\mathbf{0}, \tau^2 \I_{p})) \\
	=& \mathbb{KL}(\mathbb{P}_{(\x | y = -1);\tilde{\W}} || N(\mathbf{0}, \tau^2 \I_{p})) \\
	=& 
	\frac{1}{2} \frac{\sigma^2}{\tau^2} \underbrace{ \Tr \left( ( \I_p - \M_d(\tilde{\W}) )^{-1} \right)}_\textbf{I} \\
	&+ \frac{1}{2} \underbrace{\left( \left(- \Tilde{\W}_d  \Tilde{\W}_{d-1} \cdots  \Tilde{\W}_1 \Tilde{\w}_0 \right)^{\top} \left( \tau^2 \I_{p} \right)^{-1} \left(- \Tilde{\W}_d  \Tilde{\W}_{d-1} \cdots  \Tilde{\W}_1 \Tilde{\w}_0 \right) \right)}_\textbf{II} \\
	&+ \frac{1}{2} \underbrace{\ln \left( \frac{\text{det}(\tau^2 \I_{p})}{\text{det} (\text{Cov}(\x| y)) 
		} \right) 
	}_\textbf{III} - \frac{p}{2}
	\end{split}
	\end{align}
	
	For term \textbf{I}, the base case is proved in Lemma \ref{lemma:4.2} and Lemma \ref{lemma:4.5}. Now we prove the inductive case for the general $d$-layer setting. From Lemma \ref{lemma:4.7}, we know when $d = 2$, 
	\begin{align}
	\begin{split}
	i = 1, 2, \cdots, r :& \lambda_i^{\downarrow}((\I_p - \M_2(\tilde{\W}))^{-1}) = \frac{1}{1 - \lambda_i^{\downarrow}(\M_2'(\tilde{\W}))} \leq d_{2,i}^2 (d_{1, 1}^2 + 1) + 1 = m_{2,i}+1 ,\, \\
	i = r+1, r+2, \cdots, p:& \lambda_i^{\downarrow}((\I_p - \M_2(\tilde{\W}))^{-1}) = 1  
	\end{split}
	\end{align}
	
	We defined the constants $m_{d,i}$ recursively in \eqref{eq:59}, and we let the inductive hypothesis (\textbf{IH}) be:
	
	\begin{align}
	\begin{split}
	\forall i \in \{1, 2, \cdots,r\}: \, \lambda_i^{\downarrow}((\I_p - \M_d(\tilde{\W}))^{-1}) \leq m_{d,i} + 1
	\end{split}
	\end{align}
	and we see (\textbf{IH}) holds for $d=2$. Now suppose (\textbf{IH}) holds for any $d \geq 2$, and we want to show it holds for $d+1$ as well.
	
	Recall 
	$$\M_{d+1}(\tilde{\W}) := \tilde{\W}_{d+1} \left( \I_{n_{d}} + \tilde{\W}_{d+1}^{\top} \tilde{\W}_{d+1} - \M_{d}(\tilde{\W}) \right)^{-1} \tilde{\W}_{d+1}^{\top}$$
	$$\M_{d+1}'(\tilde{\W}) := \left( \I_{n_{d}} + \tilde{\W}_{d+1}^{\top} \tilde{\W}_{d+1} - \M_{d}(\tilde{\W}) \right)^{-1} \tilde{\W}_{d+1}^{\top} \tilde{\W}_{d+1}$$
	Follow the similar reasoning as in Lemma \ref{lemma:4.7} equations \eqref{eq:35} - \eqref{eq:39},
	\begin{align}
	\begin{split}
	\forall i \in \{1, 2, \cdots,r\}: & \, \forall j,k \in \{1, 2, \cdots,r\} \text{ s.t. } j+k=i+1, \\ &\lambda^{\downarrow}_i(\M_{d+1}'(\tilde{\W})) \leq \frac{1}{1 + \frac{1}{\lambda^{\downarrow}_j((\I_p - \M_d(\tilde{\W}))^{-1})\lambda^{\downarrow}_k(\tilde{\W}_{d+1}^{\top} \tilde{\W}_{d+1}) } } \\
	\forall i \in \{ r+1, r+2, \cdots, p \}: & \, \lambda^{\downarrow}_i(\M_{d+1}'(\tilde{\W})) = 0 \\ 
	\end{split}
	\end{align}
	
	For each $i \in \{1, 2, \cdots,r\}$, pick $j = i$ and $k = 1$, get
	\begin{align} \label{eq:64}
	\begin{split}
	\lambda^{\downarrow}_i(\M_{d+1}'(\tilde{\W})) 
	&\leq \frac{1}{1 + \frac{1}{\lambda^{\downarrow}_1((\I_p - \M_d(\tilde{\W}))^{-1})\lambda^{\downarrow}_i(\tilde{\W}_{d+1}^{\top} \tilde{\W}_{d+1}) } } \\
	& \stackrel{\textbf{(IH)}}{\leq}  \frac{1}{1 + \frac{1}{(m_{d,1} + 1) \lambda^{\downarrow}_i(\tilde{\W}_{d+1}^{\top} \tilde{\W}_{d+1}) } } \\
	&= \frac{1}{1 + \frac{1}{(m_{d,1} + 1) d_{d+1, i}^2 } } = \frac{(m_{d,1} + 1) d_{d+1, i}^2}{1 + (m_{d,1} + 1) d_{d+1, i}^2}
	\end{split}
	\end{align}

	Therefore, for $i \in \{1, 2, \cdots,r\}$
	\begin{align}
	\begin{split}
	\lambda_i^{\downarrow}((\I_p - \M_{d+1}'(\tilde{\W}))^{-1}) 
	&= \frac{1}{1 - \lambda_i^{\downarrow}(\M_{d+1}'(\tilde{\W}))} \\
	&\leq \frac{1}{1 - \frac{(m_{d,1} + 1) d_{d+1, i}^2}{1 + (m_{d,1} + 1) d_{d+1, i}^2}} = d_{d+1,i}^2 (m_{d,1} + 1) + 1  = m_{d+1,i} + 1
	\end{split}
	\end{align}
	which says (\textbf{IH}) is true for all $d \geq 2$, and 
	\begin{align}
	\begin{split}
	\forall i \in \{r+1, r+2, \cdots, p\}: \, \lambda_i^{\downarrow}((\I_p - \M_{d+1}'(\tilde{\W}))^{-1}) = 1
	\end{split}
	\end{align}
	
	Therefore
	\begin{align}
	\begin{split}
	\textbf{I} = \Tr \left((\I_p - \M_d(\tilde{\W}))^{-1} \right)
	&= \sum_{i = 1}^p \lambda_i^{\downarrow}((\I_p - \M_d(\tilde{\W}))^{-1})
	= \sum_{i = 1}^r \frac{1}{1 - \lambda_i^{\downarrow}(\M_d(\tilde{\W}))} + \sum_{i = r+1}^p \frac{1}{1 - 0} \\ 
	&\leq \sum_{i = 1}^r \left( m_{d, i} + 1 \right) + (p-r) 
	= p + \sum_{i = 1}^r m_{d, i}
	\end{split}
	\end{align}
	
	It is easy to see that \textbf{II} = $\frac{1}{\tau^2} \| \Tilde{\W}_d  \Tilde{\W}_{d-1} \cdots  \Tilde{\W}_1 \Tilde{\w}_0 \|_2^2$.
	
	Now we provide an upper bound on term \textbf{III}, where the last inequality follows from \eqref{eq:64}
	\begin{align}
	\begin{split}
	\textbf{III} 
	&= \ln \left( \frac{\text{det}(\tau^2 \I_{p})}{\text{det} \left( \text{Cov}(\x| y) \right)} \right) 
	= p \ln \left( \frac{\tau^2}{\sigma^2} \right) + \ln \left( \text{det} \left( \I_p - \M_d(\tilde{\W}) \right) \right) \\
	&= p \ln \left( \frac{\tau^2}{\sigma^2} \right) + \ln \left( \prod_{i=1}^p (1 - \lambda_i(\M_d(\tilde{\W}))) \right)
	\leq p \ln \left( \frac{\tau^2}{\sigma^2} \right) 
	\end{split}
	\end{align}
	
\end{proof}

\subsubsection{Proof of Lemma 3.7 in the main text}

In the main text, we proposed a subset $\G_{p,d}$ of $\F_{p,d}$ by two restrictions \textbf{R1} and \textbf{R2}. We find that under \textbf{R1} alone we have a tighter upper bound on a KL divergence relevant to the upper bound on mutual information $\mathbb{I}(\tilde{\W} ; S)$, which is stated in Lemma 3.7 in the main text. We now prove Lemma 3.7.

\begin{proof}
	Let $p(\x, y ; \tilde{\W})$ denote the density of $(\x, y);\tilde{\W}$ and let $p(\x, y ; \tilde{\W}^{\prime})$ denote the density of $(\x, y);\tilde{\W}^{\prime}$. Similar to the proof of Lemma \ref{lemma:4.5}, we use the fact that 
	\begin{align}
	\begin{split}
	\mathbb{KL}(\mathbb{P}_{(\x, y);\tilde{\W}} || \mathbb{P}_{(\x, y);\tilde{\W}^{\prime}}) 
	=& \sum_{y \ \in \{-1, +1\}} \int p(\x, y ; \tilde{\W}) \log \frac{p(\x, y ; \tilde{\W})}{p(\x, y ; \tilde{\W}^{\prime})} d\x \\
	=& \frac{1}{2} \mathbb{KL}(\mathbb{P}_{(\x | y = -1);\tilde{\W}} || \mathbb{P}_{(\x | y = -1);\tilde{\W}^{\prime}}) + \frac{1}{2} \mathbb{KL}(\mathbb{P}_{(\x | y = +1);\tilde{\W}} || \mathbb{P}_{(\x | y = +1);\tilde{\W}^{\prime}}) \\
	\end{split}
	\end{align}
	
	Note by Lemma 3.6 in the main text, $\x | y; \tilde{\W}$ is multivariate normal with mean $y \tilde{\W}_{d} \tilde{\W}_{d-1} \cdots \tilde{\W}_2 \tilde{\w}_1$ and covariance matrix $\sigma^2 \left( \I_p - \M_d(\tilde{\W}) \right)^{-1}$, and $\M_d(\tilde{\W})$ is recursively defined in Lemma 3.6 in the main text. By easy calculation, the matrices $\M_i(\tilde{\W})$ defined in Lemma 3.6 in the main text becomes
	\begin{align} \label{eq:87}
	\begin{split}
	\M_1(\tilde{\W}) = & \tilde{\W}_1 \left( \I_{p} + \tilde{\W}_1^{\top} \tilde{\W}_1 \right)^{-1} \tilde{\W}_1^{\top} 
	= \begin{bmatrix}
	\frac{1}{2}\I_{r} & \mathbf{0} \\
	\mathbf{0} & \frac{c^2}{1+c^2} \I_{p-r} 
	\end{bmatrix} \\
	\M_2(\tilde{\W}) = & \tilde{\W}_2 \left( \I_{p} + \tilde{\W}_2^{\top} \tilde{\W}_2 - \M_1(\tilde{\W}) \right)^{-1} \tilde{\W}_2^{\top} 
	= \begin{bmatrix}
	\frac{2}{3}\I_{r} & \mathbf{0} \\
	\mathbf{0} & \frac{c^2+c^4}{1+c^2+c^4} \I_{p-r} 
	\end{bmatrix} \\
	& \cdots \\
	\M_d(\tilde{\W}) = & \tilde{\W}_d \left( \I_{n_{p}} + \tilde{\W}_{d}^{\top} \tilde{\W}_{d} - \M_{d-1}(\tilde{\W}) \right)^{-1} \tilde{\W}_d^{\top} 
	= \begin{bmatrix}
	\frac{d}{d+1}\I_{r} & \mathbf{0} \\
	\mathbf{0} & \frac{\sum_{j=1}^d c^{2j}}{1 + \sum_{j=1}^d c^{2j}} \I_{p-r} 
	\end{bmatrix}
	\end{split}
	\end{align}
	which in turn gives the exact values of the eigenvalues of $\left( \I_p - \M_d(\tilde{\W}) \right)^{-1}$.
	
	It is clear that, although $\M_d(\tilde{\W})$ depends on $\tilde{\W}$ in the recursive definition, this dependency is essentially gone under \textbf{R1}. That is, for any $\tilde{\W} , \tilde{\W}^{\prime} \in \F$, $\M_d(\tilde{\W}) = \M_d(\tilde{\W}^{\prime})$, i.e. $\x | y; \tilde{\W}$ has the same covariance regardless of the choice of $\tilde{\W}$. Therefore the KL divergence between $(\x | y = -1);\tilde{\W}$ and $(\x | y = -1);\tilde{\W}^{\prime}$ is the KL divergence between two multivariate normal distributions with different means and same covariance. Let $\tilde{\w} = \tilde{\W}_{d} \tilde{\W}_{d-1} \cdots \tilde{\W}_2 \tilde{\w}_1$ and $\tilde{\w}^{\prime} = \tilde{\W}_{d}^{\prime} \tilde{\W}_{d-1}^{\prime} \cdots \tilde{\W}_2^{\prime} \tilde{\w}_1^{\prime}$ for $\tilde{\W}, \tilde{\W}^{\prime} \in \F$, then
	
	\begin{align}
	\begin{split}
	\mathbb{KL}(\mathbb{P}_{(\x | y = -1);\tilde{\W}} || \mathbb{P}_{(\x | y = -1);\tilde{\W}^{\prime}})
	& = \frac{1}{2} \Tr \left[ \left( \sigma^2 \left( \I_p - \M_d(\tilde{\W}^{\prime}) \right)^{-1} \right)^{-1} \sigma^2 \left( \I_p - \M_d(\tilde{\W}) \right)^{-1} \right] \\
	& + \frac{1}{2} \left( y \tilde{\w} - y \tilde{\w}^{\prime} \right)^{\top} \left( \sigma^2 \left( \I_p - \M_d(\tilde{\W}^{\prime}) \right)^{-1} \right)^{-1} \left( y \tilde{\w} - y \tilde{\w}^{\prime} \right) \\
	& - \frac{p}{2} + \frac{1}{2} \ln \left( \frac{\det \left( \sigma^2 \left( \I_p - \M_d(\tilde{\W}^{\prime}) \right)^{-1} \right) }{\det \left( \sigma^2 \left( \I_p - \M_d(\tilde{\W}) \right)^{-1} \right)} \right) \\
	& = \frac{1}{2} \Tr \left[ \I_p \right] + \frac{1}{2 \sigma^2} \left( \tilde{\w} - \tilde{\w}^{\prime} \right)^{\top} \left( \I_p - \M_d(\tilde{\W}^{\prime}) \right) \left( \tilde{\w} - \tilde{\w}^{\prime} \right) - \frac{p}{2} \\
	&= \frac{1}{2 \sigma^2} \left( \tilde{\w} - \tilde{\w}^{\prime} \right)^{\top} \left( \I_p - \M_d(\tilde{\W}^{\prime}) \right) \left( \tilde{\w} - \tilde{\w}^{\prime} \right)
	\end{split}
	\end{align}
	using the fact that $\M_d(\tilde{\W}) = \M_d(\tilde{\W}^{\prime})$.
	
	Moreover, the eigenvalues of $\M_d(\tilde{\W})$ and $\left( \I_p - \M_d(\tilde{\W}) \right)^{-1}$ are
	\begin{align} \label{eq:88}
	\begin{split}
	&\lambda_i^{\downarrow} \left( \M_d(\tilde{\W}) \right) 
	= \begin{cases}
	\frac{d}{d+1},& i \in \{1, 2, \cdots, r \} \\
	\frac{\sum_{j=1}^d c^{2j}}{1 + \sum_{j=1}^d c^{2j}}, & i \in \{r+1, r+2, \cdots, p \}
	\end{cases} \\
	\implies& 
	\lambda_i^{\downarrow} \left( \left( \I_p - \M_d(\tilde{\W}) \right)^{-1} \right) 
	= \begin{cases}
	d+1,& i \in \{1, 2, \cdots, r \} \\
	1 + \sum_{j=1}^d c^{2j}, & i \in \{r+1, r+2, \cdots, p \}
	\end{cases}
	\end{split}
	\end{align}		
	The eigenvalues of $\M_d(\tilde{\W})$ are in $(0,1)$, and from eq \eqref{eq:87} we know $\M_d(\tilde{\W})$ is diagonal. Therefore $\left( \I_p - \M_d(\tilde{\W}^{\prime}) \right)$ is a diagonal matrix with eigenvalues in $(0,1)$. Therefore, 
	\begin{align}
	\begin{split}
	\mathbb{KL}(\mathbb{P}_{(\x | y = -1);\tilde{\W}} || \mathbb{P}_{(\x | y = -1);\tilde{\W}^{\prime}})
	& = \frac{1}{2 \sigma^2} \left( \tilde{\w} - \tilde{\w}^{\prime} \right)^{\top} \left( \I_p - \M_d(\tilde{\W}^{\prime}) \right) \left( \tilde{\w} - \tilde{\w}^{\prime} \right) \\
	& \leq \frac{1}{2 \sigma^2} \| \tilde{\w} - \tilde{\w}^{\prime} \|_2^2 \\
	&\leq \frac{ \left( \| \tilde{\w} \|_2 + \| \tilde{\w}^{\prime} \|_2 \right)^2 }{2 \sigma^2} \\
	&\leq \frac{2^2}{2 \sigma^2} = \frac{2}{\sigma^2}
	\end{split}
	\end{align}
	The last inequality comes form the fact that $\tilde{\w}$ and $\tilde{\w}^{\prime}$ are both within the closed $\ell_2$-unit ball, as $\tilde{\w} = \tilde{\W}_{d} \tilde{\W}_{d-1} \cdots \tilde{\W}_2 \tilde{\w}_1$, and $\| \tilde{\w}_1  \|_2 = 1$ by restriction \textbf{R1}, and each matrix $\tilde{\W}_{i}$ has all of its eigenvalues no greater than 1.
	
	Similar reasoning gives the same upper bound for $\mathbb{KL}(\mathbb{P}_{(\x | y = -1);\tilde{\W}} || \mathbb{P}_{(\x | y = -1);\tilde{\W}^{\prime}})$, therefore, 
	\begin{align}
	\begin{split}
	\mathbb{KL}(\mathbb{P}_{(\x, y);\tilde{\W}} || \mathbb{P}_{(\x, y);\tilde{\W}^{\prime}}) \leq \frac{1}{2} \frac{2}{\sigma^2} + \frac{1}{2} \frac{2}{\sigma^2} = \frac{2}{\sigma^2}
	\end{split}
	\end{align}
	
\end{proof}

\subsection{Proof of main theorems}

We have proved most supporting lemmas for the main theorems. Now we state some well-known information-theoretic facts necessary for the main theorems and prove one of them, then we prove Theorem 3.1, a lemma for Theorem 3.3, and Theorem 3.3 in the main text. Theorem 3.2 and 3.4 are proved in the main text. 

\subsubsection{Some information-theoretic facts}

Below we provide Fano's inequality and several well known facts about mutual information and KL divergence. We only prove the last claim, Lemma \ref{lemma:4.12}.

\begin{theorem} [Fano's inequality, (Yu, 1997)] \label{thm:fano} 
	For any hypothesis $\hat{f} \in \mathcal{F}$, consider the data generating process $\bar{f} \to S \to \hat{f}$, where the dataset $S = \{ (\x_i, y_i) \}_{i=1}^n$, $(\x_i, y_i)$ i.i.d., and the true hypothesis $\bar{f}$ is chosen by nature uniformly at random from $\mathcal{F}$, then we have:
	$$
	P(\hat{f} \neq \bar{f}) \geq 1 - \frac{\mathbb{I}(\bar{f}; S) + \log2}{\log | \F |}
	$$
\end{theorem}

\begin{lemma}[(Cover et al, 2006)] \label{lemma:info3}
	Given a hypothesis class $\F$, fix any two hypotheses $f$, $f^{\prime} \in \F$. If $S$ is a collection of $n$ independent and identically distributed samples $x^{(1)}, \cdots, x^{(n)} \in \mathcal{X}$, where $\mathcal{X}$ is the sample space, then we have
	$$
	\mathbb{KL}(\mathbb{P}_{S|f} || \mathbb{P}_{S|f^{\prime}}) = \sum_{i=1}^n \mathbb{KL}(\mathbb{P}_{x^{(i)}|f} || \mathbb{P}_{x^{(i)}|f^{\prime}}) = n \mathbb{KL}(\mathbb{P}_{x|f} || \mathbb{P}_{x|f^{\prime}})
	$$
\end{lemma}

\begin{lemma}[Pairwise KL bound, (Yu, 1997)] \label{lemma:info4}
	Given a hypothesis class $\F$, consider $\bar{f} \in \F$, where $\bar{f}$ a hypothesis chosen by nature uniformly at random, and an i.i.d. sample $S$ of size $n$ is generated from $\bar{f}$, then 
	$$
	\mathbb{I}(\bar{f}; S) \leq \frac{1}{| \F |^2} \sum_{f \in \F} \sum_{f^{\prime} \in \F} \mathbb{KL}(\mathbb{P}_{S|f} || \mathbb{P}_{S|f^{\prime}}).
	$$
\end{lemma}

\begin{lemma} [Conditional mutual information is sum of weighted "conditional" mutual information] \label{lemma:4.12}
	We have that 
	$$\mathbb{I}(\tilde{\W} ; (\x,y)) 
	= \frac{1}{2} (\mathbb{I}(\tilde{\W} ; \x | y = -1) + \mathbb{I}(\tilde{\W} ; \x | y = +1))$$
	where $\mathbb{I}(\tilde{\W} ; \x | y = -1)$ denotes the 'conditional' mutual information when $y$ is held fixed.
\end{lemma}

\begin{proof} 
	Recall $y \sim \text{Uniform}\{-1, +1\}$, 
	\begin{align}
	\begin{split}
	\mathbb{I}(\tilde{\W} ; (\x,y)) 
	=& \sum_{y \in \{-1, +1\}} \int_{\x}  \int_{\tilde{\W}} p(\tilde{\W}, \x, y) \log \left( \frac{p(\tilde{\W}, \x, y)}{p(\tilde{\W}) p(\x, y)} \right) d\tilde{\W} d\x \\
	=& \sum_{y \in \{-1, +1\}} \int_{\x}  \int_{\tilde{\W}} p(\tilde{\W}) p(y) p(\x | y, \tilde{\W}) \log \left( \frac{p(y) p(\x | y, \tilde{\W} )}{p(p(y) p(\x | y)} \right) d\tilde{\W} d\x \\
	=& \frac{1}{2} \int_{\x}  \int_{\tilde{\W}} p(\tilde{\W}) p(\x | y = -1, \tilde{\W}) \log \left( \frac{p(\x | y = -1, \tilde{\W} )}{p(\x | y = -1)} \right) d\tilde{\W} d\x + \\
	& \frac{1}{2} \int_{\x}  \int_{\tilde{\W}} p(\tilde{\W}) p(\x | y = +1, \tilde{\W}) \log \left( \frac{ p(\x | y = +1, \tilde{\W} )}{p(\x | y = +1)} \right) d\tilde{\W} d\x \\
	\end{split}
	\end{align}
	
	and note that 
	\begin{align}
	\begin{split}
	\mathbb{I}(\tilde{\W} ; \x | y = -1) 
	=& \int_{\x} \int_{\tilde{\W}} p(\tilde{\W}, \x | y = -1) \log \left( \frac{p(\tilde{\W}, \x | y = -1)}{p(\tilde{\W}| y = -1) p(\x | y = -1)} \right) d\tilde{\W} d\x \\
	=& \int_{\x} \int_{\tilde{\W}} p(\tilde{\W}, \x | y = -1) \log \left( \frac{p(\tilde{\W}) p(\x | \tilde{\W}, y = -1)}{p(\tilde{\W}) p(\x | y = -1)} \right) d\tilde{\W} d\x \\
	=& \int_{\x}  \int_{\tilde{\W}} p(\tilde{\W}) p(\x | y = -1, \tilde{\W}) \log \left( \frac{p(\x | y = -1, \tilde{\W} )}{p(\x | y = -1)} \right) d\tilde{\W} d\x
	\end{split}
	\end{align}
	and similar result holds for $\mathbb{I}(\tilde{\W} ; \x | y = +1)$, which proves the desired result.
\end{proof}

\subsubsection{Proof of Theorem 3.1 in the main text}

\begin{proof}
	We utilize the improved KL divergence upper bound in Lemma 3.7 in the main text and Lemma \ref{lemma:info3} to show the claimed sample complexity lower bound. 
	For any $\tilde{\W}, \tilde{\W}^{\prime} \in \G_{p,d}$, for an i.i.d sample $S = \{ (\x_1, y_1), \cdots (\x_n, y_n)\}$ generated from our model in eq (19) in the main text, combine Lemma 3.7 in the main text and Lemma \ref{lemma:info3}, we have
	\begin{align}
	\begin{split}
	\mathbb{KL}(\mathbb{P}_{S;\tilde{\W}} || \mathbb{P}_{S;\tilde{\W}^{\prime}}) 
	= n \mathbb{KL}(\mathbb{P}_{(\x, y);\tilde{\W}} || \mathbb{P}_{(\x, y);\tilde{\W}^{\prime}})
	\leq \frac{2n}{\sigma^2}
	\end{split}
	\end{align}
	where $\sigma^2$ is a constant in the model described in eq (19) in the main text.
	
	Then by Lemma \ref{lemma:info4}, for a $d$-layer network as described in eq (19) in the main text with true parameter $\tilde{\W}^* = (\tilde{\W}_{d}^*, \cdots, \tilde{\W}_{1}^*, \tilde{\w}_0^*)$ from hypothesis class $\G_{p,d}$, we have
	\begin{align}
	\begin{split}
	\mathbb{I}(\tilde{\W}^*; S)
	\leq \frac{1}{| \F |^2} \sum_{\tilde{\W} \in \G_{p,d}} \sum_{\tilde{\W}^{\prime} \in \G_{p,d}} \mathbb{KL}(\mathbb{P}_{S;\tilde{\W}} || \mathbb{P}_{S;\tilde{\W}^{\prime}}) 
	= \frac{1}{| \G_{p,d} |^2} \sum_{\tilde{\W} \in \G_{p,d}} \sum_{\tilde{\W}^{\prime} \in \G_{p,d}} n \mathbb{KL}(\mathbb{P}_{(\x, y);\tilde{\W}} || \mathbb{P}_{(\x, y);\tilde{\W}^{\prime}}) 
	\leq \frac{2n}{\sigma^2}
	\end{split}
	\end{align}
	
	By Fano's inequality (Theorem \ref{thm:fano}), for any hypothesis $\tilde{\W} \in \G_{p,d}$, where $\tilde{\W}$ is the output of any decoder $\hat{f} \in \Psi(\F_{p,d})$ given dataset $S$, that is, $\tilde{\W} = \hat{f}(S)$, 
	\begin{align}
	\begin{split}
	&\xi_1(\hat{f}, \mathbb{P}) := P_{(\tilde{\W}^*, S) \sim \mathbb{P}}(\hat{f}(S) \neq \tilde{\W}^* ) 
	= P_{(\tilde{\W}^*, S) \sim \mathbb{P}}(\tilde{\W} \neq \tilde{\W}^* ) \\
	\geq& 1 - \frac{\mathbb{I}(\tilde{\W}^*; S) + \log(2)}{\log | \G_{p,d} |}
	\geq 1 - \frac{ (2n / \sigma^2) + \log (2)}{\log | \G_{p,d} |}
	\end{split}
	\end{align}
	where $\mathbb{P} \in \mathcal{P}' := \mathcal{P}_{\tilde{\W}, S}( \{\text{Uniform}(\mathcal{G}_{p, d})\}, \mathcal{P}_{(\x, y) | \tilde{\W}}^{\text{Id}, \mathcal{Q}_{\x}} )$ as described in Theorem 3.1 in the main text.
	
	Thus, the sample complexity lower bound is 
	\begin{align} \label{eq:92}
	\begin{split}
	\xi_1(\hat{f}, \mathbb{P}) \geq \frac{1}{2} 
	\impliedby 
	n &\leq \frac{\sigma^2}{2} \left( \frac{\log | \G_{p,d} |}{2} - \log(2) \right) 
	= \frac{\sigma^2 \left[ d \left( \sum_{i=1}^{r} \log(i) \right) + p \log(2) - \log(4) \right] }{4} 
	\end{split}
	\end{align}
	meaning that if the number of samples is of order $\Omega \left( d r \log(r) + p \right)$, then the probability of identifying the truth $\tilde{\W}^*$ is less than half, regardless of the specific decoder $\hat{f}$.
\end{proof}

\subsubsection{A supporting lemma for Theorem 3.3 in the main text}
In this section we state and prove a lemma regarding the excess risk $\tilde{R}(\tilde{\W}) := R(\tilde{\W} , \tilde{\W}^*) - R(\tilde{\W}^* , \tilde{\W}^*) = R(\tilde{\W}) - R(\tilde{\W}^*)$, where $R(\tilde{\W})$ is a shorthand for $R(\tilde{\W} , \tilde{\W}^*)$, and the prediction risk is defined as 
\begin{align}
R(\tilde{\W} , \tilde{\W}^*) := P_{(\x, y) ; \tilde{\W}^*} \left[ \x^{\top} \boldsymbol{\mu}_{\x | y; \tilde{\W}} \leq 0 \right].
\end{align}
where $\boldsymbol{\mu}_{\x | y; \tilde{\W}} = \mathbb{E}_{\x | y; \tilde{\W}} \left[ \x | y \right]$ is the mean of $\x$ conditioned on $y$ and given parameter $\tilde{\W}$ and used as the predictor. 

\begin{lemma}[Excess prediction risk for hypotheses in $\G_{p,d}$] \label{lemma:4.13}
	For a $d$-layer network as described in eq (19) in the main text, paramatrized by truth $\tilde{\W}^* \in \G_{p,d}$, under restrictions \textbf{R1} and \textbf{R2}, for any $\tilde{\W} \in \G_{p,d}$ output by any decoder $\hat{f} \in \Psi(\G_{p,d})$, if $\tilde{\W} \neq \tilde{\W}^*$, then 
	\begin{align} \label{eq:23}
	\begin{split}
	R(\tilde{\W}) - R(\tilde{\W}^*) \geq 
	\mathds{1} \{ \tilde{\w} \neq \tilde{\w}^* \} 
	\cdot 
	\frac{\text{erf} \left( c_1 \right) - \text{erf} \left( c_0 \right) }{2} 
	\end{split}
	\end{align}
	where $\tilde{\W} = \left( \tilde{\W}_d, \cdots, \tilde{\W}_1, \tilde{\w}_0 \right)$, $\tilde{\w} := \tilde{\W}_d \cdots \tilde{\W}_1 \tilde{\w}_0 $, and constants $c_0, c_1$ are
	\begin{align} \label{eq:109}
	c_0 &:= \frac{ 1 - \frac{1}{2^r} + c^{2d} \left( \frac{1}{2^r} - \frac{1}{2^{p-2}} \right) }{ \sigma \sqrt{2 \left[ (d+1) \left( 1 - \frac{1}{2^r} \right) + \left( \frac{1-c^{2(d+1)}}{1 - c^2} \right) \left( \frac{ c^{2d}}{2^r} \right) \right]} }, 
	\end{align}
	\begin{align} \label{eq:110} 
	c_1 &:= \frac{ 1 - \frac{1}{2^r} + \frac{c^{2d}}{2^r} }{ \sigma \sqrt{2 \left[ (d+1) \left( 1 - \frac{1}{2^r} \right) + \left( \frac{1-c^{2(d+1)}}{1 - c^2} \right) \left( \frac{ c^{2d}}{2^r} \right) \right]} },
	\end{align}
	with $c = \frac{1}{p-r+1}$ as described in Restriction \textbf{R2}.
\end{lemma}

\begin{proof}
	In this proof we analyze five different cases and show the claimed excess risk lower bound holds. Consider $\tilde{\W}, \tilde{\W}^* \in \G_{p,d}$ with $\tilde{\W} \neq \tilde{\W}^*$. Recall $\tilde{\W} = \left( \tilde{\W}_d, \cdots, \tilde{\W}_1, \tilde{\w}_0 \right)$ with $\tilde{\W}_i$ being $p \times p$ matrices and $\tilde{\w}_0$ being a vector in $\R^p$.
	
	First we analyze how $\tilde{\w} := \tilde{\W}_d \cdots \tilde{\W}_1 \tilde{\w}_0$ affect the risk $R(\tilde{\W})$. Let $u := y \cdot \tilde{\w}^{\top} \x$, then $u$ is normally distributed with mean $\tilde{\w}^{\top} \tilde{\w}^*$ and variance $\sigma^2 \tilde{\w}^{\top} \left( \I_p - \M_d(\tilde{\W}^*) \right)^{-1} \tilde{\w}$, because $(\x, y) ; \tilde{\W}^*$ is parametrized by $\tilde{\W}^*$, Thus the risk is 
	\begin{align}
	\begin{split}
	R(\tilde{\W})  
	&= P_{u } \left[ u \leq 0 \right] \\
	&= P_{z \sim N(0, 1) } \left[ 
	z \leq \frac{- \tilde{\w}^{\top} \tilde{\w}^* } {\sqrt{\sigma^2 \tilde{\w}^{\top} \left( \I_p - \M_d(\tilde{\W}^*) \right)^{-1} \tilde{\w}}} \right] \\
	&= \frac{1}{2} \left[ 1 + \text{erf} \left( \frac{- \tilde{\w}^{\top} \tilde{\w}^* }{\sqrt{2 \sigma^2 \tilde{\w}^{\top} \left( \I_p - \M_d(\tilde{\W}^*) \right)^{-1} \tilde{\w} }} \right) \right] \\
	&= \frac{1}{2} \left[ 1 - \text{erf} \left( \frac{ \tilde{\w}^{\top} \tilde{\w}^* }{\sqrt{2 \sigma^2 \tilde{\w}^{\top} \left( \I_p - \M_d(\tilde{\W}^*) \right)^{-1} \tilde{\w} }} \right) \right]
	\end{split}
	\end{align}
	
	As shown in eq \eqref{eq:87}, we know that $\left( \I_p - \M_d(\tilde{\W}^*) \right)^{-1} = \begin{bmatrix}
	(d+1) \I_r & \mathbf{0} \\
	\mathbf{0} & \left( 1 + \sum_{j=1}^d c^{2j} \right) \I_{p-r}
	\end{bmatrix}$
	where $c = \frac{1}{p-r+1}$. Therefore we can calculate the denominator inside the erf function of $R(\tilde{\W})$ exactly:
	\begin{align}
	\begin{split}
	&\tilde{\w}^{\top} \left( \I_p - \M_d(\tilde{\W}^*) \right)^{-1} \tilde{\w} \\
	=& (d+1) \sum_{k=1}^{r} (\tilde{\w})_k^2 + \left( 1 + \sum_{j=1}^d c^{2j} \right) \sum_{k=r+1}^{p} (\tilde{\w})_k^2 \\
	=& (d+1) \sum_{k=1}^{r} (\tilde{\w}_0)_k^2 + \left( 1 + \sum_{j=1}^d c^{2j} \right) \sum_{k=r+1}^{p} \left( c^d (\tilde{\w}_0)_k \right)^2 \\
	=& (d+1) \left( \frac{1}{2} + \frac{1}{4} + \cdots \frac{1}{2^r} \right) +  \left( 1 + \sum_{j=1}^d c^{2j} \right) \left[ c^{2d} \left( \frac{1}{2^{r+1}} + \cdots + \frac{1}{2^{p-1}} + \frac{1}{2^{p-1}} \right) \right] \\
	=& (d+1) \left( 1 - \frac{1}{2^r} \right) + \left( \frac{1-c^{2(d+1)}}{1 - c^2} \right) \left( \frac{ c^{2d}}{2^r} \right)
	\end{split}
	\end{align}
	due to our choice of $\tilde{\W}_i$ and $\tilde{\w}_0$ - the first $r$ entries of $\tilde{\w}_0$ only get permuted and do not get scaled, while the last $(p-r)$ entries of $\tilde{\w}_0$ do not get permuted but get scaled by $d$ times.
	
	Thus we have the risk of the truth $\tilde{\W}^*$:
	\begin{align}
	\begin{split}
	R(\tilde{\W}^*) 
	&= \frac{1}{2} \left[ 1 - \text{erf} \left( \frac{ \| \tilde{\w}^* \|_2^2 }{ \sigma \sqrt{2 \left[ (d+1) \left( 1 - \frac{1}{2^r} \right) + \left( \frac{1-c^{2(d+1)}}{1 - c^2} \right) \left( \frac{ c^{2d}}{2^r} \right) \right]} } \right) \right] \\
	&= \frac{1}{2} \left[ 1 - \text{erf} \left( \frac{ 1 - \frac{1}{2^r} + \frac{c^{2d}}{2^r} }{ \sigma \sqrt{2 \left[ (d+1) \left( 1 - \frac{1}{2^r} \right) + \left( \frac{1-c^{2(d+1)}}{1 - c^2} \right) \left( \frac{ c^{2d}}{2^r} \right) \right]} } \right) \right] \\
	&= \frac{1}{2} \left[ 1 - \text{erf}(c_1) \right]
	\end{split}
	\end{align}
	
	Now we partition the hypothesis class $\G_{p,d}$ into five cases, where the truth $\tilde{\W}^*$ is viewed as fixed, and we analyze the lower bound of the excess risk $R(\tilde{\W}) - R(\tilde{\W}^*)$ for each case, and $\tilde{\W}, \tilde{\W}^* \in \G_{p,d}$.
	
	\vline
	
	\noindent \textbf{Case 1:} $\tilde{\w}_0 \neq \tilde{\w}_0^*, \tilde{\W}_1 = \tilde{\W}_1^*, \cdots, \tilde{\W}_d = \tilde{\W}_d^*$. 
	Note this implies $\tilde{\W}_d \cdots \tilde{\W}_1 = \tilde{\W}_d^* \cdots \tilde{\W}_1^*$. 
	
	With loss of generatlity, as the matrix to be multiplied with the vector $\tilde{\w}_0$ is the same as the matrix to be multiplied with the vector $\tilde{\w}_0^*$, and because both matrices consist of a diagonal block of permutation matrix and another diagonal block of scaled identity matrix, we can assume
	\begin{align}
	\begin{split}
	\tilde{\W}_d = \cdots = \tilde{\W}_1 = \begin{bmatrix}
	\mathbf{I}_r & \mathbf{0} \\
	\mathbf{0} & c \mathbf{I}_{p-r}
	\end{bmatrix},
	\end{split}
	\end{align}
	that is, we assume the top $r \times r$ block of each $\tilde{\W}_i$ is $\mathbf{I}_r$. Thus 
	\begin{align}
	\tilde{\W} = \tilde{\W}_d, \cdots, \tilde{\W}_1, \tilde{\w}_0 = \begin{bmatrix}
	\mathbf{I}_r & \mathbf{0} \\
	\mathbf{0} & c^d \mathbf{I}_{p-r}
	\end{bmatrix}
	\tilde{\w}_0, 
	\tilde{\W}^* = \tilde{\W}_d^*, \cdots, \tilde{\W}_1^*, \tilde{\w}_0^* = \begin{bmatrix}
	\mathbf{I}_r & \mathbf{0} \\
	\mathbf{0} & c^d \mathbf{I}_{p-r}
	\end{bmatrix}
	\tilde{\w}_0^*.
	\end{align}
	where $c = \frac{1}{p-r+1}$.
	
	It is easy to see that, for $\tilde{\w}_0, \tilde{\w}_0^* \in \left\{ \pm \frac{1}{\sqrt{2}} \right\} 
	\times \left\{ \pm \frac{1}{\sqrt{4}} \right\} 
	\times \cdots
	\times \left\{ \pm \frac{1}{\sqrt{2^{p-2}}} \right\} 
	\times \left\{ \pm \frac{1}{\sqrt{2^{p-1}}} \right\} 
	\times \left\{ \pm \frac{1}{\sqrt{2^{p-1}}} \right\}$ and $\tilde{\w}_0 \neq \tilde{\w}_0^*$, the largest $\tilde{\w}^{\top} \tilde{\w}^*$ is attained when $\tilde{\w}_0$ and $\tilde{\w}_0^*$ differ exactly at either the last entry or the second last entry, that is, the entry with smallest magnitude, $\pm \frac{1}{\sqrt{2^{p-1}}}$. In this case, 
	\begin{align} 
	\begin{split}
	\tilde{\w}^{\top} \tilde{\w}^* &\leq \frac{1}{2} + \cdots + \frac{1}{2^r} + c^{2d} \left( \frac{1}{2^{r+1}} + \cdots + \frac{1}{2^{p-1}} \right) - \frac{c^{2d}}{2^{p-1}} 
	= \left( 1 - \frac{1}{2^r} \right) + c^{2d} \left( \frac{1}{2^r} - \frac{1}{2^{p-2}} \right)
	\end{split}
	\end{align}
	
	Therefore we obtain a lower bound of $R(\tilde{\W}) - R(\tilde{\W}^*)$:
	\begin{align} \label{eq:114}
	\begin{split}
	R(\tilde{\W})  - R(\tilde{\W}^*) &= \frac{1}{2} \left[ 1 - \text{erf} \left( \frac{ \tilde{\w}^{\top} \tilde{\w}^* }{\sqrt{2 \sigma^2 \tilde{\w}^{\top} \left( \I_p - \M_d(\tilde{\W}^*) \right)^{-1} \tilde{\w} }} \right) \right] 
	- \frac{1}{2} \left[ 1 - \text{erf}(c_1) \right] \\
	&\geq \frac{1}{2} \left[ 1 - \text{erf} \left( \frac{ 1 - \frac{1}{2^r} + c^{2d} \left( \frac{1}{2^r} - \frac{1}{2^{p-2}} \right) }{ \sigma \sqrt{2 \left[ (d+1) \left( 1 - \frac{1}{2^r} \right) + \left( \frac{1-c^{2(d+1)}}{1 - c^2} \right) \left( \frac{ c^{2d}}{2^r} \right) \right]} } \right) \right] 
	- \frac{1}{2} \left[ 1 - \text{erf}(c_1) \right] \\ 
	&= \frac{1}{2} \left[ 1 - \text{erf}(c_0) \right] - \frac{1}{2} \left[ 1 - \text{erf}(c_1) \right] 
	= \frac{\text{erf}(c_1) - \text{erf}(c_0)}{2}
	\end{split}
	\end{align}
	It is easy to see that this lower bound is positive, as the erf function is increasing and the argument inside the first erf function, $c_1$, is obviously greater than the second, $c_0$.
	
	\vline
	
	\noindent \textbf{Case 2:} $\tilde{\w}_0 \neq \tilde{\w}_0^*, ( \tilde{\W}_d, \cdots, \tilde{\W}_1 ) \neq ( \tilde{\W}_d^*, \cdots, \tilde{\W}_1^* )$ but $\tilde{\W}_d \cdots \tilde{\W}_1 = \tilde{\W}_d^* \cdots \tilde{\W}_1^*$. 
	Note that the latter means the network parameterized by $\tilde{\W}$ and the network parametrized by $\tilde{\W}^*$ are different in at least one layer. The analysis of this case is identical to \textbf{Case 1}. 
	
	\vline
	
	\noindent \textbf{Case 3:} $\tilde{\w}_0 \neq \tilde{\w}_0^*, ( \tilde{\W}_d, \cdots, \tilde{\W}_1 ) \neq ( \tilde{\W}_d^*, \cdots, \tilde{\W}_1^* ), \tilde{\W}_d \cdots \tilde{\W}_1 \neq \tilde{\W}_d^* \cdots \tilde{\W}_1^*$. 
	
	For this case, there are three sub-cases:
	
	\textbf{Case 3(i):} $\tilde{\w}_0$ and $\tilde{\w}_0^*$ differ only in the first $r$ entries.
	
	\textbf{Case 3(ii):} $\tilde{\w}_0$ and $\tilde{\w}_0^*$ differ only in the last $p-r$ entries. 
	
	\textbf{Case 3(iii):} $\tilde{\w}_0$ and $\tilde{\w}_0^*$ differ in both the first $r$ and last $p-r$ entries. 
	
	Observe that, by our choice of hypothesis class described in Restriction \textbf{R1}, the quantity that controls the risk $R(\tilde{\W})$ has
	\begin{align} \label{eq:116}
	\begin{split}
	\tilde{\w}^{\top} \tilde{\w}^* 
	&= \tilde{\w}_{[1:r]}^{\top} \tilde{\w}_{[1:r]}^* + \tilde{\w}_{[r+1:p]}^{\top} \tilde{\w}_{[r+1:p]}^* \\
	&= \boldsymbol{\pi}_r((\tilde{\w}_0)_{[1:r]})^{\top} \boldsymbol{\pi}_r^*((\tilde{\w}_0^*)_{[1:r]}) + (c^d (\tilde{\w}_0)_{[r+1:p]})^{\top} (c^d (\tilde{\w}_0^*)_{[r+1:p]}) \\
	&= \boldsymbol{\pi}_r \left( (\boldsymbol{\pi}_r^*)^{-1} ((\tilde{\w}_0)_{[1:r]}) \right)^{\top} (\tilde{\w}_0^*)_{[1:r]} + c^{2d} (\tilde{\w}_0^*)_{[r+1:p]}^{\top}  (\tilde{\w}_0)_{[r+1:p]} \\
	&= \boldsymbol{\pi} \left( (\tilde{\w}_0)_{[1:r]} \right)^{\top} (\tilde{\w}_0^*)_{[1:r]} + c^{2d} (\tilde{\w}_0^*)_{[r+1:p]}^{\top}  (\tilde{\w}_0)_{[r+1:p]}
	\end{split} 
	\end{align}
	where $\boldsymbol{\pi}_r$ and $\boldsymbol{\pi}_r^*$ in the second line respectively correspond to the permutation matrix in the top $r \times r$ block in $\tilde{\W}_d \cdots \tilde{\W}_1$ and $\tilde{\W}_d^* \cdots \tilde{\W}_1^*$, and on the third line we are essentially showing that we can assume the permutation on $(\tilde{\w}_0^*)_{[1:r]}$ is identity, and thus we only need to consider $(\tilde{\W}_d \cdots \tilde{\W}_1)_{[1:r]} \neq (\tilde{\W}_d^* \cdots \tilde{\W}_1^*)_{[1:r]} = \mathbf{I_r}$. On the fourth line, we simplify our notation by $\boldsymbol{\pi}$, where $\boldsymbol{\pi} \neq \mathbf{I_r}$. We are interested in an upper bound of $\tilde{\w}^{\top} \tilde{\w}^*$. Now we analyze the three sub-cases.
	
	\textbf{Case 3(ii)} is straightforward. As the last $(p-r)$ entries do not get permuted as shown in \eqref{eq:116}, thus the smallest negative contribution they can have is due to opposite signs at exactly one of the last two entries, where the values are in $\left\{ \pm \frac{1}{\sqrt{2^{p-1}}} \right\}$. This gives 
	\begin{align}
	\begin{split}
	\max\{ \tilde{\w}^{\top} \tilde{\w}^* : \tilde{\w} \in {\textbf{Case 3(ii)}} \}
	&= \boldsymbol{\sigma} \left( (\tilde{\w}_0)_{[1:r]} \right)^{\top} (\tilde{\w}_0^*)_{[1:r]} + c^{2d} \left( \frac{1}{2^{r+1}} + \cdots + \frac{1}{2^{p-2}} + \frac{1}{2^{p-1}} - \frac{1}{2^{p-1}} \right) \\
	&= \boldsymbol{\sigma} \left( (\tilde{\w}_0^*)_{[1:r]} \right)^{\top} (\tilde{\w}_0^*)_{[1:r]} + c^{2d} \left( \frac{1}{2^{r}} - \frac{1}{2^{p-2}} \right) \\
	&< (\tilde{\w}_0^*)_{[1:r]}^{\top} (\tilde{\w}_0^*)_{[1:r]} + c^{2d} \left( \frac{1}{2^{r}} - \frac{1}{2^{p-2}} \right) \\
	&= 1 - \frac{1}{2^r} + c^{2d} \left( \frac{1}{2^{r}} - \frac{1}{2^{p-2}} \right)
	\end{split}
	\end{align}
	where the second equality follows from \textbf{Case 3(ii)} assumes the first $r$ entries of $\tilde{\w}_0$ and $\tilde{\w}_0^*$ are identical. The third line follows from $\boldsymbol{\pi} \neq \mathbf{I}_r$ and the rearrangement inequality, because $\max \{ \boldsymbol{\pi} \left( (\tilde{\w}_0^*)_{[1:r]} \right)^{\top} (\tilde{\w}_0^*)_{[1:r]}: \text{any } r\text{-permutation } \boldsymbol{\pi} \}$ is attained by $\mathbf{I}_r$. 
	
	As a result, \textbf{Case 3(iii)} has the same upper bound on $\tilde{\w}^{\top} \tilde{\w}^*$:
	\begin{align}
	\begin{split}
	&\max \{ \boldsymbol{\pi} \left( (\tilde{\w}_0)_{[1:r]} \right)^{\top} (\tilde{\w}_0^*)_{[1:r]}: \text{any } r \text{-permutation } \boldsymbol{\pi} \neq \mathbf{I}_r, (\tilde{\w}_0)_{[1:r]} \neq (\tilde{\w}_0^*)_{[1:r]} \} \\
	<& \max \{ \boldsymbol{\pi} \left( (\tilde{\w}_0^*)_{[1:r]} \right)^{\top} (\tilde{\w}_0^*)_{[1:r]}: \text{any } r\text{-permutation } \boldsymbol{\pi} \} = 1 - \frac{1}{2^r} \\
	\implies& \max\{ \tilde{\w}^{\top} \tilde{\w}^* : \tilde{\w} \in {\textbf{Case 3(iii)}} \} \\
	&= \max\{ \boldsymbol{\pi} (\tilde{\w}_0^{\top}) \tilde{\w}_0^* : (\tilde{\w}_0)_{[1:r]} \neq (\tilde{\w}_0^*)_{[1:r]}, (\tilde{\w}_0)_{[r+1:p]} \neq (\tilde{\w}_0^*)_{[r+1:p]},  \text{any } r \text{-permutation } \boldsymbol{\pi} \neq \mathbf{I}_r \} \\
	&< 1 - \frac{1}{2^r} + c^{2d} \left( \frac{1}{2^{r}} - \frac{1}{2^{p-2}} \right)
	\end{split}
	\end{align}
	
	\textbf{Case 3(i)} has the same upper bound as well,
	\begin{align} \label{eq:119}
	\begin{split}
	\max\{ \tilde{\w}^{\top} \tilde{\w}^* : \tilde{\w} \in {\textbf{Case 3(i)}} \} \leq 1 - \frac{1}{2^{r-2}} + \frac{c^{2d}}{2^{r}} < 1 - \frac{1}{2^r} + c^{2d} \left( \frac{1}{2^{r}} - \frac{1}{2^{p-2}} \right)
	\end{split}
	\end{align}
	
	The first inequality in \eqref{eq:119} is obtained under either of the two scenarios by a particular $(\tilde{\w}_0)_{[1:r]}$ and $\boldsymbol{\pi}$.
	
	\textbf{Scenario 1:} For $i = 1, \cdots, r-1, (\tilde{\w}_0)_i = (\tilde{\w}_0^*)_i$ and $(\tilde{\w}_0)_r = -(\tilde{\w}_0^*)_r$, and the $r$-permutation $\boldsymbol{\pi} \neq \mathbf{I}_r$ only switches the last two entries.
	
	\textbf{Scenario 2:} For $i \neq r-1,  (\tilde{\w}_0)_i = (\tilde{\w}_0^*)_i$ and $(\tilde{\w}_0)_{r-1} = -(\tilde{\w}_0^*)_{r-1}$, and the $r$-permutation $\boldsymbol{\pi} \neq \mathbf{I}_r$ only switches the last two entries.
	
	It is easy to see that both scenarios lead to the same $\tilde{\w}^{\top} \tilde{\w}^*$:
	\begin{align}
	\begin{split}
	\tilde{\w}^{\top} \tilde{\w}^* 
	&= \boldsymbol{\pi}((\tilde{\w}_0)_{[1:r]})^{\top} (\tilde{\w}_0^*)_{[1:r]} + (\tilde{\w}_0)_{[r+1:p]}^{\top} (\tilde{\w}_0^*)_{[r+1:p]} \\ 
	&= (\tilde{\w}_0^*)_{[1:r-2]}^{\top} (\tilde{\w}_0^*)_{[1:r-2]} + \frac{1}{\sqrt{2^{r-1}}} \frac{1}{\sqrt{2^r}} - \frac{1}{\sqrt{2^{r-1}}} \frac{1}{\sqrt{2^r}} + (\tilde{\w}_0^*)_{[r+1:p]}^{\top} (\tilde{\w}_0^*)_{[r+1:p]} \\
	&= \left( \frac{1}{2} + \cdots + \frac{1}{2^{r-2}} \right) + c^{2d} \left( \frac{1}{2^{r+1}} + \cdots + \frac{1}{2^{p-2}} + \frac{1}{2^{p-1}} + \frac{1}{2^{p-1}} \right) 
	= 1 - \frac{1}{2^{r-2}} + \frac{c^{2d}}{2^{r}}
	\end{split}
	\end{align}
	where, since $c = \frac{1}{p-r+1}$, the second inequality in \eqref{eq:119} holds.
	
	\vline
	
	\noindent \textbf{Case 4:} $\tilde{\w}_0 = \tilde{\w}_0^*, ( \tilde{\W}_d, \cdots, \tilde{\W}_1 ) \neq ( \tilde{\W}_d^*, \cdots, \tilde{\W}_1^* )$ but $\tilde{\W}_d \cdots \tilde{\W}_1 = \tilde{\W}_d^* \cdots \tilde{\W}_1^*$.
	
	In this case we will have $R(\tilde{\W}) = R(\tilde{\W}^*)$ as $\tilde{\w} = \tilde{\w}^*$. Thus we analyze the size of this case, that is, calculate how many $\tilde{\W} = ( \tilde{\W}_d, \cdots, \tilde{\W}_1, \tilde{\w}_0)$ fall into this case.
	
	Consider the truth $\tilde{\W}^* = ( \tilde{\W}_d^*, \cdots, \tilde{\W}_1^*, \tilde{\w}_0^*)$ fixed, then $\tilde{\W}_d^* \cdots \tilde{\W}_1^*$ is a fixed diagonal matrix, with its top $r \times r$ block determined by the top $r \times r$ block of each $\tilde{\W}_i^*$ ($i = 1, \cdots, d$), and its top $r \times r$ block, $(\tilde{\W}_d^* \cdots \tilde{\W}_1^*)_{[1:r,1:r]}$, represents an $r$-permutation. Let this $r$-permutation be $\boldsymbol{\tau}$. 
	
	Therefore, in order to have $\tilde{\W}_d \cdots \tilde{\W}_1 = \tilde{\W}_d^* \cdots \tilde{\W}_1^*$, we effectively need $(\tilde{\W}_d \cdots \tilde{\W}_1)_{[1:r,1:r]} = (\tilde{\W}_d)_{[1:r,1:r]} \cdots (\tilde{\W}_1)_{[1:r,1:r]} =  \boldsymbol{\tau}$. Consider the first $d-1$ matrices in the product, $(\tilde{\W}_d)_{[1:r,1:r]} \cdots (\tilde{\W}_2)_{[1:r,1:r]}$, a product of $d-1$ $r$-permutations. Consider any $\left( (\tilde{\W}_d)_{[1:r,1:r]}, \cdots, (\tilde{\W}_2)_{[1:r,1:r]} \right)$, there are $(r!)^{d-1}$ of those $(d-1)$-tuples of $r$-permutations. In order to have the product of $\left( (\tilde{\W}_d)_{[1:r,1:r]}, \cdots, (\tilde{\W}_2)_{[1:r,1:r]} \right)$ and $(\tilde{\W}_1)_{[1:r,1:r]}$ equal to $\boldsymbol{\tau}$, there is only one choice for $(\tilde{\W}_1)_{[1:r,1:r]}$. Therefore the set $\{ (\tilde{\W}_d, \cdots, \tilde{\W}_1): \tilde{\W}_d \cdots \tilde{\W}_1 = \tilde{\W}_d^* \cdots \tilde{\W}_1^* \}$ has cardinality $(r!)^{d-1}$. 
	
	Then, because we have $\tilde{\w}_0 = \tilde{\w}_0^*$, we have $\{ (\tilde{\W}_d, \cdots, \tilde{\W}_1, \tilde{\w}_0): \tilde{\W}_d \cdots \tilde{\W}_1 = \tilde{\W}_d^* \cdots \tilde{\W}_1^* \}$ of cardinality $(r!)^{d-1}$ as well. Furthermore, \textbf{Case 4} is a proper subset of it:
	\begin{align} \label{eq:105}
	\begin{split}
	&\{ (\tilde{\W}_d, \cdots, \tilde{\W}_1, \tilde{\w}_0): ( \tilde{\W}_d, \cdots, \tilde{\W}_1 ) \neq ( \tilde{\W}_d^*, \cdots, \tilde{\W}_1^* ), \tilde{\W}_d \cdots \tilde{\W}_1 = \tilde{\W}_d^* \cdots \tilde{\W}_1^*\} \\
	&\subsetneq \{ (\tilde{\W}_d, \cdots, \tilde{\W}_1, \tilde{\w}_0): \tilde{\W}_d \cdots \tilde{\W}_1 = \tilde{\W}_d^* \cdots \tilde{\W}_1^*\} \\
	\implies& \\
	&| \{ (\tilde{\W}_d, \cdots, \tilde{\W}_1, \tilde{\w}_0): ( \tilde{\W}_d, \cdots, \tilde{\W}_1 ) \neq ( \tilde{\W}_d^*, \cdots, \tilde{\W}_1^* ), \tilde{\W}_d \cdots \tilde{\W}_1 = \tilde{\W}_d^* \cdots \tilde{\W}_1^*\} | \\
	&< |\{ (\tilde{\W}_d, \cdots, \tilde{\W}_1, \tilde{\w}_0): \tilde{\W}_d \cdots \tilde{\W}_1 = \tilde{\W}_d^* \cdots \tilde{\W}_1^*\}| = (r!)^{d-1} \\
	\end{split}
	\end{align}
	
	Now we show \textbf{Case 4} is very small compared to the entire hypothesis class.
	\begin{align}
	\begin{split}
	&\log | \G_{p,d} \setminus \{ (\tilde{\W}_d, \cdots, \tilde{\W}_1, \tilde{\w}_0): \textbf{Case 4} \} | \\
	>& \log \left( |\G_{p,d}| - (r!)^{d-1} \right) 
	= \log \left( 2^p \cdot (r!)^d - (r!)^{d-1} \right) \\
	=& \log \left( \left( 2^p \cdot r! - 1 \right) \cdot (r!)^{d-1} \right) 
	= \log \left( 2^p \cdot r! - 1 \right) + (d-1) \sum_{i=1}^r \log(i) \\
	\approx& \log \left( 2^p \cdot r! \right) + (d-1) \sum_{i=1}^r \log(i) \\
	=& p \log2 + d \sum_{i=1}^r \log(i)
	\in \Theta(p + dr)
	\end{split}
	\end{align}
	which is of the same order as $\log |\G_{p,d}|$ in Restriction \textbf{R1}.
	
	\vline
	
	\noindent \textbf{Case 5:} $\tilde{\w}_0 = \tilde{\w}_0^*, ( \tilde{\W}_d, \cdots, \tilde{\W}_1 ) \neq ( \tilde{\W}_d^*, \cdots, \tilde{\W}_1^* ), \tilde{\W}_d \cdots \tilde{\W}_1 \neq \tilde{\W}_d^* \cdots \tilde{\W}_1^*$.
	
	Without loss of generality, we assume
	\begin{align}
	\begin{split}
	\tilde{\W}_d^* = \cdots = \tilde{\W}_1^* = \begin{bmatrix}
	\mathbf{I}_r & \mathbf{0} \\
	\mathbf{0} & c \mathbf{I}_{p-r}
	\end{bmatrix}, 
	\tilde{\w}_0^* = \tilde{\w}_0 = \left( \frac{1}{\sqrt{2}}, \frac{1}{\sqrt{4}}, \cdots, \frac{1}{\sqrt{2^{p-2}}}, \frac{1}{\sqrt{2^{p-1}}}, \frac{1}{\sqrt{2^{p-1}}} \right)
	\end{split}
	\end{align}. 
	
	As $\tilde{\W}_d \cdots \tilde{\W}_1 \neq \tilde{\W}_d^* \cdots \tilde{\W}_1^*$, $\tilde{\w}_{[1:r]} = \boldsymbol{\pi}(\tilde{\w}_0)_{[1:r]}$ wil not be the same as $\tilde{\w}^*_{[1:r]}$, that is, their entries will consist of the same values, but their values will not be in the same order and thus will not achieve the maximum $1$ by the rearrangement inequality. 
	
	Futhermore, the second largest value $\tilde{\w}_{[1:r]}^{\top} \tilde{\w}^*_{[1:r]} =  \boldsymbol{\pi}(\tilde{\w}_0)_{[1:r]}^{\top}  \tilde{\w}^*_{[1:r]}$ could obtain is by an $r$-permutation $\boldsymbol{\pi}$ that switches the two smallest entries, that is, $\boldsymbol{\pi}$ that exchanges $\frac{1}{\sqrt{2^{r-1}}}$ and $\frac{1}{\sqrt{2^{r}}}$ in $\tilde{\w}_0$. This $\boldsymbol{\pi}$ would give us $\tilde{\w}_{[1:r]}$ such that
	\begin{align}
	\begin{split}
	\tilde{\w}^{\top} \tilde{\w}^* 
	=& \left( \frac{1}{\sqrt{2}} \right)^2 + \left( \frac{1}{\sqrt{2^2}} \right)^2 + \cdots + \left( \frac{1}{\sqrt{2^{r-2}}} \right)^2 + \left( \frac{1}{\sqrt{2^{r}}} \right) \left( \frac{1}{\sqrt{2^{r-1}}} \right) + \left( \frac{1}{\sqrt{2^{r-1}}} \right) \left( \frac{1}{\sqrt{2^{r}}} \right) \\
	&+ \left( \frac{c^{d}}{\sqrt{2^{r+1}}} \right)^2 + \cdots + \left( \frac{c^{d}}{\sqrt{2^{p-2}}} \right)^2 + \left( \frac{c^{d}}{\sqrt{2^{p-1}}} \right)^2 + \left( \frac{c^{d}}{\sqrt{2^{p-1}}} \right)^2 \\
	=& 1 - \frac{1}{2^{r-2}} + \frac{1}{\sqrt{2^{2r-3}}} +  \frac{c^{2d}}{2^r}
	\end{split}
	\end{align}
	By comparing this to the excess risk lower bound in \eqref{eq:114}, it is easy to see that $\left( 1 - \frac{1}{2^r} \right) + c^{2d} \left( \frac{1}{2^r} - \frac{1}{2^{p-2}} \right)$ in $\eqref{eq:114}$ is still greater than $1 - \frac{1}{2^{r-2}} + \frac{1}{\sqrt{2^{2r-3}}} +  \frac{c^{2d}}{2^r}$ for $c = \frac{1}{p-r+1}$. 
	
	Therefore, after examining the five cases of $\tilde{\W} = (\tilde{\W}_d, \cdots, \tilde{\W}_1, \tilde{\w}_0)$ together with any fixed $\tilde{\W}^* = (\tilde{\W}_d^*, \cdots, \tilde{\W}_1^*, \tilde{\w}_0^*)$, we know except for one case that has negligible size compared with the entire hypothesis class $\G_{p,d}$, the largest excess risk lower bound is given by \textbf{Case 1} in \eqref{eq:114}, 
	\begin{align} \label{eq:127} 
	\begin{split}
	R(\tilde{\W})  - R(\tilde{\W}^*) 
	=& \frac{1}{2} \left[ 1 - \text{erf} \left( \frac{ \tilde{\w}^{\top} \tilde{\w}^* }{\sqrt{2 \sigma^2 \tilde{\w}^{\top} \left( \I_p - \M_d(\tilde{\W}^*) \right)^{-1} \tilde{\w} }} \right) \right] \\
	&- \frac{1}{2} \left[ 1 - \text{erf} \left( \frac{ \| \tilde{\w}^* \|_2^2 }{\sqrt{2 \sigma^2 \tilde{\w}^{\top} \left( \I_p - \M_d(\tilde{\W}^*) \right)^{-1} \tilde{\w} }}\right) \right] \\
	\geq& \frac{1}{2} \left[ \text{erf} \left( \frac{ 1 - \frac{1}{2^r} + \left( \frac{c^{2d}}{2^r} \right) }{\sqrt{2 \sigma^2 \tilde{\w}^{\top} \left( \I_p - \M_d(\tilde{\W}^*) \right)^{-1} \tilde{\w} }}\right) 
	- \text{erf} \left( \frac{ \tilde{\w}^{\top} \tilde{\w}^* }{\sqrt{2 \sigma^2 \tilde{\w}^{\top} \left( \I_p - \M_d(\tilde{\W}^*) \right)^{-1} \tilde{\w} }} \right) \right] \\
	=& \frac{\text{erf}(c_1) - \text{erf}(c_0)}{2},
	\end{split}
	\end{align}
	and for $\tilde{\W} \in \G_{p,d}$ such that $\tilde{\W} \neq \tilde{\W}^*$ but $\tilde{\w} = \tilde{\w}^*$, $R(\tilde{\W}) = R(\tilde{\W}^*)$. Thus we have the equation \eqref{eq:23} in the lemma.
	
\end{proof}

\subsubsection{Proof of Theorem 3.3 in the main text}

Now we have analyzed the excess risk $R(\tilde{\W}) - R(\tilde{\W}^*)$ and found that $\tilde{\w} = \tilde{\w}^*$ or not decides whether the excess risk is positive or not, we remove the indicator function $\mathds{1}\{ \tilde{\w} \neq \tilde{\w}^* \}$ by applying the distance-based Fano's inequality (Lemma 3.8 in the main text). 

\begin{proof}
	
	By Theorem 3.1 in the main text, we know that
	\begin{align}
	\begin{split}
	n \leq \frac{\sigma^2 \left[ d \left( \sum_{i=1}^{r} \log(i) \right) + p \log(2) - \log(4) \right] }{4}
	\implies& 
	\xi_1(\hat{f}, \mathbb{P}) := P_{(\tilde{\W}^*, S) \sim \mathbb{P}}(\hat{f}(S) \neq \tilde{\W}^* ) \geq \frac{1}{2}
	\end{split}
	\end{align}
	for a fixed truth $\tilde{\W}^* \in \G_{p,d}$, any decoder $\hat{f} \in \Psi(\G_{p,d})$ and any $\mathbb{P} \in \mathcal{P}' := \mathcal{P}_{\tilde{\W}, S}( \{\text{Uniform}(\mathcal{G}_{p, d})\}, \mathcal{P}_{(\x, y) | \tilde{\W}}^{\text{Id}, \mathcal{Q}_{\x}} )$ as described in Theorem 3.1 in the main text.
	
	By Lemma \ref{lemma:4.13} we know that
	\begin{align}
	\begin{split}
	\tilde{\W} \neq \tilde{\W}^* 
	\implies 
	R(\tilde{\W}) - R(\tilde{\W}^*) \geq 
	\mathds{1} \{ \tilde{\w} \neq \tilde{\w}^* \} \cdot \frac{\text{erf}(c_1) - \text{erf}(c_0)}{2}
	\end{split}
	\end{align}
	
	However, the indicator function $\mathds{1} \{ \tilde{\w} \neq \tilde{\w}^* \}$ is undesired in an excess risk lower bound. Also it signals an identifiability issue, as it is possible to have distinct $\tilde{\W} = (\tilde{\W}_d, \cdots, \tilde{\W}_1, \tilde{\w}_0)$ and $\tilde{\W}^* = (\tilde{\W}_d^*, \cdots, \tilde{\W}_1^*, \tilde{\w}_0^*)$ such that $\tilde{\W}_d \cdots \tilde{\W}_1 \tilde{\w}_0 = \tilde{\W}_d^* \cdots \tilde{\W}_1^* \tilde{\w}_0^*$. This is discussed in the proof of Lemma \ref{lemma:4.13}, where we showed such scenario does exist, though for any fixed truth $\tilde{\W}^*$, the set of $\tilde{\W} \neq \tilde{\W}^*$ such that $\tilde{\w} = \tilde{\w}^*$ has exponentially small cardinality. 
	
	To resolve this issue, we employ a distance-based Fano's inequality (Lemma 3.8 in the main text) and we achieve a positive excess risk lower bound with a sample complexity lower bound losing a factor of $d$.
	
	We first define a symmetric function $\rho: \mathcal{V} \times \mathcal{V} \to \mathbb{R}$, and show it is a metric (though $\rho$ is not required to be a metric for the distance-based Fano's inequality), and then we show for $t = 1$ we can achieve the claimed sample complexity lower bound.
	
	Let $\rho(\tilde{\W}, \tilde{\W}') := \mathds{1} \{\tilde{\W} \neq \tilde{\W}'\} + \mathds{1} \{\tilde{\w} \neq \tilde{\w}'\}$ for any $\tilde{\W}, \tilde{\W}' \in \G_{p,d}$ as defined in restriction \textbf{R1}, where $\tilde{\w} = \tilde{\W}_d \cdots \tilde{\W}_1 \tilde{\w}_0$. It is easy to see that $\rho$ is symmetric and nonnegative, and equals to zero if and only if $\tilde{\W} = \tilde{\W}'$. Then it remains to verify the triangle inequality. We verify this case by case.
	
	\noindent \textbf{Case 1:} $\tilde{\W} = \tilde{\W}'$. Then for any $\tilde{\W}'' \in \G_{p,d}$, it is trivial that $\rho(\tilde{\W}, \tilde{\W}') = 0 \leq \rho(\tilde{\W}, \tilde{\W}'') + \rho(\tilde{\W}'', \tilde{\W}')$.
	
	\noindent \textbf{Case 2:} $\tilde{\W} \neq \tilde{\W}'$ and $\tilde{\w} = \tilde{\w}'$. The latter means $\rho(\tilde{\W}, \tilde{\W}') = 1$. Then for any $\tilde{\W}'' \in \G_{p,d}$, without loss of generality, there are three possibilities, and for all three possibilities the triangle inequality $\rho(\tilde{\W}, \tilde{\W}') \leq \rho(\tilde{\W}, \tilde{\W}'') + \rho(\tilde{\W}'', \tilde{\W}')$ holds. 
	
	(i) $\tilde{\W}'' = \tilde{\W}'$, then $\rho(\tilde{\W}, \tilde{\W}'') = \rho(\tilde{\W}, \tilde{\W}') = 1$ and $\rho(\tilde{\W}'', \tilde{\W}') = 0$. 
	
	(ii) $\tilde{\W}'' \neq \tilde{\W}'$ and $\tilde{\w}'' = \tilde{\w}'$, then $\rho(\tilde{\W}, \tilde{\W}'') = \rho(\tilde{\W}, \tilde{\W}') = 1$ and $\rho(\tilde{\W}'', \tilde{\W}') = 1$. 
	
	(iii) $\tilde{\W}'' \neq \tilde{\W}'$ and $\tilde{\w}'' \neq \tilde{\w}'$, then $\rho(\tilde{\W}, \tilde{\W}'') = \rho(\tilde{\W}, \tilde{\W}') = 2$ and $\rho(\tilde{\W}'', \tilde{\W}') = 2$. 
	
	\noindent \textbf{Case 3:} $\tilde{\W} \neq \tilde{\W}'$ and $\tilde{\w} \neq \tilde{\w}'$. The latters means $\rho(\tilde{\W}, \tilde{\W}') = 2$. Then for any $\tilde{\W}'' \in \G_{p,d}$, without loss of generality, there are also three possibilities, and the triangle inequality $\rho(\tilde{\W}, \tilde{\W}') \leq \rho(\tilde{\W}, \tilde{\W}'') + \rho(\tilde{\W}'', \tilde{\W}')$ holds. 
	
	(i) $\tilde{\W}'' = \tilde{\W}'$, then $\rho(\tilde{\W}, \tilde{\W}'') = \rho(\tilde{\W}, \tilde{\W}') = 2$ and $\rho(\tilde{\W}'', \tilde{\W}') = 0$. 
	
	(ii) $\tilde{\W}'' \neq \tilde{\W}'$ and $\tilde{\w}'' = \tilde{\w}'$, then $\rho(\tilde{\W}, \tilde{\W}'') = \rho(\tilde{\W}, \tilde{\W}') = 2$ and $\rho(\tilde{\W}'', \tilde{\W}') = 1$. 
	
	(iii) $\tilde{\W}'' \neq \tilde{\W}'$ and $\tilde{\w}'' \neq \tilde{\w}'$, then $\rho(\tilde{\W}, \tilde{\W}'') = \rho(\tilde{\W}, \tilde{\W}') = 2$ and $\rho(\tilde{\W}'', \tilde{\W}') = 2$. 
	
	Now we calculate $N_t^{\max}$ and $N_t^{\min}$ for $t = 1$. Note $N_1^{\max} = \max_{v \in \mathcal{V}} \{ \text{card} \{ v' \in \mathcal{V}: \rho(v, v') \leq 1 \} \} = \max_{\tilde{\W} \in \G_{p,d}} \{ 	\text{card} \{ \tilde{\W}' \in \G_{p,d}: \tilde{\w}' = \tilde{\w}\} \} = (r!)^{d-1}$. That is, for any $\tilde{\W} \in \G_{p,d}$, the set $\{ \tilde{\W}' \in \G_{p,d}: \tilde{\w}' = \tilde{\w}\}$ has same cardinality $(r!)^{d-1}$ as shown in eq \eqref{eq:105}. Thus $N_1^{\min} = (r!)^{d-1}$ as well. Obviously $|\G_{p,d}| - N_1^{\min} > N_1^{\max}$.
	
	Now apply the distace-based Fano's inequality, for any truth $\tilde{\W}^*$ uniformly chosen by nature from $\G_{p,d}$ and any hypothesis $\tilde{\W} \in \G_{p,d}$ obtained by any procedure, we have
	\begin{align}
	&P(\rho(\tilde{\W}, \tilde{\W}^*) > 1) 
	\geq 1 - \frac{\mathbb{I}(\tilde{\W}^*;S) + \log 2 }{\log \frac{|\G_{p,d}|}{N_1^{\max}}} 
	\geq 1 - \frac{\frac{2n}{ \sigma^2 } + \log 2 }{\log \frac{ \left( r! \right)^{d}  \cdot 2^p }{(r!)^{d-1}}} \stackrel{\text{set}}{\geq} \frac{1}{2} 
	\impliedby
	n \leq \frac{\sigma^2}{4} \left[ p \log(2) + \sum_{i=1}^{r} \log(i) \right]
	\end{align}
	
	Observe that the event $\{ \rho(\tilde{\W}, \tilde{\W}^*) > 1 \}$ is equivalent to $\{ \tilde{\W} \neq \tilde{\W}^* \text{ and } \tilde{\w} \neq \tilde{\w}'\}$, which in turn implies eq \eqref{eq:127}, $R(\tilde{\W}) - R(\tilde{\W}^*) \geq \frac{1}{2} \left[ \text{erf}(c_1) - \text{erf}(c_0) \right]$. Therefore we have
	\begin{align}
	\begin{split}
	&n \leq \frac{\sigma^2}{4} \left[ p \log(2) + \sum_{i=1}^{r} \log(i) \right] 
	\implies 
	\xi_2(\hat{f}, \mathbb{P}) := P_{(\tilde{\W}^*, S) \sim \mathbb{P}} \left(\tilde{R}(\hat{f}(S)) \geq
	\frac{\text{erf} \left( c_1 \right) - \text{erf} \left( c_0 \right) }{2} \right) \geq \frac{1}{2} \\
	\end{split}
	\end{align}
\end{proof}

\subsection{Linear approximation of the excess risk lower bound in Theorem 3.3 in the main text}

As the arguments in both $\text{erf}$ functions in \eqref{eq:23}, $c_1$ and $c_0$, have numerator less than $1$ and the denominator is $\sigma \sqrt{2 \left[(d+1) \left( 1 - \frac{1}{2^r} \right) + \left( \frac{1-c^{2(d+1)}}{1 - c^2} \right) \left( \frac{ c^{2d}}{2^r} \right) \right]}$, which is roughly $2 \sigma \sqrt{d}$, with $\sigma$ being a constant reflecting the variance of input data $\x$ in eq (19) in the main text.

We argue that a linear approximation of the erf function is acceptable. Both $c_0$ and $c_1$ in \eqref{eq:23} are roughly $\frac{1}{2 \sigma \sqrt{d}}$, the derivative of the erf function is $\frac{d \text{erf}(x)}{dx} = \frac{2 \exp(-x^2)}{\sqrt{\pi}}$, and for small $x$, e.g. $\frac{1}{2 \sigma \sqrt{d}}$ (in deep networks, $d$ is large), $\exp(-x^2)$ would be close to 1, thus we can use $\frac{d \text{erf}(x)}{dx} \rvert _{x=0} = 2 / \sqrt{\pi}$ for a linear approximation of the erf function. Thus, we have
\begin{align}
\begin{split}
R(\tilde{\W})  - R(\tilde{\W}^*) 
&\geq \frac{\text{erf} \left( c_1 \right) - \text{erf} \left( c_0 \right) }{2} \\
&\approx \frac{2}{\sqrt{\pi}} \frac{ c_1 - c_0 }{2} \\
&= \frac{\frac{1}{\sqrt{2 \pi}} \cdot \frac{c^{2d}}{2^{p-2}}}{ \sigma \sqrt{ (d+1) \left( 1 - \frac{1}{2^r} \right) + \left( \frac{1-c^{2(d+1)}}{1 - c^2} \right) \left( \frac{ c^{2d}}{2^r} \right)} } \\
&\geq \frac{\frac{1}{\sqrt{2 \pi}} \cdot \frac{1}{2^{p-2}} \frac{1}{p^{2d}} }{ \sigma \sqrt{(d+1) \left( 1 - \frac{1}{2^r} \right) + \left( \frac{1-c^{2(d+1)}}{1 - c^2} \right) \left( \frac{ c^{2d}}{2^r} \right)}} 
\end{split}
\end{align}
where $c = \frac{1}{p-r+1}$ and $c \geq  \frac{1}{p}$ is used for the last inequality.

Now we further simplify this approximate lower bound. Observe, for any $k \in \mathbb{N}$, if $r \leq \frac{p}{k} + 1$, then
\begin{align}
\begin{split}
&p - r + 1 \geq p + 1 - \frac{p}{k} - 1 = \frac{k-1}{k} p 
\implies \frac{1}{p} = \frac{1}{p-r+1} = c \leq \frac{k}{k-1} \frac{1}{p} \\
\implies& \frac{1}{1-c^2} \leq \frac{1}{1 - \left[ \left( \frac{k}{k-1} \right) \frac{1}{p} \right]^2} = \frac{p^2}{p^2 - \left( \frac{k}{k-1} \right)^2} = 1 + \frac{\left( \frac{k}{k-1} \right)^2}{p^2 - \left( \frac{k}{k-1} \right)^2} \leq \frac{p^2}{p^2-4} \leq \frac{3^2}{3^2-4} = \frac{9}{5}
\end{split}
\end{align}
where the second last inequality is attained by the maximum of $\left( \frac{k}{k-1} \right)^2$ at $k = 2$, and the last inequality is attained by the maximum of $\frac{p^2}{p^2-4}$ by $p = 3$ (if $p = 2$, this term is not defined). 

Therefore the risk gap could be further lower bounded, 
\begin{align}
\begin{split}
R(\tilde{\W})  - R(\tilde{\W}^*) 
&\gtrapprox
\frac{\frac{1}{\sqrt{2 \pi}} \cdot \frac{1}{2^{p-2}} \frac{1}{p^{2d}} }{ \sigma \sqrt{ \left[ (d+1) \left( 1 - \frac{1}{2^r} \right) + \frac{9}{5} \left( 1-c^{2(d+1)} \right) \left( \frac{ c^{2d}}{2^r} \right) \right]}} \\
&\geq \frac{\frac{1}{\sqrt{2 \pi}} \cdot \frac{1}{2^{p-2}} \frac{1}{p^{2d}} }{ \sigma \sqrt{(d+1) \left( 1 - \frac{1}{2^r} \right) + \frac{9}{5} \left( \frac{ c^{2d}}{2^r} \right)}} \\
&\geq \frac{\frac{1}{\sqrt{2 \pi}} \cdot \frac{1}{2^{p-2}} \frac{1}{p^{2d}} }{ \sigma \sqrt{(d+1) \left( 1 - \frac{1}{2^r} \right) + \frac{9}{5} \left( \frac{k}{k-1} \frac{1}{p} \right)^{2d} \left( \frac{1}{2^r} \right)}} \\
&\geq \frac{\frac{1}{\sqrt{2 \pi}} \cdot \frac{1}{2^{p-2}} \frac{1}{p^{2d}} }{ \sigma \sqrt{(d+1) \left( 1 - \frac{1}{2^r} \right) + \frac{9}{5} \left( \frac{2}{3} \right)^{2d} \left( \frac{1}{2^r} \right)}} \\
&\geq \frac{\frac{1}{\sqrt{2 \pi}} \cdot \frac{1}{2^{p-2}} \frac{1}{p^{2d}} }{ \sigma \sqrt{ (d+1) \left( 1 - \frac{1}{2^r} \right) + \frac{9}{10} \left( \frac{2}{3} \right)^{2d}}} \\
&\geq \frac{\frac{1}{\sqrt{2 \pi}} \cdot \frac{1}{2^{p-2}} \frac{1}{p^{2d}} }{ \sigma \sqrt{ (d+1) + \frac{2}{5}}} \\
&= \frac{\frac{1}{\sqrt{2 \pi}} \cdot \frac{1}{2^{p-2}} \frac{1}{p^{2d}} }{ \sigma \sqrt{ d + \frac{7}{5}}} \in \Theta \left( \frac{1}{2^p \cdot p^{2d} \cdot \sigma \sqrt{d}} \right)
\end{split}
\end{align}
where the first inequality uses $1 - \frac{1}{2^r} \leq 1$, and the second inequality uses an assumption that $r \leq \frac{p}{k} + 1$ for any $k \geq 2$, and the third inequality follows as $k \geq 2$ and $p \geq 3$. The fourth inequality uses $\frac{1}{2^r} \leq \frac{1}{2}$, and the fifth inequality uses $d \geq 1$.

Now let $\epsilon = \frac{c_1'}{2^p \cdot p^{2d} \cdot \sigma \sqrt{d}}$ for some constant $c_1'$. We first hold $d$ fixed, then $\epsilon = \frac{c_2'}{2^p p^{2d}}$ for some constant $c_2'$, where $2^p$ is the dominant term. Therefore we write $\epsilon = \frac{c_3'}{2^p}$ for some other constant $c_3'$, i.e. $\frac{1}{\epsilon} \in \Theta \left( 2^p \right)$, which gives $p \in \Theta\left( \log \frac{1}{\epsilon} \right)$. Then we hold $p$ fixed, then $\epsilon = \frac{c_4'}{p^{2d} \sqrt{d + 7/5}} \approx \frac{1}{p^{2d}}$ as $p^{2d}$ is the dominant factor. Thus we have $2d \log(p) \in \Theta \left( \log \frac{1}{\epsilon} \right)$, which in turn gives $d = \frac{\log \frac{1}{\epsilon}}{2 \log(p)} \in \Theta\left( \frac{\log \frac{1}{\epsilon}}{2 \log \frac{1}{\epsilon}} \right) = \Theta(1)$. 

Combine $p \in \Theta\left( \log \frac{1}{\epsilon} \right)$ and $d = \Theta(1)$, and assume $r \propto p$, the sample complexity lower bound we found in Theorem 3.1 in the main text, $n \in \Omega \left( d r \log(r) + p \right)$ becomes $n \in \Omega \left(\log \frac{1}{\epsilon} \cdot \log \left( \log \frac{1}{\epsilon} \right) + \log \frac{1}{\epsilon} \right)$.

\end{document}